\documentclass[aos]{imsart}

\RequirePackage{amsthm,amsmath,amsfonts,amssymb}
\RequirePackage{graphicx}

\usepackage[colorlinks=true,citecolor=blue]{hyperref}
\newcommand{\circled}[1]{\small{\raisebox{.6pt}{\textcircled{\raisebox{-.8pt}{#1}}}}}

\usepackage{amssymb}
\usepackage{amsmath,mathrsfs,dsfont}
\usepackage{nicefrac}
\usepackage{algorithm}
\usepackage{algorithmicx}
\usepackage{algpseudocode}

\usepackage{booktabs}

\usepackage{color}
\usepackage[usenames,dvipsnames,svgnames,table]{xcolor}
\usepackage{enumitem}

\usepackage{array}
\usepackage{graphicx,tikz}
\usepackage[mathscr]{euscript}
\usepackage{amsthm}
\usepackage[sort&compress,square,comma,numbers]{natbib}
\usepackage{cite}

\usepackage{bm}
\usepackage{bbm}

\usepackage{epstopdf}
\usepackage{subcaption}
\usepackage{cleveref}
\usepackage{thmtools}
\usepackage{thm-restate}

\usepackage{bbding}
\usepackage{mathtools}
\counterwithin{equation}{section}
\allowdisplaybreaks
\startlocaldefs

\newcommand{\bUl}{\bU^{(\bL)}}
\newcommand{\bUlr}{\bU^{(\bL,r_0)}}
\newcommand{\bSigmal}{\bSigma^{(\bL)}}
\newcommand{\bSigmalr}{\bSigma^{(\bL,r_0)}}
\newcommand{\hlambdal}{\hat \lambda^{(\bL)}}

\newcommand{\tbv}{\tilde {\mathbf v}}
\newcommand{\tphi}{\tilde \phi}

\newcommand{\tbd}{\tilde \bd}
\newcommand{\td}{\tilde d}
\newcommand{\tg}{\tilde g}
\newcommand{\tE}{\tilde E}
\newcommand{\tT}{\tilde T}
\newcommand{\dVC}{d^{\textup{(VC)}}}
\newcommand{\tOmega}{\tilde \Omega}

\usepackage{framed}
\usepackage{xcolor}

\newcommand{\cfrakR}{\mathfrak{R}} 

\newcommand{\Span}{\mathop\mathrm{Span}}

\newcommand{\bbx}{\overset{\rightharpoonup}{\bx}}

\newcommand{\ind}{\mathop\mathrm{ind}}

\endlocaldefs

\graphicspath{{illustrations/}}
\makeatletter
\let\TLCincludegraphics\includegraphics
\renewcommand{\includegraphics}[2][]{%
  \begingroup
  \filename@parse{#2}%
  \def\TLCeps{eps}%
  \ifx\filename@ext\TLCeps
    \edef\TLCgraphicpath{\filename@area\filename@base-eps-converted-to.pdf}%
    \TLCincludegraphics[#1]{\TLCgraphicpath}%
  \else
    \TLCincludegraphics[#1]{#2}%
  \fi
  \endgroup
}
\makeatother
\newcounter{optproblem}

\newtheoremstyle{mytheoremstyle} 
    {\topsep}                    
    {\topsep}                    
    {\normalfont}                
    {}                           
    {\bfseries}                   
    {.}                          
    {.5em}                       
    {}  

\theoremstyle{mytheoremstyle}
\newtheorem{theorem}{Theorem}[section]
\newtheorem{remark}[theorem]{Remark}
\newtheorem{proposition}[theorem]{Proposition}
\newtheorem{corollary}[theorem]{Corollary}

\newtheorem{claim}{Claim}[section]
\newtheorem*{theorem*}{Theorem}
\newtheorem*{lemma*}{Lemma}
\newtheorem*{remark*}{Remark}
\newtheorem*{claim*}{Claim}

\newtheorem{lemma}[theorem]{Lemma}

\newtheorem{assumption}{Assumption}

\theoremstyle{remark}
\newtheorem{definition}{Definition}[section]

\DeclareMathAlphabet{\pazocal}{OMS}{zplm}{m}{n}
\DeclareMathAlphabet{\mathpzc}{OMS}{pzc}{m}{it}

\setlist[itemize]{leftmargin=*}

\renewcommand{\hat}{\widehat}

\newcommand{\bfm}[1]{\ensuremath{\mathbf{#1}}}
\newcommand{\bfsym}[1]{\ensuremath{\boldsymbol{#1}}}

\def\ba{\boldsymbol a}   \def\bA{\bfm A}

\def\bd{\bfm d}

\def\bi{\bfm i}   \def\bI{\bfm I}  
   \def\bJ{\bfm J}  
   \def\bK{\bfm K}  
   \def\bL{\bfm L}  
     
     \def\NN{\mathbb{N}}
     
     \def\PP{\mathbb{P}}
     
   \def\bR{\bfm R}  \def\RR{\mathbb{R}}
   \def\bS{\bfm S}  
\def\bt{\bfm t}     
   \def\bU{\bfm U}  
\def\bv{\bfm v}     
     
\def\bx{\bfm x}   \def\bX{\bfm X}  
\def\by{\bfm y}   \def\bY{\bfm Y}  
\def\bz{\bfm z}   \def\bZ{\bfm Z}  
\def\bzero{\bfm 0} 

 \def\cA{{\cal  A}}
 
 \def\cC{{\cal  C}}
 
 \def\cE{{\cal  E}}
 \def\cF{{\cal  F}}
 \def\cG{{\cal  G}}
 \def\cH{{\cal  H}}

 \def\cL{{\cal  L}}

 \def\cO{{\cal  O}}
 \def\cP{{\cal  P}}
 
 \def\cR{{\cal  R}}

 \def\cX{{\cal  X}}
 \def\cY{{\cal  Y}}

\def\balpha{\bfsym \alpha}

\def\bpi{\bfsym {\pi}}

\def\bsigma{\bfsym \sigma}
\def\bSigma{\bfsym \Sigma}

\def\+#1{\mathcal{#1}}
\def\-#1{\textup{#1}}

\def\set#1{\left\{ #1 \right\}}
\def\pth#1{\left( #1 \right)}
\def\bth#1{\left[ #1 \right]}
\def\abth#1{\left | #1 \right |}

\def\defeq {\coloneqq}

\newcommand{\overbar}[1]{\mkern 1.5mu\overline{\mkern-1.5mu#1\mkern-1.5mu}\mkern 1.5mu}

\def \longmid {\,\middle\vert\,}
\DeclareMathSymbol{\relcolon}{\mathrel}{operators}{"3A}

\newcommand{\La}{\left\langle\kern-0.64ex\left\langle}
\newcommand{\Ra}{\right\rangle\kern-0.64ex\right\rangle}

\def\Norm#1#2{{\left\vert\kern-0.4ex\left\vert\kern-0.4ex\left\vert #1
    \right\vert\kern-0.4ex\right\vert\kern-0.4ex\right\vert}_{#2}}
\def\norm#1#2{{\left\|#1\right\|}_{#2}}

\def\ltwonorm#1{\norm{#1}{2}}

\newcommand{\rank}{\textup{rank}}

\def\tr#1{\textup{tr}\left(#1\right)}

\newcommand{\1}{{\rm 1}\kern-0.25em{\rm I}}
\def\indict#1{{\rm 1}\kern-0.25em{\rm I}_{\set{#1}}}

\def \eps  {\epsilon}
\def \eps {\varepsilon}

\def \diff {{\rm d}}
\def \iprod#1#2{\left\langle #1, #2 \right\rangle}

\def\set#1{\left\{#1\right\}}

\def\unitsphere#1{\mathbb{S}^{#1}}
\DeclareMathOperator*{\argmin}{arg\,min}

\def \E {\mathbb{E}}
\def\Expect#1#2{\E_{#1}\left[#2\right]}

\def \Pr {\textup{Pr}}
\newcommand{\Prob}[1]{\Pr\left[#1\right]}
\def \Var#1{\textup{Var}\left[#1\right]}

\def \lsim {\lesssim}

\newcommand{\beq}{\begin{equation}}
\newcommand{\eeq}{\end{equation}}
\newcommand{\beqa}{\begin{eqnarray}}
\newcommand{\eeqa}{\end{eqnarray}}
\newcommand{\beqas}{\begin{eqnarray*}}
\newcommand{\eeqas}{\end{eqnarray*}}
\def\bal#1\eal{\begin{align}#1\end{align}}
\def\bals#1\eals{\begin{align*}#1\end{align*}}
\def\bsal#1\esal{\begin{small}\begin{align}#1\end{align}\end{small}}
\def\bsals#1\esals{\begin{small}\begin{align*}#1\end{align*}\end{small}}
\def\bsfal#1\esfal{\begin{small}\begin{flalign}#1\end{flalign}\end{small}}

\begin{document}

\begin{frontmatter}
\title{Improved Generalization Bounds for Transductive Learning by Transductive Local Complexity and Its Applications}
\runtitle{Transductive Local Complexity and Its Applications}

\begin{aug}
\author[A]{\fnms{Yingzhen}~\snm{Yang}\ead[label=e1]{yingzhen.yang@asu.edu}},

\address[A]{School of Computing and Augmented Intelligence, Arizona State University \printead[presep={,\ }]{e1}}

\end{aug}

\begin{abstract}
We study the extension of Transductive Local Complexity (TLC)~\citep{yang2025a-concentration-sampling-without-replacement} as a general tool for analyzing the generalization performance of transductive learning methods. Building on earlier localized complexity methods for transductive learning~\citep{TolstikhinBK14-local-complexity-TRC}, the TLC framework was introduced in~\citep{yang2025a-concentration-sampling-without-replacement}. We extend its concentration analysis to general training and test sample sizes through chain analysis and entropy inequalities underlying the exponential Efron Stein inequality. These inequalities are applied to the supremum of the difference between test and training losses and its companion process over the squared function class, then averaged over independent uniform relabeling. Combining these concentration bounds with the peeling strategy, expectation decomposition, and surrogate variance operator introduced in~\citep{yang2025a-concentration-sampling-without-replacement} yields excess risk bounds nearly consistent with classical Local Rademacher Complexity (LRC) bounds up to a logarithmic gap.
We further advance the analysis of transductive learning through two applications: (1) for realizable transductive learning over binary-valued function classes with finite VC dimension $\dVC$ and $u \ge m \ge \dVC$, where $u$ and $m$ are the numbers of test features and training features, Theorem~\ref{theorem:TLC-delta-ell-f-excess-risk-upper-bound-VC-dim} gives a TLC bound for a single static empirical risk minimizer of order
$O(\dVC \log(me/\dVC)/m)$ at any fixed confidence level, improving the logarithmic complexity term in the known ERM bound under a uniform random split~\citep[Theorem~7]{tolstikhin2016minimaxlowerboundsrealizable} from $\log n$ to $\log m$, and (2) Theorem~\ref{theorem:TLC-kernel} presents a sharper excess risk bound for transductive kernel learning compared to the prior local complexity-based result in~\citep{TolstikhinBK14-local-complexity-TRC}.
\end{abstract}



\end{frontmatter}

\section{Introduction}
We study transductive learning, where the learner has access to both labeled training data and unlabeled test data, and the goal is to predict the labels of the test data. Obtaining tight generalization bounds in this setting is a central problem in statistical learning theory. Classical tools from inductive learning, such as Rademacher complexity and VC dimension, have been applied to transductive settings, including empirical risk minimization, transductive regression, and transductive classification~\citep{Vapnik1982,Vapnik1998,Cortes2006-transductive-regression,El-Yaniv2009-TRC}. To achieve sharper generalization bounds, localized measures such as the Local Rademacher Complexity (LRC)~\citep{bartlett2005} have been employed, notably in~\citep{TolstikhinBK14-local-complexity-TRC}. The classical LRC framework~\citep{bartlett2005} provides a sharp excess risk bound for the empirical risk minimizer $\hat f$ in inductive learning: for every $x > 0$, with probability at least $1 - \exp(-x)$,
\bal\label{eq:conceptual-risk-bound-inductive}
\textsf {Excess Risk of } \,\, \hat f \le  \Theta\pth{\textsf {Fixed Point  for Certain Empirical Process} + \frac xn },
\eal%
where $\Theta$ only hides a constant factor, and $n$ is the size of the training data. Given the fact that LRC is capable of achieving various minimax rates for M-estimators in tasks such as nonparametric regression in the inductive regime, we propose to solve the following interesting and important question for LRC-based transductive learning: can we have a local complexity-based bound for the excess risk of transductive learning matching or nearly matching that for the inductive setting?

An earlier local complexity result relevant to this question is presented in \citep[Corollary 14]{TolstikhinBK14-local-complexity-TRC}, where the excess risk bound is given as the following inequality which happens with high probability for transductive kernel learning:
\bal\label{eq:local-complexity-existing-excess-risk-bound}
{\textsf {Excess Risk of} } \,\, \hat f &\le  \Theta\pth{\frac nu r_m^*+ \frac nm r_u^*+\frac 1m +\frac 1u}. 
\eal
Here $r_m^*$ and $r_u^*$ are the fixed points of upper bounds for certain empirical processes, where $m,u$ are the sizes of the training data and test data. It is remarked that the above bound may diverge due to the undesirable factors of $n/m$ and $n/u$ before the fixed points. When $m$ or $u$ grows at a much slower rate than $n$ and $n = u+m$,  $n/m \cdot r_u^* + n/u \cdot r_m^*$ may not converge to $0$. An example for the standard transductive kernel learning is given in Section~\ref{sec:TKL}. As a result, there is a remarkable difference between the current state-of-the-art excess bound (\ref{eq:local-complexity-existing-excess-risk-bound}) in the transductive setting and the excess risk bound (\ref{eq:conceptual-risk-bound-inductive}) in the inductive setting, and the latter always converges to $0$ under standard learning models. Such a gap mainly arises from the lack of effective concentration inequalities of the supremum of empirical processes for sampling without replacement (see examples in Section~\ref{sec:concentration-general-sup-empirical}), while concentration inequalities for sampling with replacement have been widely studied in the inductive learning literature with numerous powerful tools.  We note that the excess risk bound in \citep[Corollary 13]{TolstikhinBK14-local-complexity-TRC} still diverges when $m = o(\sqrt n)$ or $u = o(\sqrt n)$ as $m,u \to \infty$.

One of our main results is the following improved bound for the excess risk of transductive learning. With probability at least $1-3\exp(-x)-3\delta$ for every $x>0$ and $\delta \in (0,1/3)$,
\bal\label{eq:conceptual-excess-risk-TLC}
{\textsf {Excess Risk of} } \,\, \hat f  \le \Theta\pth{r_{u,m} + r^* + \frac{(\log_2 (4\min\set{u,m}/\delta)) x}{\min\set{u,m}}}, (\ref{eq:TLC-ell-f-excess-risk-upper-bound}){\textup{ in Theorem}}~\ref{theorem:TLC-delta-ell-f-excess-risk-upper-bound}.
\eal
Here $r_{u,m},r^*$ are the fixed points of upper bounds for certain empirical processes, which all converge to $0$ with a fast rate as is the case in the popular inductive learning models. As a result of the improved excess risk bound (\ref{eq:conceptual-excess-risk-TLC}), we extend the corresponding TLC analysis to general training and test sample sizes, with only a logarithmic gap of $\log_2 (4\min\set{u,m}/\delta)$ from the LRC-based inductive bound.

This paper extends the Transductive Local Complexity (TLC) framework introduced in~\citep{yang2025a-concentration-sampling-without-replacement}, building on earlier localized complexity methods for transductive learning~\citep{TolstikhinBK14-local-complexity-TRC}. Chain analysis and entropy averaging under uniform relabeling yield concentration bounds for general training and test sample sizes in Theorem~\ref{theorem:main-inequality-TLC}, extending the generic excess risk analysis and kernel application beyond the sample size regimes covered by the conference paper. We also obtain a TLC bound for a single static empirical risk minimizer in realizable binary classification whose logarithmic complexity term depends on the training sample size, with both applications summarized in Section~\ref{sec:summary-main-results} and detailed in Section~\ref{sec:applications}:
(1) For realizable transductive learning over $\set{0,1}$-valued function classes with finite VC-dimension $2 \le \dVC < \infty$ and $u \ge m \ge \dVC$, where $u$ and $m$ are the numbers of test features and training features, Theorem~\ref{theorem:TLC-delta-ell-f-excess-risk-upper-bound-VC-dim} establishes a TLC bound for a single static empirical risk minimizer of order
$\cO\pth{\frac{\dVC \log(me/\dVC)}{m}}$ at any fixed confidence level. This improves the logarithmic complexity term in the known ERM bound under a uniform random split~\citep[Theorem~7]{tolstikhin2016minimaxlowerboundsrealizable} from $\log n$ to $\log m$. The remaining factor is a feature of this bound for the static empirical minimizer. Theorem~\ref{theorem:optimal-transductive-transfer} gives a bound without this factor for an appropriate randomized learner.
(2) Theorem~\ref{theorem:TLC-kernel} provides a new, sharper excess risk bound for transductive kernel learning, significantly improving upon the prior local complexity–based result in~\citep{TolstikhinBK14-local-complexity-TRC} under the same assumptions.

It is worthwhile to mention that concentration inequalities about sampling without replacement have been actively studied in the literature~\citep{BARDENET2015-sampling-without-replacement,
Tolstikhin2017-sampling-without-replacement}, including those on the multislice which are based on the modified log-Sobolev inequalities~\citep{Sambale2022}. Corollary~\ref{corollary:concentration-gu} provides concentration bounds whose advantages in particular sample size regimes are discussed in Section~\ref{sec:concentration-general-sup-empirical}, in comparison with~\citep{TolstikhinBK14-local-complexity-TRC}. Furthermore, in contrast with our results, the supremum of empirical process involving sampling without replacement is not addressed in ~\citep{BARDENET2015-sampling-without-replacement}. We remark that the localized analysis in \citep{yang2025a-concentration-sampling-without-replacement} for concentration inequalities for sampling without replacement is limited to the restrictive regime that $m \gg u^2$ or $u \gg m^2$. In a strong contrast, we present general transductive learning bounds without such limitations on $m$ and $u$, incorporating the results of \citep{yang2025a-concentration-sampling-without-replacement} as special cases of this work. Furthermore, compared to~\citep{yang2025a-concentration-sampling-without-replacement}, we present a new and nearly optimal TLC bound for a single static empirical risk minimizer in realizable transductive learning, and a stronger excess risk bound for transductive kernel learning without the limitation of restrictive regime.

\subsection{Notations}
We use bold letters for matrices and vectors, and regular lower letters for scalars throughout this paper. The bold letter with a single superscript indicates the corresponding column of a matrix, e.g. $\bA^{(i)}$ is the $i$-th column of matrix $\bA$, and the bold letter with subscripts indicates the corresponding element of a matrix or vector. 
We use $\ba(i)$ to denote the $i$-th element of a vector $\ba$, and $\ba(i:j)$ denotes the vector formed by elements of $\ba$ with indices between $i$ and $j$ inclusively. Similarly, $\ba(\cC)$ denotes the vector formed by elements of $\ba$ whose indices are in the set $\cC$.  $\set{\bZ}$ denotes the set containing all the elements of a vector $\bZ$ regardless of the order of these elements. $\norm{\cdot}{F}$ and
$\norm{\cdot}{p}$ denote the Frobenius norm and the vector $\ell^{p}$-norm or the matrix $p$-norm. $\Var{\cdot}$ denotes the variance of a random variable. $\bI_n$ is an $n \times n$ identity matrix.  $\indict{E}$ is an indicator function which takes the value of $1$ if event $E$ happens, or $0$ otherwise. The complement of a set $A$ is denoted by $\overline{A}$, and $\abth{A}$ is the cardinality of the set $A$.
$\tr{\cdot}$ is the trace of a matrix. We denote the unit sphere in $d$-dimensional Euclidean space by $\unitsphere{d-1} \defeq \{\bx \colon  \bx \in \RR^d, \ltwonorm{\bx} =1\}$. Let $L^2(\cX, \mu^{(P)}) $ denote the space of square-integrable functions on the input space $\cX$ with probability measure $\mu^{(P)}$, and the inner product $\iprod{\cdot}{\cdot}_{ \mu^{(P)}}$ and $\norm{\cdot}{\mu^{(P)}}^2$ are defined as $\iprod{f}{g}_{L^2} \coloneqq \int_{\cX}f(x)g(x) \diff \mu^{(P)}(x)$ and $\norm{f}{L^2}^2 \coloneqq \int_{\cX}f^2(x) \diff \mu^{(P)} (x) <\infty$.
$\iprod{\cdot}{\cdot}_{\cH}$ and $\norm{\cdot}{\cH}$ denote the inner product and the norm in the Hilbert space $\cH$.
$a = \cO(b)$ or $a \lsim b$ indicates that there exists a constant $c>0$ such that $a \le cb$. $\tilde \cO$ indicates there are specific requirements in the constants of the $\cO$ notation. $a = o(b)$ and $a = w(b)$ indicate that $\lim \abth{a/b}  = 0$ and $\lim \abth{a/b}  = \infty$ respectively. $a = \Theta(b)$ or $a \asymp b$ denotes that
there exist constants $c_1,c_2>0$ such that $c_1b \le a \le c_2b$.
$\RR^+$ is the set of all nonnegative real numbers, and $\NN$ is the set of all the natural numbers. We use the convention that $\sum_{i=p}^q = 0$ if $p > q$ or $q=0$. $[m\relcolon n]$ denotes the set of all the natural numbers between $m$ and $n$ inclusively, and we abbreviate $[1\relcolon n]$ as $[n]$. We use the convention that
$[m\relcolon n] = \emptyset$ if $m > n$. 

\section{Problem Setup of Transductive Learning}
\label{sec:transductive-basic-setup}
We consider a set $\set{(\bbx_i, y_i)}_{i=1}^{m+u}$,
where $y_i$ is the label for the point $\bbx_i$. Let $n = m+u$, $\set{\bbx_i}_{i=1}^n \subseteq \cX \subseteq \RR^d$, $\set{y_i}_{i=1}^n \subseteq \cY \subseteq \RR$ where $\cX,\cY$ are the input and output spaces.
The learner is provided with the unlabeled full sample $\bX_{n} \defeq \set{\bbx_i}_{i=1}^{n}$. Under the standard setting of transductive learning \citep{El-Yaniv2009-TRC,TolstikhinBK14-local-complexity-TRC}, the training features $\bX_m$ of size $m$ are sampled uniformly from $\bX_n$ without replacement, and the remaining features  are the test features denoted by $\bX_u = \bX_n \setminus \bX_m$.  It follows by symmetry that $\bX_u$ is sampled uniformly from all the subsets of size $u$ from $\bX_n$ without replacement, and we let
$\bX_u = \set{\bX_{\bZ(i)}}_{i=1}^u$ with $\bZ$ being the vector containing the indices of the test features in $\bX_n$. It can be verified that
$\set{\bZ}$ is a subset of size $u$ sampled uniformly
from $[n]$ without replacement. $\bZ$, throughout this paper, is the first $u$ elements of a permutation of $[n]$ which is chosen uniformly at random,
and in Section~\ref{sec:roadmap-proofs-sample-alg} we describe an algorithm which samples $\bZ$ by sampling $u$ independent random variables. 
We define the following loss functions. For simplicity of notations, we let $g(i) = g(\bbx_i,y_i)$ or $g(i) = g(\bbx_i)$ for a function $g$ defined on $\bX_n \times \cY$ or $\bX_n$. Given a prediction function $f \colon \bX_n \to \RR$ and a loss function $\ell(\cdot,\cdot)$ defined on $\RR^2$, we write $\ell \circ f$ as $\ell_f$ and let $\ell_f(i) = \ell (f(\bbx_i),y_i)$ be the loss on the $i$-th data point.
Let $\cH$ be a class of functions defined on $\bX_n \times \cY$ and $h \in \cH$. For any nonempty set $\cA \subseteq [n]$, we define
$\cL_h(\cA) \defeq 1/\abth{\cA} \cdot \sum\limits_{i \in \cA} h(i)$.
Then the average loss,  the training loss, and the test loss
associated with $h$ are
$\cL_h([n])$,  $\cL_h(\overline{\set{\bZ}})$, and $\cL_h(\set{\bZ})$, respectively,
where $\overline{\set{\bZ}} = [n] \setminus \set{\bZ}$.
For simplicity of notations, these three losses are also denoted
as $\cL_n(h)$,
$\cL_m(h)$, and $\cL_u(h)$, respectively,  with $\set{\bZ}$ omitted. We also define
the average squared loss associated with $h$ as
 $T_n(h) \defeq \frac 1n \sum\limits_{i=1}^n h^2(i)$.
When $h = \ell_f$, $\cL_m(h)$ and $\cL_u(h)$ are the training loss and the test loss of the prediction function $f$.
Throughout this paper when there are no special notes, the expectations or probabilistic results of functions involving
$\cL_u(h)$ or $\cL_m(h)$ are over $\set{\bZ}$. For example, we have $\Expect{}{\cL_m(h)} = \Expect{}{\cL_u(h)} = \cL_{n}(h)$.

This paper studies the improved generalization bounds of transductive learning algorithms.
We assume that all the points in the full sample $\bX_n$ are distinct, and always consider separable function classes throughout this paper.
The training features together with their labels, $\set{y_i}_{i \in \overline{\set{\bZ}}}$, are given to the learner as a training set. We denote the labeled training set by $\bS_m \defeq \set{\pth{\bbx_i, y_i}}_{i \in \overline{\set{\bZ}}}$. The learner's goal is to predict the labels of the test features $\bX_u$ based on $\bS_m \bigcup \bX_u$.
We let $\min\set{u,m} \ge 2$ throughout this paper to avoid the discussions about trivial cases where $m=1$ or $u = 1$.


\section{Summary of Main Results}
\label{sec:summary-main-results}
We summarize our main results in this section, with the basic definitions introduced first.

\subsection{Basic Definitions}
We define the test-train process $g$ as the supremum of the empirical process of the gap between the test loss and the training loss over a function class $\cH$ defined on $\bX_n \times \cY$, that is,
\bal\label{eq:def-g}
g(\cH) \defeq \sup_{h \in \cH} \pth{\cL_u(h)-
\cL_m(h) }.
\eal
We then define Rademacher variables and Transductive Complexity (TC), and relate TC to the conventional inductive Rademacher complexity.
\begin{definition}[Rademacher Variables]
\label{def:rad-variables}
Let $\{\sigma_i\}_{i=1}^n$ be $n$ i.i.d. random variables such that $\Pr[\sigma_i = 1] = \Pr[\sigma_i = -1] = \frac{1}{2}$, and they are defined as the
Rademacher variables.
\end{definition}
The Transductive Complexity is defined below.
\begin{definition}[Transductive Complexity, or TC]
\label{def:TC}
Define
\bals
R^+_{u} h \defeq \cL_u(h) - \cL_n(h),
\quad
R^+_{m} h \defeq \cL_m(h) - \cL_n(h),
\eals
and $R^{-}_{u} h \defeq -R^+_{u} h$,
$R^-_{m} h \defeq -R^+_{m} h$.
The four types of Transductive Complexity (TC) of a function class $\cH$ are defined as
\bal\label{eq:TC-def}
\cfrakR^+_u(\cH) \defeq \Expect{}{\sup_{h \in \cH} R^+_{u} h}, \cfrakR^-_u(\cH) \defeq \Expect{}{\sup_{h \in \cH} R^-_{u} h}, \nonumber \\
\cfrakR^+_m(\cH) \defeq \Expect{}{\sup_{h \in \cH} R^+_{m} h}, \cfrakR^-_m(\cH) \defeq \Expect{}{\sup_{h \in \cH} R^-_{m} h}.
\eal
\end{definition}
The expectations in (\ref{eq:TC-def}) are over $\set{\bZ}$ where the subset $\bZ$ is sampled uniformly from $[n]$ without
replacement. We remark that the proposed TC is fundamentally different from the transductive version of the Rademacher complexity in \citep[Definition 1]{El-Yaniv2009-TRC} in the sense that our TC is defined on the random training set or test set, while the counterpart in \citep[Definition 1]{El-Yaniv2009-TRC} operates on the entire full sample.

Let $\bY^{(u)} = \set{Y_1,\ldots,Y_u}$ be sampled uniformly and independently from $[n]$ with replacement. Similarly, let $\bY^{(m)} = \set{Y_1,\ldots,Y_m}$ be sampled uniformly and independently from $[n]$ with replacement. The following theorem relates the TC defined in Definition~\ref{def:TC} to the conventional inductive Rademacher complexity.
\begin{theorem}
\label{theorem:TC-RC}
Let $\bsigma = \set{\sigma_i}_{i=1}^{\max\set{u,m}}$ be i.i.d. Rademacher variables. Define
$R^{(\textup{ind})}_{\bsigma,\bY^{(u)}} h \defeq \frac 1u \sum\limits_{i=1}^u \sigma_i h(Y_i)$ and $R^{(\textup{ind})}_{\bsigma,\bY^{(m)}} h \defeq \frac 1m \sum\limits_{i=1}^m \sigma_i h(Y_i)$. Then
\bal\label{eq:TRC-RC}
\max\set{\cfrakR^+_u(\cH),\cfrakR^-_u(\cH)} \le 2\cfrakR^{(\textup{ind})}_u(\cH),
\quad
\max\set{\cfrakR^+_m(\cH),\cfrakR^-_m(\cH)} \le 2\cfrakR^{(\textup{ind})}_m(\cH),
\eal
where
$\cfrakR^{(\textup{ind})}_u(\cH) \defeq \Expect{\bY^{(u)},\bsigma}{\sup_{h \in \cH} R^{(\textup{ind})}_{\bsigma,\bY^{(u)}} h},
\cfrakR^{(\textup{ind})}_m(\cH) \defeq \Expect{\bY^{(m)},\bsigma}{\sup_{h \in \cH} R^{(\textup{ind})}_{\bsigma,\bY^{(m)}} h}$.
\end{theorem}
\begin{remark}
$\cfrakR^{(\textup{ind})}_u(\cH)$ and $\cfrakR^{(\textup{ind})}_m(\cH)$ are the Rademacher complexity in the inductive setting.  It is remarked that (\ref{eq:TRC-RC}) indicates that the established symmetrization inequality of inductive Rademacher complexity also holds for the transductive complexity defined in Definition~\ref{def:TC}.  For simplicity of notations  we also write $\cfrakR^{(\textup{ind})}_u(\cH) = \Expect{}{\sup_{h \in \cH} R^{(\textup{ind})}_{\bsigma,\bY^{(u)}} h}$ and $\cfrakR^{(\textup{ind})}_m(\cH) = \E\bigg[\sup_{h \in \cH} R^{(\textup{ind})}_{\bsigma,\bY^{(m)}} h\bigg]$ .
\end{remark}
We define the sub-root function below, which will be extensively used for deriving new bounds based on our transductive local complexity to be introduced in
Section~\ref{sec:detailed-results}.

\begin{definition}[Sub-root function, {\citep[Definition 3.1]{bartlett2005}}]
\label{def:sub-root-function}
A function $\psi \colon [0,\infty) \to [0,\infty)$ is sub-root if it is nonnegative,
nondecreasing and if $\frac{\psi(r)}{\sqrt r}$ is nonincreasing for $r >0$.
\end{definition}

Throughout this paper, the fixed point of a sub-root function $\psi$ means its unique positive fixed point when $\psi\not\equiv0$, and it means zero when $\psi\equiv0$. To justify this convention, observe that for $0<s<t$,
$\psi(s)\le\psi(t)\le\sqrt{t/s}\psi(s)$.
These inequalities imply continuity on $(0,\infty)$. If $\psi\not\equiv0$, they also imply that $\psi(r)>0$ for every $r>0$ and that $\psi(r)/r$ is strictly decreasing from $+\infty$ to zero. Thus $\psi(r)=r$ has exactly one positive solution. If this solution is $r^*$, then $\psi(r)\le\sqrt{rr^*}$ for $r\ge r^*$ and $\psi(r)\ge\sqrt{rr^*}$ for $0<r\le r^*$. No continuity at zero is needed.

\subsection{Main Results}
Our main results are summarized as follows. This summary also features a high-level description of the ideas we have developed to obtain the detailed technical results in Section~\ref{sec:detailed-results} and Section~\ref{sec:applications}.
First, we establish the improved bound for the excess risk of empirical minimizers in transductive learning using a local complexity-based approach inspired by LRC~\citep{bartlett2005}. The resulting bound~(\ref{eq:conceptual-excess-risk-TLC}), with its formal version in Theorem~\ref{theorem:TLC-delta-ell-f-excess-risk-upper-bound}, is nearly consistent with the existing sharp excess risk bound for inductive learning~(\ref{eq:conceptual-risk-bound-inductive}) with only a logarithmic gap.
Two ingredients underpin this result: (1) a new concentration inequality for the process defined in (\ref{eq:def-g}) for general training and test sample sizes, and (2) the TLC framework, including the peeling strategy and surrogate variance operator, introduced in~\citep{yang2025a-concentration-sampling-without-replacement}. The new concentration inequality extends the range of sample sizes covered by that framework.
Our new concentration inequality for the test-train process in Theorem~\ref{theorem:main-inequality-TLC} uses the combinatorial
property of the test-train process for sampling without replacement and a proof strategy which applies product-coordinate entropy inequalities~\citep{Boucheron2003-concentration-entropy-method} to both the test-train process and its squared-function companion. To show the advantage of our concentration inequality for sampling without replacement over the representative existing work \citep{TolstikhinBK14-local-complexity-TRC}, we derive new concentration inequalities for the general supremum of empirical processes involving random variables (RVs) in the setting of sampling uniformly without replacement in Corollary~\ref{corollary:concentration-gu} in Section~\ref{sec:concentration-general-sup-empirical} as a direct consequence of Theorem~\ref{theorem:main-inequality-TLC}. The comparison in Section~\ref{sec:concentration-general-sup-empirical} identifies sample size regimes and complexity conditions under which these concentration inequalities improve upon the displayed bounds in \citep{TolstikhinBK14-local-complexity-TRC}.


Second, we advance the analysis of transductive learning using TLC in two important applications. As the first application, we obtain a nearly optimal TLC bound for a single static empirical risk minimizer in
Theorem~\ref{theorem:TLC-delta-ell-f-excess-risk-upper-bound-VC-dim} for realizable transductive learning over a $\set{0,1}$-valued function class of finite VC dimension $2 \le \dVC < \infty$ with $u \ge m \ge \dVC$.
At any fixed confidence level, the bound has order $\cO(\dVC \log(me/\dVC)/m)$, which matches the minimax lower bound in \citep[Theorem 3]{tolstikhin2016minimaxlowerboundsrealizable} of $\Theta(\dVC/m)$ up to a logarithmic factor.
The known ERM bound under a uniform random split in \citep[Theorem 7]{tolstikhin2016minimaxlowerboundsrealizable} has order
$\cO(\dVC \log(ne/\dVC)/m)$ at fixed confidence under the same setup with $u \ge m \ge \dVC-1$. Our TLC analysis improves the logarithmic complexity term from $\log n$ to $\log m$.
This improvement concerns the static empirical minimizer. Theorem~\ref{theorem:optimal-transductive-transfer} gives a bound without the remaining logarithmic factor for an appropriate randomized learner.
As the second application, we derive a new excess risk bound for transductive kernel learning in Theorem~\ref{theorem:TLC-kernel} in Section~\ref{sec:TKL}, which significantly improves upon the existing local complexity-based result of~\citep{TolstikhinBK14-local-complexity-TRC}.

\section{TLC Excess Risk Bound for Generic Transductive Learning}
\label{sec:detailed-results}
\subsection{Concentration Inequality for the Test-Train Process}
\label{sec:starting-concentration}
Given a function class $\cH$, we define the function class $\cH^2 \defeq \set{h^2 \mid h \in \cH}$ as the ``squared version'' of $\cH$. We then have the following concentration inequality for the test-train process $g$ over the function class $\cH$.

\begin{theorem}[Concentration inequality for the test-train process (\ref{eq:def-g}) for a general function class]
\label{theorem:main-inequality-TLC}
Let $\cH$ be a class of functions defined on $\bX_n \times \cY$ and for any $h \in \cH$, $0 \le \abth{h(i)} \le H_0$ for all $i \in [n]$ with a positive number $H_0$. Assume that $H_0 \ge 2\sqrt 2$, and there is a positive number $r > 0$ such that $\sup_{h \in \cH} T_n(h) \le r$. Then for every $x > 0$ and every $\delta \in (0,1)$, with probability at least $1-\exp(-x) - \delta$ over $\set{\bZ}$,
\bal\label{eq:concentration-gd-TLC}
g(\cH) &\le \Expect{}{g(\cH)}  +4\sqrt{\frac{10rx}
{N_{u,m,\delta}}} +2\sqrt{2} \inf_{\alpha > 0}
\pth{\frac{\cfrakR^+_{\min\set{u,m}}(\cH^2) }{\alpha} +
\frac{\alpha x}{N_{u,m,\delta}}}
+  \frac {4 H_0^2x}{N_{u,m,\delta}},
\eal
where $\cH^2 = \set{h^2 \mid h \in \cH}$,
\bal\label{eq:N-um-delta-def}
N_{u,m,\delta} \defeq \frac{\min\set{u,m}}
{\log_2 (4\min\set{u,m}/\delta)},
\eal
 $\cfrakR_u^+(\cdot), \cfrakR_m^+(\cdot)$ are the Transductive Complexity defined in (\ref{eq:TC-def}).
\end{theorem}
For a technical reason we need $H_0 \ge 2\sqrt 2$ in Theorem~\ref{theorem:main-inequality-TLC}, which is always satisfied by setting $H_0 = \max\{2\sqrt 2, \allowbreak \sup_{h \in \cH, i \in [n]} \abth{h(i)} \}$ for every bounded function class on the full sample.

\noindent \textbf{Key Innovations in the Proof of Theorem~\ref{theorem:main-inequality-TLC}.} The proof of Theorem~\ref{theorem:main-inequality-TLC} is based on the combinatorial property of the test-train process revealed in Lemma~\ref{lemma:uniform-draw-diff} (and similarly Lemma~\ref{lemma:uniform-draw-diff-tE} of the appendix), which indicates that the change of the test-train loss for every loss function $h$, $\cL_u(h)-\cL_m(h)$, is always a difference of $h$ between a pair of two elements. Such a property is used to derive the upper bound for the upper variance of $g(\cH)$, $V_+(g)$, defined in (\ref{eq:V+-def}). Such an upper bound for $V_+(g)$ involves another empirical process over the class $\cH^2$. The proof for $m\ge u$ combines separate bounds for these two upper variances with the probability of $\Omega(Q)$ through product-coordinate entropy tensorization. Theorem~\ref{theorem:main-inequality-TLC} is a direct consequence of two intermediate results deferred to the appendix, Theorem~\ref{theorem:concentration-g-m-greater-u} and Theorem~\ref{theorem:concentration-tildeg-u-greater-m}, which are the variants of Theorem~\ref{theorem:main-inequality-TLC} under the condition that $m \ge u$ and $u \ge m$, respectively.

\noindent \textbf{Roadmap of the Proof of Theorem~\ref{theorem:main-inequality-TLC}.} When $m \ge u$, Theorem~\ref{theorem:main-inequality-TLC} is Theorem~\ref{theorem:concentration-g-m-greater-u}. We bound the upper variances of $g(\bd)$ and $t_u(\bd,\cH^2)$ separately on $\Omega(Q)$ and its complement. An independent uniform relabeling of the sampled indices makes the resulting empirical processes independent of the event $\Omega(Q)$, so that the probability estimate in Lemma~\ref{lemma:zd-pi-qi-repeat-number} enters the entropy bounds. These two bounds then yield concentration of $g$ by the product-coordinate entropy method. The detailed proof is presented in Section~\ref{sec:proof-concentration-test-train-process-m-greater-u} of the appendix.

\noindent \textbf{Significant Improvement Over \citep{yang2025a-concentration-sampling-without-replacement}.} It is worthwhile to mention that Theorem~\ref{theorem:main-inequality-TLC} covers the main result of \citep{yang2025a-concentration-sampling-without-replacement}, that is, \citep[Theorem 3.1]{yang2025a-concentration-sampling-without-replacement}, as a special case under the unbalanced regime that either
$m \ggg u^2$ or $u \gg m^2$. In strong contrast to \citep{yang2025a-concentration-sampling-without-replacement} where the results are obtained under the restrictive conditions that $m \ggg u^2$ or $u \gg m^2$, all the results of this paper are free of such restrictions. More detailed discussions are deferred to Section~\ref{sec:special-case-prelim-version-appendix} of the appendix.

When the function class $\cH$ is nonnegative and still bounded by $H_0$, the operator $T_n$ used in
Theorem~\ref{theorem:main-inequality-TLC}
can be replaced with $\cL_n$, leading to the following theorem. While we can directly apply Theorem~\ref{theorem:main-inequality-TLC} with $r = H_0 \sup_{h \in \cH} \cL_n(h)$ to this case, by following the argument in the proof of Theorem~\ref{theorem:main-inequality-TLC} with the operator $\cL_n$, we can have better constants as shown in Theorem~\ref{theorem:main-inequality-TLC-nonnegative-func-class}, which involves the TC of the original function class $\cH$, $\cfrakR^+_{\min\set{u,m}}(\cH)$, instead of the TC of the squared function class $\cH^2$, $\cfrakR^+_{\min\set{u,m}}(\cH^2)$, in Theorem~\ref{theorem:main-inequality-TLC}. Moreover, Theorem~\ref{theorem:main-inequality-TLC-nonnegative-func-class} does not require $H_0 \ge 2\sqrt 2$.
\begin{theorem}[Concentration inequality for the test-train process (\ref{eq:def-g}) for nonnegative function class]
\label{theorem:main-inequality-TLC-nonnegative-func-class}
Let $\cH$ be a class of functions defined on $\bX_n \times \cY$ and for any $h \in \cH$, $0 \le h(i) \le H_0$ for all $i \in [n]$ with a positive number $H_0$. Assume that there is a positive number $r > 0$ such that $\sup_{h \in \cH} \cL_n(h) \le r$. Then for every $x > 0$ and every $\delta \in (0,1)$, with probability at least $1-\exp(-x) - \delta$ over $\set{\bZ}$,
\bal\label{eq:concentration-gd-TLC-nonnegative-func-class}
g(\cH) &\le \Expect{}{g(\cH)}  +4\sqrt{\frac{2H_0 rx}
{N_{u,m,\delta}}} + 2\inf_{\alpha > 0}
\pth{\frac{\cfrakR^+_{\min\set{u,m}}(\cH) }{\alpha} +
\frac{\alpha H_0 x}{N_{u,m,\delta}}}
+  \frac {4 H_0 x}{N_{u,m,\delta}},
\eal
where $N_{u,m,\delta} $ is defined in (\ref{eq:N-um-delta-def}),
$\cfrakR_u^+(\cdot), \cfrakR_m^+(\cdot)$ are the Transductive Complexity defined in (\ref{eq:TC-def}).
\end{theorem}

As a direct consequence of Theorem~\ref{theorem:main-inequality-TLC}, we have the following bounds for the empirical processes including
$\sup_{h \in \cH} \pth{\cL_u(h) - \cL_n(h)}$ and $\sup_{h \in \cH} \pth{\cL_n(h) - \cL_u(h)}$, and we will show the advantage of our bounds over those in the representative work \citep{TolstikhinBK14-local-complexity-TRC} in the next subsection.
\begin{corollary}
\label{corollary:concentration-gu}
Suppose the conditions of Theorem~\ref{theorem:main-inequality-TLC} hold.
Define $g^+_u(\cH) \defeq \sup_{h \in \cH} \pth{\cL_u(h) - \cL_n(h)}$, $g^-_u(\cH) \defeq \sup_{h \in \cH} (\cL_n(h) - \cL_u(h))$,
then for every $x > 0$ and every $\delta \in (0,1/2)$, with probability at least $1-2\exp(-x)-2\delta$ over $\set{\bZ}$,
\bal\label{eq:concentration-gu}
&\max\set{g^+_u(\cH) - \Expect{}{g^+_u(\cH)},g^-_u(\cH)- \Expect{}{g^-_u(\cH)}} \nonumber \\
&\le \frac mn \pth{4\sqrt{\frac{10rx}
{N_{u,m,\delta}}} +2\sqrt{2} \inf_{\alpha > 0}
\pth{\frac{\cfrakR^+_{\min\set{u,m}}(\cH^2) }{\alpha} +
\frac{\alpha x}{N_{u,m,\delta}}}
+  \frac {4 H_0^2x}{N_{u,m,\delta}}},
\eal
where $N_{u,m,\delta}$ is defined in (\ref{eq:N-um-delta-def}).
\end{corollary}

\subsection{Comparison for Concentration Inequality for Supremum of Empirical Process in the Setting of Sampling Uniformly Without Replacement}
\label{sec:concentration-general-sup-empirical}
It is noted that Corollary~\ref{corollary:concentration-gu} offers
concentration inequalities for supremum of empirical process involving RVs sampled uniformly without replacement.
We hereby compare our concentration inequality (\ref{eq:concentration-gu}) in Corollary~\ref{corollary:concentration-gu} with existing concentration inequalities for supremum of empirical process under the same setting in \citep{TolstikhinBK14-local-complexity-TRC}. There are two versions of such inequalities in \citep{TolstikhinBK14-local-complexity-TRC}, which are presented as follows. For the signed class considered here, we apply the cited results to $(h(i)-\cL_n(h))/(2H_0)$, which has mean zero under the uniform distribution on $[n]$ and takes values in $[-1,1]$.
For the first version, with probability at least $1-\exp(-t)$,
\bal\label{eq:existing-concentration-sup-empirical-process-1}
g^+_u(\cH)- \Expect{}{g^+_u(\cH)} \le 2\sigma \sqrt{\frac{2nt}{u^2}}, \quad  \textup{\citep[Theorem 1]{TolstikhinBK14-local-complexity-TRC}}
\eal
with the variance $\sigma^2 = 1/n \cdot \sup_{h \in \cH}\sum_{i=1}^n (h(i)-\cL_n(h))^2$.
For the second version, with probability at least $1-\exp(-t)$,
\bal\label{eq:existing-concentration-sup-empirical-process-2}
g^+_u(\cH) - \Expect{\bY^{(u)}}{\tilde g_u(\bY^{(u)},\cH)} \le  \sqrt{\frac{2\pth{\sigma^2+
4H_0\Expect{\bY^{(u)}}{\tilde g_u(\bY^{(u)},\cH)}}t}{u}} + \frac {2H_0t}{3u}, \quad  \textup{\citep[Theorem 2]{TolstikhinBK14-local-complexity-TRC}}
\eal
where $\tilde g_u (\bY^{(u)},\cH) \defeq \sup_{h \in \cH} \big(1/u \cdot \sum\limits_{i=1}^u
h(Y_i) - \cL_n(h)\big)$ is the supremum of empirical process with i.i.d. random variables $\set{\bY^{(u)}}$.
Because we always expect the deviation between $g^+_u$ and its expectation, the gap between $\Expect{\bY^{(u)}}{\tilde g_u}$ and $\Expect{}{g_u^+}$ is given by \textup{\citep[Lemma 3]{TolstikhinBK14-local-complexity-TRC}} as follows:
\bals
0 \le \Expect{\bY^{(u)}}{\tilde g_u(\bY^{(u)},\cH)} - \Expect{}{g^+_u(\cH)} \le \frac{4H_0u^2}{n}.
\quad  \textup{\citep[Lemma 3]{TolstikhinBK14-local-complexity-TRC}}
\eals
It follows from (\ref{eq:existing-concentration-sup-empirical-process-2}) and the above inequality that for the second version,
\bal\label{eq:existing-concentration-sup-empirical-process-2-repeat}
g^+_u(\cH) -  \Expect{}{g^+_u(\cH)} \le \sqrt{\frac{2\pth{\sigma^2+
4H_0\Expect{\bY^{(u)}}{\tilde g_u(\bY^{(u)},\cH)}}t}{u}} + \frac {2H_0t}{3u} + \frac{4H_0u^2}{n}.
\eal

As a result, for fixed $t>0$, the RHS of (\ref{eq:existing-concentration-sup-empirical-process-1}) diverges when $u = o(\sqrt n)$ and $\sigma$ is bounded away from zero, and the RHS of (\ref{eq:existing-concentration-sup-empirical-process-2-repeat}) diverges when $u = w(\sqrt n)$ as $u,n\to \infty$. In contrast, the RHS of our bound (\ref{eq:concentration-gu}) in Corollary~\ref{corollary:concentration-gu} converges to $0$ under many standard learning models by noting that (1) $\cfrakR^+_{\min\set{u,m}}(\cH^2)$ can be bounded by the inductive Rademacher complexity of $\cH^2$, $\cfrakR^{(\textup{ind})}_{\min\set{u,m}}(\cH^2)$,
using Theorem~\ref{theorem:TC-RC};
(2) the inductive Rademacher complexity $\cfrakR^{(\textup{ind})}_{\min\set{u,m}}(\cH^2)$ usually converges to $0$ at a fast rate, such as $\cO(\sqrt{1/\min\set{u,m}})$, for many standard learning models \citep{Bartlett2003}, when combined with the contraction property of the inductive Rademacher complexity, e.g., Theorem~\ref{theorem:contraction-RC} in the appendix.

\subsection{The First Bound by Transductive Local Complexity}
\label{sec:first-TLC-bound}
Using Theorem~\ref{theorem:main-inequality-TLC} and the peeling strategy in the proof of \citep[Theorem 3.3]{bartlett2005}, we have the following bound for the test-train process  involving the fixed point of a sub-root function as the upper bound for the TC of localized function classes.
\begin{theorem}\label{theorem:TLC}
Let $\cH$ be a class of functions defined on $\bX_n \times \cY$ containing the zero function. For any $h \in \cH$, $0 \le \abth{h(i)} \le H_0$ for all $i \in [n]$ with a positive number $H_0 \ge 2 {\sqrt 2}$.
Suppose $K_0 > 1$ is an arbitrary constant, and $\tT_n \colon \cH \to [0,\infty)$ is a functional such that $T_n(h) \le \tT_n(h)$ for all $h \in \cH$ and $\tT_n(0)=0$. Let $\psi_{u,m}$ be a sub-root function and let $r_{u,m}$ be the fixed point of $\psi_{u,m}$. Assume that for all $r \ge r_{u,m}$,
\bal
\psi_{u,m}(r) \ge \max\Bigg\{&\Expect{}{\sup_{h \in \cH,\tT_n(h) \le r} R^+_{u} h},
\Expect{}{\sup_{h \in \cH,\tT_n(h) \le r} R^-_{m} h}, \Expect{}{\sup_{h \in \cH,\tT_n(h) \le r} R^+_{\min\set{u,m}} h^2}\Bigg\}, \label{eq:TLC-cond-psi-general}
\eal
then for every $x>0$ and every $\delta \in (0,1)$, with probability at least $1-\exp(-x) -\delta$ over $\set{\bZ}$,
\bal\label{eq:TLC-bound-g-upper-bound}
\cL_u(h) \le \cL_m(h) + \frac{\tT_n(h)}{K_0} +c_0 r_{u,m}  + \frac{c_1x}{N_{u,m,\delta}}, \quad \forall h \in \cH,
\eal
where $N_{u,m,\delta}$ is defined in (\ref{eq:N-um-delta-def}), $c_0,c_1$ are absolute positive constants depending on $K_0 $, and $c_1$ also depends on $H_0$.
\end{theorem}
\begin{remark}
$\tT_n(\cdot)$ is termed a surrogate variance operator, and it is an upper bound for the usual variance operator $T_n(\cdot)$.
Each term inside the maximum operator on the RHS of (\ref{eq:TLC-cond-psi-general}) is the TC for a localized function class, $\set{h \in \cH  \colon \tT_n(h) \le r}$, where every function $h$ has its surrogate variance $\tT_n(h)$ bounded by $r$. In this sense, $\psi_{u,m}(r)$ is the upper bound for the TC of a localized function class, so we attribute the results of Theorem~\ref{theorem:TLC} to Transductive Local Complexity (TLC).
\end{remark}
Theorem~\ref{theorem:TLC} is the TLC-based bound for a bounded function class $\cH$ containing zero.
The following theorem states the TLC-based bound for a nonnegative function class based on Theorem~\ref{theorem:main-inequality-TLC-nonnegative-func-class}.
\begin{theorem}\label{theorem:TLC-nonnegative-func-class}
Let $\cH$ be a class of functions defined on $\bX_n \times \cY$ containing the zero function. For any $h \in \cH$, $0 \le h(i) \le H_0$ for all $i \in [n]$ with a positive number $H_0$. Suppose $K_0 > 1$ is an arbitrary constant. Let $\psi_{u,m}$ be a sub-root function and let $r_{u,m}$ be the fixed point of $\psi_{u,m}$.
Assume that for all $r \ge r_{u,m}$,
\bal
\psi_{u,m}(r) \ge \max\bigg\{&\Expect{}{\sup_{h \in \cH,\cL_n(h) \le r} R^+_{u} h},
 \Expect{}{\sup_{h \in \cH,\cL_n(h) \le r} R^-_{m} h}, \Expect{}{\sup_{h \in \cH,\cL_n(h) \le r} R^+_{\min\set{u,m}} h} \bigg\}, \label{eq:TLC-cond-psi-nonnegative}
\eal
then for every $x>0$ and every $\delta \in (0,1)$, with probability at least $1-\exp(-x) -\delta$ over $\set{\bZ}$,
\bal\label{eq:TLC-bound-g-upper-bound-nonnegative-func-class}
\cL_u(h) \le \frac{K_0+1}{K_0-1}\cL_m(h) + c'_0 r_{u,m}
+ \frac{c'_1x}{N_{u,m,\delta}}, \quad \forall h \in \cH,
\eal
where $N_{u,m,\delta}$ is defined in (\ref{eq:N-um-delta-def}), $c'_0,c'_1$ are absolute positive constants depending on $K_0 $, and $c'_1$ also depends on $H_0$.
\end{theorem}

\subsection{Sharp Excess Risk Bounds using Transductive Local Complexity (TLC) for Generic Transductive Learning}
\label{sec:TLC-bound-generic}

We now apply Theorem~\ref{theorem:TLC} to the transductive learning task introduced in Section~\ref{sec:transductive-basic-setup}, and derive improved bounds for the excess risk. Suppose we have a function class $\cF$ which contains all the prediction functions. We assume $0 \le \ell_f(i) \le L_0$ for all $f \in \cF$ and all $i \in [n]$ throughout this paper, and $L_0 \ge 2\sqrt 2$.
We define $\hat f_{u} \defeq
\argmin_{f \in \cF} \cL_u(\ell_f)$
as the oracle predictor with minimum loss on the test set $\bX_u$, and
$\hat f_{m} \defeq \argmin_{f \in \cF}
\cL_m(\ell_f)$
as empirical minimizer, that is, the predictor with minimum loss on the training set $\bX_m$. It is noted that both
$\hat f_{u} $ and $\hat f_{m}$ depend on the random vector $\bZ$ defined in Section~\ref{sec:transductive-basic-setup}, which is omitted for simplicity of notations. The excess risk of a prediction function $f \in \cF$ is then defined by
\bal\label{eq:excess-risk-def}
\cE(f) \defeq \cL_u(\ell_{f})-
\cL_u(\ell_{\hat f_{u}}).
\eal
We remark that the excess risk $\cE(f)$ can be arbitrarily close to the risk, $\cL_u(\ell_{f})- \inf_{g \in \cF }\cL_u(\ell_{g})$, if there is no minimizer of $\cL_u(\ell_{g})$ over $\cF$, by choosing $\hat f_{u}$ such that $\cL_u(\ell_{\hat f_{u}})$ is arbitrarily close to $\inf_{g \in \cF }\cL_u(\ell_{g})$. Moreover, the literature extensively studies the excess risk of the empirical minimizer, $\cE(\hat f_{m})$.

We define the function class $\Delta_{\cF} \defeq \set{h \colon h = \ell_{f_1} - \ell_{f_2}, f_1, f_2 \in \cF}$, which is a basic function class for the analysis of excess risk bounds.
For $h = \ell_{f_1} - \ell_{f_2} \in \Delta_{\cF}$, we use the surrogate variance operator introduced in~\citep{yang2025a-concentration-sampling-without-replacement}, defined in (\ref{eq:tTn-def}) as a functional $\tT_n \colon \Delta_{\cF} \to [0,\infty)$ such that $T_n(h) \le \tT_n(h)$, as claimed in Lemma~\ref{lemma:TLC-delta-ell-f}. As a result, we can apply Theorem~\ref{theorem:TLC}
to the functional class $\Delta_{\cF}$ with the surrogate variance operator $\tT_n(h)$ (\ref{eq:tTn-def}). The following assumption, which is the standard assumption adopted by existing local complexity-based methods for both inductive and transductive learning \citep{bartlett2005,TolstikhinBK14-local-complexity-TRC}, is introduced below for performance guarantee of transductive learning.
\begin{assumption}
[Main Assumption]
\label{assumption:main}
\begin{itemize}
\item[(1)] There is a function $f^*_n \in \cF$ such that
$\cL_n(\ell_{f^*_n}) = \inf_{f \in \cF} \cL_n(\ell_f)$.
\item[(2)] There is a positive constant $B$ such that for any $h \in \Delta^*_{\cF}$, $T_n(h) \le B \cL_n(h)$, where $\Delta^*_{\cF} \defeq \set{\ell_f - \ell_{f^*_n} \colon f \in \cF}$.
\end{itemize}
\end{assumption}
Since $\cL_n(h)\ge0$ for $h\in\Delta^*_{\cF}$, any finite constant satisfying Assumption~\ref{assumption:main}(2) can be enlarged to a positive one. Thus the positivity of $B$ does not restrict the assumption.

\begin{remark}
Assumption~\ref{assumption:main} is not restrictive. It is the standard assumption also used in \citep{TolstikhinBK14-local-complexity-TRC}. In addition, Assumption~\ref{assumption:main}(2) holds if $\cF$ is convex and the loss function $\ell(\cdot,\cdot)$ is Lipschitz continuous in its first argument with a quadratic uniform convexity condition on $\ell$. An example of such a loss function is $\ell(y',y) = (y'-y)^2$, when the function class
$\cF$ is convex and uniformly bounded~\citep{bartlett2005}.
\end{remark}
\begin{lemma}\label{lemma:TLC-delta-ell-f}
Suppose that Assumption~\ref{assumption:main} holds, and $K_0 > 1$ is a fixed constant.
 For $h = \ell_{f_1} - \ell_{f_2} \in \Delta_{\cF}$ with $f_1,f_2 \in \cF$,  let
\bal\label{eq:tTn-def}
\tT_n(h) \defeq \inf_{f_1,f_2 \in \cF \colon \ell_{f_1} - \ell_{f_2} = h} 2B \cL_n(\ell_{f_1} - \ell_{f^*_n}) + 2B \cL_n(\ell_{f_2} - \ell_{f^*_n}).
\eal
The zero function belongs to $\Delta_{\cF}$ by taking $f_1=f_2=f^*_n$, and $\tT_n(0)=0$. Then Theorem~\ref{theorem:TLC} holds with $T_n$ replaced by the surrogate variance operator $\tT_n$ and $\cH = \Delta_{\cF}$, $H_0 = L_0$.
\end{lemma}
We can have $\psi_{u,m}$ in Theorem~\ref{theorem:TLC} as the upper bound for inductive Rademacher complexities using Theorem~\ref{theorem:TC-RC}, leading to the following corollary.
\begin{corollary}\label{corollary:TLC-delta-ell-f-ind-rc}
Suppose Assumption~\ref{assumption:main} holds, fix $K_0>1$, and take $\cH=\Delta_{\cF}$ with the surrogate variance operator $\tT_n$ defined in (\ref{eq:tTn-def}). Let $\psi_{u,m}(r)$ be a sub-root function and let $r_{u,m}$ be its fixed point. If for all $r \ge r_{u,m}$,
\bal\label{eq:TLC-cond-u-delta-ell-f-ind-rc}
\psi_{u,m}(r) \ge 2\max\bigg\{&\Expect{}{\sup_{h \colon h \in \Delta_{\cF},\tT_n(h) \le r} R^{(\textup{ind})}_{\bsigma,\bY^{(u)}} h}, \Expect{}{\sup_{h \colon h \in \Delta_{\cF},\tT_n(h) \le r} R^{(\textup{ind})}_{\bsigma,\bY^{(m)}} h}, \nonumber \\
&\Expect{}{\sup_{h \colon h \in \Delta_{\cF},\tT_n(h) \le r} R^{(\textup{ind})}_{\bsigma,\bY^{(\min\set{u,m})}} h^2}\bigg\}.
\eal
Then for every $x>0$ and every $\delta \in (0,1)$, with probability at least $1-\exp(-x)-\delta$ over $\set{\bZ}$,
\bal\label{eq:TLC-bound-g-upper-bound-ind-rc}
\cL_u(h) \le \cL_m(h)+ \frac{\tT_n(h)}{K_0} +c_0 r_{u,m}+ \frac{c_1x}{N_{u,m,\delta}}, \quad \forall h \in \Delta_{\cF},
\eal
where $N_{u,m,\delta}$ and the constants $c_0,c_1$ are as in Theorem~\ref{theorem:TLC}. This follows from Theorem~\ref{theorem:TC-RC}, Lemma~\ref{lemma:TLC-delta-ell-f}, and Theorem~\ref{theorem:TLC}.
\end{corollary}
Applying Theorem~\ref{theorem:TLC} and Lemma~\ref{lemma:TLC-delta-ell-f} to the function class $\Delta_{\cF}$, we have the  excess risk bound for the empirical minimizer $\hat f_{m}$ in
Theorem~\ref{theorem:TLC-delta-ell-f-excess-risk-upper-bound}. The following theorem is needed for the proof of Theorem~\ref{theorem:TLC-delta-ell-f-excess-risk-upper-bound}.
\begin{theorem}\label{theorem:TLC-delta-star-ell-f}
Suppose that Assumption~\ref{assumption:main} holds. Let $\psi^*_{u,m}$ be a sub-root function and $r^*$ be the fixed point of $\psi^*_{u,m}$, and
$\cF^*_r \defeq \set{h \colon h \in\Delta^*_{\cF},B \cL_n(h) \le r}$.
Assume that for all $r \ge r^*$,
\bal
\psi^*_{u,m} (r) \ge  \max\Bigg\{&\Expect{}{\sup_{h \colon h \in\cF^*_r} R^-_{u} h}, \Expect{}{\sup_{h \colon h \in\cF^*_r} R^-_{m} h}, \Expect{}{\sup_{h \colon h \in\cF^*_r} R^+_{\min\set{u,m}} h^2}\Bigg\}. \label{eq:TLC-cond-um-delta-star-ell-f-psi-star}
\eal
Then for every $x>0$ and every $\delta \in (0,1/2)$, with probability at least
$1-2\exp(-x)-2\delta$ over $\set{\bZ}$,
\bal\label{eq:TLC-bound-delta-star-ell-f}
\cL_n(\ell_{\hat f_{ u}} - \ell_{f^*_n}) \le  c_2 \pth{r^* + \frac{x}{N_{u,m,\delta}}}, \quad \cL_n(\ell_{\hat f_{ m}} - \ell_{f^*_n}) \le  c_2 \pth{r^* + \frac{x}{N_{u,m,\delta}}},
\eal
where $N_{u,m,\delta}$ is defined in (\ref{eq:N-um-delta-def}), $c_2$ is an absolute positive constant which depends on $B$ and $L_0$.
\end{theorem}

\begin{theorem}\label{theorem:TLC-delta-ell-f-excess-risk-upper-bound}
Suppose that Assumption~\ref{assumption:main} holds, and $K_0 > 1$ is an arbitrary constant. Let $\psi_{u,m}$ be a sub-root function satisfying (\ref{eq:TLC-cond-psi-general})
and $r_{u,m}$ be the fixed point of $\psi_{u,m}$.
Let $\psi^*_{u,m}$ be a sub-root function satisfying (\ref{eq:TLC-cond-um-delta-star-ell-f-psi-star}) and $r^*$ be the fixed point of $\psi^*_{u,m}$.
Then for every $x>0$ and every $\delta \in (0,1/3)$, with probability at least $1-3\exp(-x)-3\delta$ over $\set{\bZ}$,
 the excess risk $\cE(\hat f_{ m})$ satisfies
\bal\label{eq:TLC-ell-f-excess-risk-upper-bound}
\cE(\hat f_{ m}) \le c_0 r_{u,m} +\frac{4Bc_2r^*}{K_0}
+  \frac{c_3x}{N_{u,m,\delta}}.
\eal
Here $c_3 = c_1 + 4B c_2/K_0$, and $c_0,c_1,c_2$ are the positive constants in Theorem~\ref{theorem:TLC} and
Theorem~\ref{theorem:TLC-delta-star-ell-f}.
\end{theorem}
We first present a roadmap of the proof of Theorem~\ref{theorem:TLC-delta-ell-f-excess-risk-upper-bound}, then provide the formal proof.

\smallskip
\noindent \textbf{Roadmap of the Proof of Theorem~\ref{theorem:TLC-delta-ell-f-excess-risk-upper-bound}.}
The proof has three main steps. First, the excess risk of the empirical minimizer is rewritten as the test loss of the difference of two loss functions,
\bals
\cE(\hat f_m)
=
\cL_u(\ell_{\hat f_m})-\cL_u(\ell_{\hat f_u})
=
\cL_u(\ell_{\hat f_m}-\ell_{\hat f_u}),
\eals
where $\hat f_u$ is the oracle predictor minimizing the test loss. Thus the relevant function is
$h=\ell_{\hat f_m}-\ell_{\hat f_u}\in \Delta_{\cF}$.

Second, Lemma~\ref{lemma:TLC-delta-ell-f} allows us to apply the TLC bound in Theorem~\ref{theorem:TLC} to the difference class $\Delta_{\cF}$ with the surrogate variance operator $\tT_n$. This gives a high-probability upper bound for
$\cL_u(h)$ in terms of $\cL_m(h)$, the fixed point $r_{u,m}$, and the two full-sample excess quantities
\bals
\cL_n(\ell_{\hat f_m}-\ell_{f^*_n})
\quad\text{and}\quad
\cL_n(\ell_{\hat f_u}-\ell_{f^*_n}).
\eals
The training term is nonpositive,
\bals
\cL_m(\ell_{\hat f_m}-\ell_{\hat f_u})\le 0,
\eals
by the optimality of $\hat f_m$ on the training set.

Third, Theorem~\ref{theorem:TLC-delta-star-ell-f} controls the two full-sample excess terms above by the second fixed point $r^*$ and the deviation term $x/N_{u,m,\delta}$. Substituting these two bounds into the TLC inequality and taking the union bound over the three high-probability events yields
\bals
\cE(\hat f_m)
\le
c_0 r_{u,m}
+
\frac{4Bc_2r^*}{K_0}
+
\frac{c_3x}{N_{u,m,\delta}},
\eals
with probability at least $1-3\exp(-x)-3\delta$.

\begin{proof}
Applying (\ref{eq:TLC-bound-g-upper-bound}) in Theorem~\ref{theorem:TLC} with
$h = \ell_{\hat f_{m}}  - \ell_{\hat f_{u}}$, $\cH = \Delta_{\cF}$, and setting the surrogate variance operator $\tT_n$ according to
(\ref{eq:tTn-def}) in Lemma~\ref{lemma:TLC-delta-ell-f}, we have
\bals
\cL_u(\ell_{\hat f_{m}}  - \ell_{\hat f_{u}}) \le \cL_m(\ell_{\hat f_{m}}  - \ell_{\hat f_{u}})
&+ \frac {2B}{K_0} \pth{\cL_n(\ell_{\hat f_{m}} - \ell_{f^*_n}) + \cL_n(\ell_{\hat f_{u}} - \ell_{f^*_n})} +c_0 r_{u,m}  + \frac{c_1x}{N_{u,m,\delta}}.
\eals
The upper bound for the excess risk, (\ref{eq:TLC-ell-f-excess-risk-upper-bound}), then follows by plugging the upper bounds (\ref{eq:TLC-bound-delta-star-ell-f}) for $\cL_n(\ell_{\hat f_{u}} - \ell_{f^*_n})$ and $\cL_n(\ell_{\hat f_{m}} - \ell_{f^*_n})$ from Theorem~\ref{theorem:TLC-delta-star-ell-f} to the above inequality, and noting that
$\cL_m(h) = \cL_m(\ell_{\hat f_{m}}  - \ell_{\hat f_{u}}) \le 0$ by the
optimality of $\hat f_{m}$.
\end{proof}

We remark that our excess risk bound (\ref{eq:TLC-ell-f-excess-risk-upper-bound}) and the existing excess risk bound (\ref{eq:local-complexity-existing-excess-risk-bound}) are derived under exactly the same Assumption~\ref{assumption:main}.
We also remark that (\ref{eq:TLC-ell-f-excess-risk-upper-bound}) is nearly consistent with the sharp bound for excess risk for inductive learning (\ref{eq:conceptual-risk-bound-inductive}), and there are only constant factors on the fixed points $r_{u,m}$ and $r^*$. In contrast, the existing local complexity-based excess risk bound (\ref{eq:local-complexity-existing-excess-risk-bound}) for transductive learning involves undesirable factors $n/u$ and $n/m$, which can make the bound (\ref{eq:local-complexity-existing-excess-risk-bound}) diverge under standard transductive learning models, such as the Transductive Kernel Learning (TKL). In Section~\ref{sec:TKL}, we have an instantiation of the TLC-based excess risk bound (\ref{eq:TLC-ell-f-excess-risk-upper-bound}) for TKL, significantly improving upon the current state-of-the-art.

\section{Proofs}
\label{sec:roadmap-proofs}

We first describe the sampling process of the test features $\bX_u$ and present the proof of Theorem~\ref{theorem:main-inequality-TLC} in Section~\ref{sec:roadmap-proofs-sample-alg}.
Theorem~\ref{theorem:main-inequality-TLC} is the direct consequence of Theorem~\ref{theorem:concentration-g-m-greater-u} for the case that $m \ge u$  and Theorem~\ref{theorem:concentration-tildeg-u-greater-m} for the case that $u \ge m$. The proof of Theorem~\ref{theorem:concentration-g-m-greater-u} along with its proof roadmap and the required lemmas is presented in Section~\ref{sec:proof-concentration-test-train-process-m-greater-u}. Following a similar proof strategy, the key ideas in the proof of Theorem~\ref{theorem:concentration-tildeg-u-greater-m} are described in Section~\ref{sec:proof-concentration-test-train-process-u-greater-m}.
We then present the proofs of Theorem~\ref{theorem:TLC} and Theorem~\ref{theorem:TLC-nonnegative-func-class} in Section~\ref{sec:proofs-theorem-TLC-TLC-nonnegative-func-class}. The proofs of all the other results of this paper are deferred to the appendix.
\subsection{Sampling Random Set Uniformly Without Replacement}
\label{sec:roadmap-proofs-sample-alg}

We specify the sampling process of $\bX_u$ as a random subset of $\bX_n$ of size $u$ sampled uniformly without replacement. The sampling strategy in \citep{El-Yaniv2009-TRC} is adopted to sample $u$ points from the full sample $\bX_n$ uniformly at random among all subsets of size $u$, which is described in Algorithm~\ref{alg:randperm}. The same strategy is also used in \citep{BARDENET2015-sampling-without-replacement} for sampling without replacement.

Let $\bd = \bth{d_1,\ldots,d_u} \in \NN^u$ be a random vector, where $\set{d_i}_{i=1}^u$ are $u$ independent random variables such that $d_i$ takes values in $[i:n]$ uniformly at random. Algorithm~\ref{alg:randperm}, which is adapted from \citep{El-Yaniv2009-TRC}, describes how to obtain $\bZ_{\bd} = \bth{\bZ_{\bd} (1),\ldots,\bZ_{\bd} (u)}^{\top} \in \NN^u$ as the first $u$ elements of a uniformly distributed permutation of $[n]$, so that $\set{\bZ_{\bd}}$ is a subset of size $u$ sampled uniformly from $[n]$ without replacement, and is also the set of indices of the test features sampled uniformly from $\bX_n$ without replacement.
As a result of the discussion above, the set $\bZ$ in all
the previous sections is in fact $\bZ = \bZ_{\bd}$, and the notation $\bZ_{\bd}$ explicitly includes $\bd$ indicating that $\bZ$ is
sampled by sampling $\bd$. Similarly, functions in terms of $\bZ$, such as the test-train process $g(\cH)$, can also
be denoted as functions of $\bd$, such as $g(\cH) = g(\cH, \bd)$ indicating the dependence on $\bd$. In the following text and also
the appendix, we omit the notation $\cH$ associated with $g$ if the function class $\cH$ is clear from the context but may still keep the notation $\bd$, that is, we write $g(\bd)$ as an abbreviation for $g(\cH,\bd)$.

Algorithm~\ref{alg:randperm} produces each ordered sample of $u$ distinct indices with probability $(n-u)!/n!$. Hence $\set{\bZ_{\bd}}$ is uniform over the $\binom{n}{u}$ subsets of $[n]$ of size $u$. Every event and random variable that depends on $\bd$ only through $\set{\bZ_{\bd}}$ therefore has exactly the same probability or expectation under $\bd$ as under a uniform $u$-subset. The concentration events for $g$, $\cL_u$, and $\cL_m$ below have this property. When Algorithm~\ref{alg:randperm} samples the training indices instead, their complement is likewise a uniform test subset.

\begin{algorithm}[!hbt]
\renewcommand{\algorithmicrequire}{\textbf{Input:}}
\renewcommand\algorithmicensure {\textbf{Output:} }
\caption{The RANDPERM Algorithm in \citep{El-Yaniv2009-TRC}, which obtains $\bZ_{\bd} \in \NN^u$ as the first $u$ elements of a uniformly distributed permutation of $[n]$ by sampling a vector $\bd$ of independent random variables: $\bd = \bth{d_1,\ldots,d_u}$.}
\label{alg:randperm}
\begin{algorithmic}[1]
\State $\bZ_{\bd} \leftarrow$ RANDPERM($u$)
\State \textbf{\bf input: } $u$
\State \textbf{\bf initialize: } $\bI \in \NN^n$, $\bI(j) = j$ for all $j \in [n]$. $\bd, \bZ_{\bd} \in \NN^u$ are initialized as zero vectors.

\State \textbf{\bf for } $i=1,\ldots,u$ \,\,\textbf{\bf do }
\\
\quad Sample $d_i$ uniformly from $[i:n]$. \\
\quad $\bd(i) = d_i$, $\bZ_{\bd}(i) = \bI(d_i)$. \\
\quad Swap the values of $\bI(i)$ and $\bI(d_i)$.
\State \textbf{\bf end for }
\State \textbf{\bf return} $\bZ_{\bd}$
\end{algorithmic}
\end{algorithm}

For the convenience of the proof of Theorem~\ref{theorem:main-inequality-TLC}, we run Algorithm~\ref{alg:randperm} to sample the indices of the smaller set among the test features $\bX_u$ and the training features $\bX_m$. When $m \ge u$, Algorithm~\ref{alg:randperm} is employed to sample the indices of the test features, $\bZ_{\bd}$. When $u \ge m$, we invoke function $\textup{RANDPERM}$ in Algorithm~\ref{alg:randperm} with input changed from $u$ to $m$, and then $\textup{RANDPERM}(m)$ returns the indices of the training features. Accordingly, we have the concentration inequality for the test-train process in Theorem~\ref{theorem:concentration-g-m-greater-u} for the case that $m \ge u$, and in Theorem~\ref{theorem:concentration-tildeg-u-greater-m} for the case that $u \ge m$.
\begin{theorem}\label{theorem:concentration-g-m-greater-u}
Under the conditions of Theorem~\ref{theorem:main-inequality-TLC}, if $m \ge u$, then for every $x > 0$ and every $\delta \in (0,1)$,
with probability at least $1-\exp(-x) - \delta$ over $\bd$,
\bal\label{eq:concentration-g-m-greater-u}
g(\bd) \le &\Expect{\bd}{g(\bd)} + 4\sqrt{\frac{10 \log_2 (\frac{4u}{\delta}) rx}{u}}
+2\sqrt{2} \inf_{\alpha > 0} \pth{\frac{\cfrakR^+_u(\cH^2) }{\alpha} + \frac{\alpha \log_2 (\frac{4u}{\delta}) x}{u}}  +  \frac {4\log_2 (\frac{4u}{\delta})
H_0^2x}{u}.
\eal
\end{theorem}
\begin{theorem}\label{theorem:concentration-tildeg-u-greater-m}
Under the conditions of Theorem~\ref{theorem:main-inequality-TLC}, if $u \ge m$, then for every $x > 0$ and every $\delta \in (0,1)$, with probability at least $1-\exp(-x)
-\delta$ over $\bd$,
\bal\label{eq:concentration-g-u-greater-m}
g(\bd) \le &\Expect{\bd}{g(\bd)}+ 4\sqrt{\frac{10 \log_2 (\frac{4m}{\delta})rx}{m}}
+2\sqrt{2} \inf_{\alpha > 0} \pth{\frac{\cfrakR^+_m(\cH^2)}{\alpha} + \frac{\alpha \log_2 (\frac{4m}{\delta}) x}{m}}  +  \frac {4 \log_2 (\frac{4m}{\delta}) H_0^2x}{m}.
\eal
\end{theorem}
Theorem~\ref{theorem:main-inequality-TLC} is then proved below as the direct consequence of Theorem~\ref{theorem:concentration-g-m-greater-u} and Theorem~\ref{theorem:concentration-tildeg-u-greater-m}.
\begin{proof}[\textbf{\textup{Proof of Theorem~\ref{theorem:main-inequality-TLC}}}]
We first note that $\Expect{\bd}{g(\bd)} = \Expect{\set{\bZ_{\bd}}}{g(\bd)}$.
(\ref{eq:concentration-gd-TLC}) follows by combining the upper bound (\ref{eq:concentration-g-m-greater-u})
in Theorem~\ref{theorem:concentration-g-m-greater-u} for the case that $m \ge u$ and the upper bound (\ref{eq:concentration-g-u-greater-m}) for the case that $u \ge m$ in Theorem~\ref{theorem:concentration-tildeg-u-greater-m}, along with relation between the
probabilistic result over $\bd$ and that over $\set{\bZ_{\bd}}$ discussed in Section~\ref{sec:roadmap-proofs-sample-alg}.
\end{proof}

We then present the proof of Theorem~\ref{theorem:concentration-g-m-greater-u}, and the proof of Theorem~\ref{theorem:concentration-tildeg-u-greater-m}, deferred to the appendix, follows a similar strategy. Let $X_1, X_2, \ldots, X_n$ be independent random variables taking values in a measurable space $\cX$, and let $X_1^n$ denote the vector of these $n$ random variables. We then introduce the definition of upper variance which plays a central role in the exponential version of the Efron-Stein inequality
\citep[Theorem 2]{Boucheron2003-concentration-entropy-method}. Let $f \colon \cX^n \to \RR$ be a measurable function, and we are concerned with concentration of the random variable $Z = f(X_1,X_2,\ldots, X_n)$. Let $X'_1, X'_2, \ldots, X'_n$ denote independent copies of $X_1, X_2, \ldots, X_n$, and we write
\bals
Z^{(i)} = f(X_1,\ldots,X_{i-1},X'_i,X_{i+1},\ldots,X_n).
\eals
We define the upper variance for $Z$ as
\bal\label{eq:V+-def}
V_+(Z \mid X_1^n) \defeq \Expect{}{\sum\limits_{i=1}^n \pth{Z-Z^{(i)}}^2 \indict{Z>Z^{(i)}}
\mid X_1^n}.
\eal
We also write $V_+(Z \mid X_1^n)$
as $V_+(Z)$ if $X_1^n$ is clear from the context.

\subsection{Proof of Theorem~\ref{theorem:concentration-g-m-greater-u}}
\label{sec:proof-concentration-test-train-process-m-greater-u}

We first introduce the necessary definitions for our proof. To apply the entropy inequality (\ref{eq:concentration-g-m-greater-u-product-entropy}) in Section~\ref{sec::math-tools}, we let $\bd' = [d'_1,\ldots,d'_u]$ be an independent copy of $\bd$, and $\bd^{(i)} = [d_1,\ldots,d_{i-1},d'_i, d_{i+1},\ldots,d_u]$ for every $i \in [u]$, so that $\bd^{(i)}$ differs from $\bd$ only at the $i$-th coordinate.

For a function class $\cH'$, we define
\bal
t_u(\bd,\cH') &\defeq \sup_{h \in \cH'} \pth{\cL_h(\set{\bZ_{\bd}}) - \cL_n(h)} \label{eq:tdu}.
\eal
We introduce the definition of a chain and the set $\Omega(Q)$ comprising every vector $\bd$ wherein the length of the longest chain is not less than $Q$.
\begin{definition}
\label{def:chain}
Let $N \in \NN$, $Q \in [N]$, and $\bv \in \NN^N$.
A set $\set{j_1,\ldots,j_{Q}} \subseteq [N]$, where $j_1 < j_2 < \ldots < j_Q$, is defined to be a chain of length $Q$ in $[N]$ associated with $\bv$ if $j_k = \bv(j_{k-1})$
holds for all $k \in [2:Q]$ when $Q \ge 2$.
\end{definition}
\begin{remark}
If $Q=1$, then the singleton set $\set{j}$ is a chain in $[N]$ for any
$j \in [N]$.
\end{remark}

\begin{definition}
For $Q \in [u]$, define
\bal\label{eq:omega-Q-def}
\Omega(Q) \defeq
\set{\bd \colon \textup{there exists a chain } \set{j_1,\ldots,j_{Q'}}
\textup{ in } [u] \textup{ associated with } \bd, Q' \in [Q:u]}
\eal
as the subset of $\bd$ for which the size of the maximum chain is not less than $Q$.
\end{definition}

We remark that
we can run two instances
of Algorithm~\ref{alg:randperm} to generate
$\bZ_{\bd}$ and $\bZ_{\bd^{(i)}}$ respectively.
In the first instance of Algorithm~\ref{alg:randperm},
$\bZ_{\bd}$ is generated where
the sampled vector is $\bd$. In the second instance of Algorithm~\ref{alg:randperm},
$\bZ_{\bd^{(i)}}$ is generated
where the sampled
random vector is $\bd^{(i)}$. The vector $\bI$ in
the first instance of Algorithm~\ref{alg:randperm} for
 generating $\bZ_{\bd}$ is
still denoted as $\bI$ without confusion, and the
vector $\bI$ in
the second instance of Algorithm~\ref{alg:randperm} for generating
$\bZ_{\bd^{(i)}}$ is denoted as $\bI^{(i)}$. In this paper, ``right after the $i$-th iteration of Algorithm~\ref{alg:randperm}'' means after the $i$-th iteration and before the $(i+1)$-th iteration of Algorithm~\ref{alg:randperm}. Moreover, when it is said ``for all $\bd$'', it means for all $\bd \in [n]^u$ where $\bd(i)$ is uniformly distributed over $[i:n]$ for all $i \in [u]$.

Let $\cH$ be the function class in Theorem~\ref{theorem:main-inequality-TLC}. For any $h \in \cH$, we define
\bal\label{eq:E-change-test-train-loss}
E(h,\bd,\bd^{(i)}) \defeq \cL_h(\set{\bZ_{\bd}}) -\cL_h(\overline{\set{\bZ_{\bd}}}) - \cL_h(\set{\bZ_{\bd^{(i)}}})  + \cL_h(\overline{\bZ_{\bd^{(i)}}})
\eal
as the change of the test-train loss for a particular function
$h$ if $\bd$ is changed to $\bd^{(i)}$. Then we have the following lemma showing the values of $E(h,\bd,\bd^{(i)}) $ under four specific cases. The key argument in the proof of this
lemma is  that there can be at most only one pair of different elements in $\set{\bZ_{\bd}}$
and $\set{\bZ_{\bd^{(i)}}}$. Because $\bd$ and $\bd^{(i)}$ only differ at the
$i$-th element, it can be verified that right after the $(i-1)$-th
iteration of both instances of Algorithm~\ref{alg:randperm}, we have $\bI(i) = \bI^{(i)}(i)$.
\begin{lemma}
\label{lemma:uniform-draw-diff}
For any $h \in \cH$, there are four cases for $E(h,\bd,\bd^{(i)})$ for any $i \in [u]$.
\begin{itemize}[leftmargin=40pt]
\item[Case 1:] $E(h,\bd,\bd^{(i)}) = \pth{\frac 1u + \frac 1m} \pth{h(\bZ_{\bd}(i)) - h(\bZ_{\bd^{(i)}}(q(i)))}$,

if $d_i \neq d'_i$, $q(i) \le u, p(i) > u$,
\item[Case 2:] $E(h,\bd,\bd^{(i)}) = \pth{\frac 1u + \frac 1m} \pth{h(\bZ_{\bd}(p(i))) - h(\bZ_{\bd^{(i)}}(i))}$,

if $d_i \neq d'_i$, $p(i) \le u, q(i) > u$,
\item[Case 3:]  $E(h,\bd,\bd^{(i)}) = \pth{\frac 1u + \frac 1m} \pth{h(\bZ_{\bd}(i)) - h(\bZ_{\bd^{(i)}}(i))}$,

if $d_i \neq d'_i$, $p(i) >u, q(i) > u$,

\item[Case 4:]  $E(h,\bd,\bd^{(i)}) = 0$,

if $d_i = d'_i$ or $p(i),q(i) \le u$.
\end{itemize}
Here
\bal
q(i) &\defeq \min\set{i' \in [i+1,u] \colon \bZ_{\bd^{(i)}}(i') = i_0},
\label{eq:qi}\\
p(i) &\defeq \min\set{i' \in [i+1,u]\colon \bZ_{\bd}(i') = i_0}, \label{eq:pi}
\eal
where $i_0 \defeq \bI(i) = \bI^{(i)}(i)$ is the $i$-th element of $\bI$ and $\bI^{(i)}$ right after the $(i-1)$-th
iteration of both instances of Algorithm~\ref{alg:randperm} which generate $\bZ_{\bd}$ and $\bZ_{\bd^{(i)}}$, respectively. In (\ref{eq:qi}) and (\ref{eq:pi}), we use the convention that the $\min$ over an empty set returns $+\infty$.
\end{lemma}
\begin{remark}
It follows from Claim~\ref{claim:singleton-pi-qi}, deferred to the appendix before the proof of
Lemma~\ref{lemma:uniform-draw-diff}, that
$q(i),p(i)$ are either finite integers in $[2:u]$ or $\infty$,
so that $q(i) > u$ or $p(i) > u$  indicates that $q(i) = \infty$
or $p(i) = \infty$.
\end{remark}
Lemma~\ref{lemma:uniform-draw-diff} shows that the change of the test-train loss for every loss function $h$, $\cL_u(h)-\cL_m(h)$, is always a difference of $h$ between a pair of two elements specified by $\bd$ and $\bd^{(i)}$, which is the combinatorial property used in~\citep{yang2025a-concentration-sampling-without-replacement}.

\noindent \textbf{Roadmap of the Proof of Theorem~\ref{theorem:concentration-g-m-greater-u}.}
We work directly with $g(\bd)$ and $t_u(\bd,\cH^2)$. Their conditional upper variances have one bound on $\Omega(Q)^c$ and a coarser bound on $\Omega(Q)$. By independently relabeling the indices returned by Algorithm~\ref{alg:randperm}, we can average the entropy inequalities over the relabeling while keeping the empirical processes independent of $\Omega(Q)$. The sharpened probability estimate in Lemma~\ref{lemma:zd-pi-qi-repeat-number} then gives the required logarithmic variance coefficients.

\begin{proof}[\textbf{Proof of Theorem~\ref{theorem:concentration-g-m-greater-u}}]
For a subset $A\subseteq[n]$ of size $u$, we use $g(A)$ and $t_u(A,\cH^2)$ for the expressions defining $g(\bd)$ and $t_u(\bd,\cH^2)$ with $\set{\bZ_{\bd}}$ replaced by $A$. In particular, $g(\bd)=g(\set{\bZ_{\bd}})$. Identify each $h\in\cH$ with the vector $(h(1),\ldots,h(n))\in\RR^n$. These vectors lie in the compact cube $[-H_0,H_0]^n$, so their closure is compact. For each fixed subset of $[n]$, the objectives defining $g$ and $t_u(\cdot,\cH^2)$ are continuous functions of this vector. Therefore, their suprema over $\cH$ equal maxima over the closure. The bounds $|h(i)|\le H_0$ and $T_n(h)\le r$ continue to hold on the closure, since $T_n(h)=n^{-1}\sum_{i=1}^n h^2(i)$ is continuous. The exchange identities and multiplicity bounds below depend only on these $n$ values and consequently extend to the closure. Thus any maximizer $h^*$ used below may be chosen from this closure.

First suppose $u\ge2$ and fix $Q\in[2:u]$. For each $i\in[u]$, let $i_0$ have the meaning in Lemma~\ref{lemma:uniform-draw-diff}. In particular, it may depend on $i$. For the four cases of that lemma, retain the notation
\bal\label{eq:set-A}
\cA_{h,p}\defeq\set{i\in[u]\colon \bd\textup{ and }\bd^{(i)}
\textup{ satisfy Case }p\textup{ of Lemma~\ref{lemma:uniform-draw-diff}}},\qquad p\in[4].
\eal

\noindent \textbf{Step 1: Conditional upper variances.}
Fix $\bd$ and let $h^*\in\cH$ attain $g(\bd)$. For every replacement $\bd^{(i)}$,
\bal\label{eq:concentration-g-m-greater-u-direct-diff}
\pth{g(\bd)-g(\bd^{(i)})}_+
\le \pth{E(h^*,\bd,\bd^{(i)})}_+.
\eal
This comparison holds for all $\bd$ and $\bd^{(i)}$, including replacements crossing the boundary of $\Omega(Q)$, because it involves $g$ itself. Lemma~\ref{lemma:uniform-draw-diff} and $(a-b)^2\le2(a^2+b^2)$ show that each nonzero summand in $V_+(g\mid\bd)$ is bounded by
$2(1/u+1/m)^2$ times the sum of the squared values of $h^*$ at the two exchanged indices.

The old endpoint of the one-exchange expression is $\bZ_{\bd}(i)$ in Cases 1 and 3 and $\bZ_{\bd}(p(i))$ in Case 2. By Lemma~\ref{lemma:zd-pi-upper-bound}, its squared values, summed over $i$, are at most
$Q\sum_{j=1}^u h^2(\bZ_{\bd}(j))$ when $\bd\notin\Omega(Q)$. The bound
$u\sum_{j=1}^u h^2(\bZ_{\bd}(j))$ holds for every $\bd$.
The new endpoint is $\bZ_{\bd^{(i)}}(i)$ in Cases 2 and 3 and is $i_0=\bZ_{\bd^{(i)}}(q(i))$ in Case 1. The two endpoints may coincide, in which case the corresponding difference is zero.

For the latter case, conditional on $\bd$, the draw $d'_i$ is uniform on $[i:n]$. After iteration $i$, the value $i_0$ occupies position $d'_i$. With the later draws fixed, exactly $u-i$ of the $n-i+1$ possible positions of that value lead to its selection during iterations $i+1,\ldots,u$. Consequently,
\bal\label{eq:concentration-g-m-greater-u-q-prob}
\Prob{q(i)\le u\mid\bd}=\frac{u-i}{n-i+1}\le\frac{u}{m}.
\eal
The value $i_0$ is fixed conditional on $\bd$. As shown in the proof of Lemma~\ref{lemma:zd-pi-qi-repeat-number}, the values $i_0$, as $i$ varies, have multiplicity at most $Q-1$ if $\bd\notin\Omega(Q)$, and at most $u$ in general. Hence, for any $h\in\cH$,
\bal\label{eq:concentration-g-m-greater-u-new-indices}
\Expect{}{\sum_{i\in\cA_{h,1}}h^2(\bZ_{\bd^{(i)}}(q(i)))\mid\bd}
&\le\frac{u}{m}\sum_{i=1}^u h^2(i_0)
\le\begin{cases}
\frac{(Q-1)un}{m}T_n(h),&\bd\notin\Omega(Q),\\
\frac{u^2n}{m}T_n(h),&\bd\in\Omega(Q),
\end{cases}\nonumber\\
\Expect{}{\sum_{i=1}^u h^2(\bZ_{\bd^{(i)}}(i))\mid\bd}
&\le\frac{un}{m}T_n(h).
\eal
The second inequality follows because, conditional on $\bd$, $\bZ_{\bd^{(i)}}(i)$ is uniform over the $n-i+1$ indices not selected before iteration $i$.

Using $(1/u+1/m)^2\le4/u^2$, $n/m\le2$, and $T_n(h^*)\le r$, the preceding estimates yield
\bal\label{eq:concentration-g-m-greater-u-direct-variance}
V_+(g\mid\bd)
&\le\begin{cases}
\frac{8Q}{u}\pth{t_u(\bd,\cH^2)+3r},
&\bd\notin\Omega(Q),\\[2mm]
8\pth{t_u(\bd,\cH^2)+4r},
&\bd\in\Omega(Q).
\end{cases}
\eal
For the second line, the bound before the final simplification is
$8\cL_{(h^*)^2}(\set{\bZ_{\bd}})
+8(2+2/u)T_n(h^*)$. The stated line follows from $u\ge2$.

We also bound the upper variance of $t_u(\bd,\cH^2)$. Let $h^*$ now attain this supremum. Since $h^2\ge0$, the same one-exchange description gives, whenever the difference is positive,
\bal
0<t_u(\bd,\cH^2)-t_u(\bd^{(i)},\cH^2)
\le\frac1u(h^*(j))^2,
\eal
where $j$ is the old endpoint in the corresponding case of Lemma~\ref{lemma:uniform-draw-diff}. Since $(h^*(j))^4\le H_0^2(h^*(j))^2$, the preceding multiplicity bounds for the old endpoint imply
\bal\label{eq:concentration-g-m-greater-u-square-variance}
V_+(t_u(\bd,\cH^2)\mid\bd)
\le\begin{cases}
\frac{QH_0^2}{u}\pth{t_u(\bd,\cH^2)+r},
&\bd\notin\Omega(Q),\\[2mm]
H_0^2\pth{t_u(\bd,\cH^2)+r},
&\bd\in\Omega(Q).
\end{cases}
\eal
Both right-hand sides are nonnegative because $t_u(\bd,\cH^2)\ge-r$.

If $u=1$, then $\Omega(1)$ is the whole sample space. The same exchange argument, with $\bZ_{\bd^{(1)}}(1)$ uniform on $[n]$, gives
$V_+(g\mid\bd)\le8(t_u(\bd,\cH^2)+2r)$ and
$V_+(t_u(\bd,\cH^2)\mid\bd)\le H_0^2(t_u(\bd,\cH^2)+r)$.
Thus the second lines of (\ref{eq:concentration-g-m-greater-u-direct-variance}) and
(\ref{eq:concentration-g-m-greater-u-square-variance}) also hold in this case with $Q=1$.

\noindent \textbf{Step 2: Averaging the entropy inequalities by a novel relabeling strategy.}
Let $\pi$ be an independent uniform permutation of $[n]$. For each fixed $\pi$, applying the preceding variance bounds to the relabeled class $\set{h\circ\pi:h\in\cH}$ is legitimate because $T_n(h\circ\pi)=T_n(h)$ and its elements have the same bounds $H_0$ and $r$. Conditional on $\bd$, the set $\pi\set{\bZ_{\bd}}$ is uniform among the $u$-subsets of $[n]$ when $\pi$ is random. Hence $\pi\set{\bZ_{\bd}}$ and $\bd$ are independent. For each fixed $\pi$, the distribution of $\pi\set{\bZ_{\bd}}$ is also uniform, which makes the entropy left-hand side independent of $\pi$.

For each fixed $\pi$, apply the entropy inequality (\ref{eq:concentration-g-m-greater-u-product-entropy}) in Section~\ref{sec::math-tools} to the $u$ independent coordinates of $\bd$, first with
$Z=g(\pi\set{\bZ_{\bd}})$ and then with
$Z=t_u(\pi\set{\bZ_{\bd}},\cH^2)$.
The left sides do not depend on $\pi$. Averaging the right sides over $\pi$ and using the independence just established turns each factor
$\indict{\bd\in\Omega(Q)}$ into $\Prob{\Omega(Q)}$ under either exponential tilt. From now on take
$Q=\min\{u,\lceil\log_2(4u)\rceil-1\}$ when $u\ge2$, and take $Q=1$ when $u=1$. Lemma~\ref{lemma:zd-pi-qi-repeat-number} gives
\bal\label{eq:concentration-g-m-greater-u-weight}
Q\Prob{\Omega(Q)^c}+u\Prob{\Omega(Q)}\le\log_2(4u).
\eal
For $u=1$, the left side equals $1$, which is also at most $\log_2(4u)$. Since $t_u(\bd,\cH^2)+r\ge0$, (\ref{eq:concentration-g-m-greater-u-direct-variance})--(\ref{eq:concentration-g-m-greater-u-weight}) give
\bal\label{eq:concentration-g-m-greater-u-entropy-g}
\lambda\frac{\Expect{\bd}{g(\bd)e^{\lambda g(\bd)}}}{\Expect{\bd}{e^{\lambda g(\bd)}}}
-\log\Expect{\bd}{e^{\lambda g(\bd)}}
\le\frac{8\log_2(4u)\lambda^2}{u}
\pth{\frac{\Expect{\bd}{t_u(\bd,\cH^2)e^{\lambda g(\bd)}}}
{\Expect{\bd}{e^{\lambda g(\bd)}}}+4r}
\eal
and
\bal\label{eq:concentration-g-m-greater-u-entropy-tu}
\lambda\frac{\Expect{\bd}{t_u(\bd,\cH^2)e^{\lambda t_u(\bd,\cH^2)}}}
{\Expect{\bd}{e^{\lambda t_u(\bd,\cH^2)}}}
-\log\Expect{\bd}{e^{\lambda t_u(\bd,\cH^2)}}
\le\frac{H_0^2\log_2(4u)\lambda^2}{u}
\pth{\frac{\Expect{\bd}{t_u(\bd,\cH^2)e^{\lambda t_u(\bd,\cH^2)}}}
{\Expect{\bd}{e^{\lambda t_u(\bd,\cH^2)}}}+r}.
\eal

\noindent \textbf{Step 3: Closing the moment bounds.}
Set $\theta=H_0^2\log_2(4u)/u$. Integrating
(\ref{eq:concentration-g-m-greater-u-entropy-tu}) from $0$ to $\lambda<1/\theta$, and using
$\Expect{\bd}{t_u(\bd,\cH^2)}=\cfrakR^+_u(\cH^2)$, yields
\bal\label{eq:concentration-g-m-greater-u-mgf-tu}
\log\Expect{\bd}{e^{\lambda t_u(\bd,\cH^2)}}
\le\frac{\lambda\pth{\cfrakR^+_u(\cH^2)+\theta\lambda r}}
{1-\theta\lambda}.
\eal
Indeed, (\ref{eq:concentration-g-m-greater-u-entropy-tu}) is equivalent to
\bal
\frac{d}{d\lambda}\left[
(1-\theta\lambda)\pth{
\frac1\lambda\log\Expect{\bd}{e^{\lambda t_u(\bd,\cH^2)}}+r}
\right]\le0.
\eal
As $\lambda\downarrow0$, the expression in brackets tends to
$\cfrakR^+_u(\cH^2)+r$. This proves
(\ref{eq:concentration-g-m-greater-u-mgf-tu}). Moreover,
$\cfrakR^+_u(\cH^2)\ge0$, because each function
$\cL_{h^2}(\set{\bZ_{\bd}})-T_n(h)$ has mean zero and
the expectation of their supremum is at least the supremum of their expectations.

The variational formula for relative entropy, applied to the probability measure with density
$e^{\lambda g(\bd)}/\Expect{\bd}{e^{\lambda g(\bd)}}$, gives
\bal\label{eq:concentration-g-m-greater-u-tilt}
\lambda
\frac{\Expect{\bd}{t_u(\bd,\cH^2)e^{\lambda g(\bd)}}}
{\Expect{\bd}{e^{\lambda g(\bd)}}}
\le
\lambda
\frac{\Expect{\bd}{g(\bd)e^{\lambda g(\bd)}}}
{\Expect{\bd}{e^{\lambda g(\bd)}}}
-\log\Expect{\bd}{e^{\lambda g(\bd)}}
+\log\Expect{\bd}{e^{\lambda t_u(\bd,\cH^2)}}.
\eal
Because $H_0^2\ge8$, we have $8\log_2(4u)/u\le\theta$.
Substituting (\ref{eq:concentration-g-m-greater-u-mgf-tu}) and
(\ref{eq:concentration-g-m-greater-u-tilt}) into
(\ref{eq:concentration-g-m-greater-u-entropy-g}), and collecting the entropy term on the left, gives
\bal\label{eq:concentration-g-m-greater-u-entropy-closed}
\lambda\frac{\Expect{\bd}{g(\bd)e^{\lambda g(\bd)}}}
{\Expect{\bd}{e^{\lambda g(\bd)}}}
-\log\Expect{\bd}{e^{\lambda g(\bd)}}
\le\frac{8\log_2(4u)\lambda^2
\pth{\cfrakR^+_u(\cH^2)+4r}}
{u(1-\theta\lambda)^2}.
\eal
Indeed, before the last simplification the denominator is
$(1-8\log_2(4u)\lambda/u)(1-\theta\lambda)$ and the numerator contains
$\cfrakR^+_u(\cH^2)+4r-3\theta\lambda r$.
A final integration, using
\bal
\frac{d}{d\lambda}\left[
\frac1\lambda\log
\Expect{\bd}{e^{\lambda(g(\bd)-\Expect{\bd}{g(\bd)})}}
\right]
\le\frac{8\log_2(4u)\pth{\cfrakR^+_u(\cH^2)+4r}}
{u(1-\theta\lambda)^2}
\eal
and the zero limit of the expression in brackets as
$\lambda\downarrow0$, gives
\bal\label{eq:concentration-g-m-greater-u-mgf-g}
\log\Expect{\bd}{e^{\lambda(g(\bd)-\Expect{\bd}{g(\bd)})}}
\le\frac{8\log_2(4u)\lambda^2
\pth{\cfrakR^+_u(\cH^2)+4r}}
{u(1-\theta\lambda)},\qquad 0<\lambda<1/\theta.
\eal
The Chernoff bound for (\ref{eq:concentration-g-m-greater-u-mgf-g}) therefore shows that, with probability at least $1-e^{-x}$,
\bal\label{eq:concentration-g-m-greater-u-direct-tail}
g(\bd)\le\Expect{\bd}{g(\bd)}
+2\sqrt{\frac{8\log_2(4u)
\pth{\cfrakR^+_u(\cH^2)+4r}x}{u}}
+\frac{H_0^2\log_2(4u)x}{u}.
\eal
Finally, $\log_2(4u)\le\log_2(4u/\delta)$,
$\sqrt{a+b}\le\sqrt a+\sqrt b$, and
\[
2\sqrt{2}\inf_{\alpha>0}
\pth{\frac{\cfrakR^+_u(\cH^2)}{\alpha}
+\frac{\alpha\log_2(4u/\delta)x}{u}}
=4\sqrt{\frac{2\cfrakR^+_u(\cH^2)\log_2(4u/\delta)x}{u}}.
\]
Thus (\ref{eq:concentration-g-m-greater-u-direct-tail}) implies
(\ref{eq:concentration-g-m-greater-u}). Its failure probability
$e^{-x}$ is no larger than $e^{-x}+\delta$.
\end{proof}

We now state the combinatorial lemmas used above and related concentration lemmas used in later arguments.
\begin{lemma}\label{lemma:zd-pi-qi-repeat-number}
Suppose  $m \ge u$ and $Q \in [2:u]$, and we define the following sets:
\bal
\Omega^{(1)} (Q) &\defeq \set{\bd \colon \textup{there exists
a subset } \cP \subseteq [u], \abth{\cP} \ge Q,
\textup{ s.t. } p(s) = p(t) \le u, \forall
s,t \in \cP} \label{eq:zd-pi-repeat-set}, \\
\Omega^{(2)}(Q) &\defeq \{\bd \colon \textup{there exists
a subset } \cP \subseteq [u], \abth{\cP} \ge Q,
\textup{ s.t. } q(s) \le u, \forall s \in \cP,
\nonumber \\
&\hspace{0.3in} \bZ_{\bd^{(s)}}(q(s)) = \bZ_{\bd^{(t)}}(q(t)), \forall
s,t \in \cP\} \label{eq:zd-qi-repeat-set},
\eal
where $q,p$ are defined in (\ref{eq:qi})-(\ref{eq:pi}).
The event $\Omega^{(2)}(Q)$ is understood on the joint probability space of
$\bd$ and the replacement draws $\bd'$, and $\Omega(Q)$ is lifted to that
space. The following inclusions hold for every realization of $\bd'$:
\bal\label{eq:zd-pi-qi-repeat-number-prob-Omega12}
\Omega^{(1)}(Q) \subseteq \Omega(Q), \quad
\Omega^{(2)}(Q) \subseteq \Omega(Q).
\eal
Furthermore,
\bal\label{eq:omega-Q-prob-bound}
\Prob{\Omega(Q)}
&\le \binom{u}{Q}
\pth{\frac1u\sum_{i=1}^u\frac1{n-i+1}}^{Q-1}
<\binom{u}{Q}\pth{\frac{\ln 2}{u}}^{Q-1},\nonumber\\
\Prob{\Omega(Q)}&\le\frac{u-Q+1}{2^{Q-1}}
<\frac{u}{2^{Q-1}}.
\eal
In particular, when $Q=2$, $\Prob{\Omega(2)}\le u^2/m$.
Moreover, for $u\ge2$, the choice
$Q=\min\{u,\lceil\log_2(4u)\rceil-1\}$ satisfies
\bal\label{eq:omega-Q-weighted-prob-bound}
Q\Prob{\Omega(Q)^c}+u\Prob{\Omega(Q)}
\le\log_2(4u).
\eal
\end{lemma}
\begin{remark}
It follows from (\ref{eq:zd-pi-qi-repeat-number-prob-Omega12}) and
(\ref{eq:omega-Q-prob-bound}) that
\bals
\Prob{\Omega^{(1)}(Q) \cup \Omega^{(2)}(Q)} \le
\Prob{\Omega(Q)} <\frac{u}{2^{Q-1}}.
\eals
\end{remark}

\begin{lemma}\label{lemma:zd-pi-upper-bound}
Suppose $\bd \notin \Omega(Q)$ where $ \Omega(Q)$ is defined in (\ref{eq:omega-Q-def}) and $Q \in [2:u]$. Then for any $h \in \cH$, we have
\bal
\sum\limits_{i \in \cA_{h,2} } \pth{h(\bZ_{\bd}(p(i)))}^2 &\le (Q-1) \sum\limits_{i=1}^u \pth{h(\bZ_{\bd}(i))}^2. \label{eq:zd-pi-upper-bound}
\eal
\end{lemma}

\begin{lemma}
\label{lemma:Zdi-random}
Suppose $\bd \notin \Omega(Q)$ where $ \Omega(Q)$ is defined in
(\ref{eq:omega-Q-def}) and $Q \in [2:u]$. Then for any $h \in \cH$, we have
\bal
\Expect{}{\sum\limits_{i =1}^u \pth{h(\bZ_{\bd^{(i)}}(i))}^2 \longmid \bd} &\le \frac{nu}{m} T_n(h), \label{eq:Zdi-random-u-i} \\
 \Expect{}{\sum\limits_{i \in \cA_{h,1} } \pth{h(\bZ_{\bd^{(i)}}(q(i)))}^2
\longmid \bd  } &\le\frac{2(Q-1)nu}{m}T_n(h) \label{eq:Zdi-random-u-q}.
\eal
\end{lemma}

\subsection{Key Idea in the Proof of Theorem~\ref{theorem:concentration-tildeg-u-greater-m}}
\label{sec:proof-concentration-test-train-process-u-greater-m}
The proof of Theorem~\ref{theorem:concentration-tildeg-u-greater-m} requires sampling $m$ elements from the full sample $\bX_n$ uniformly without replacement as the training features. To this end, let $\tbd = \bth{\td_1,\ldots,\td_m} \in \NN^m$ be a random vector, and
$\set{\td_i}_{i=1}^m$ are $m$ independent random variables such that $\td_i$ takes values in $[i:n]$ uniformly at random. If we invoke the function $\textup{RANDPERM}$ in Algorithm~\ref{alg:randperm} with the input changed from $u$ to $m$ and
the vector of independent random variables changed from $\bd$ to $\tbd$, then $\bZ_{\tbd} = \textup{RANDPERM}(m)$ consists of the first $m$ elements of a uniformly distributed permutation of $[n]$. We use $\overline{\bZ_{\tbd}} = [n] \setminus \set{\bZ_{\tbd}}$ to denote the indices not in $\set{\bZ_{\tbd}}$.

Theorem~\ref{theorem:concentration-tildeg-u-greater-m} states the concentration inequality for the supremum of the test-train process $g(\bd)$ when $u \ge m$. Many technical results in the proof of
Theorem~\ref{theorem:concentration-g-m-greater-u} are adapted to the proof of Theorem~\ref{theorem:concentration-tildeg-u-greater-m}. The major difference is that we consider a different test-train process $\tg$ for a technical reason. $\tg$ is expressed as
\bal\label{eq:test-train-g-td}
\tg(\tbd) = \sup_{h \in \cH}
\pth{\cL_h(\overline{\set{\bZ_{\tbd}}}) -\cL_h(\set{\bZ_{\tbd}})},
\eal
where the training features are $\set{\bbx_{\bZ_{\tbd}(i)}}_{i=1}^m$ and the test features are $\set{\bbx_i}_{i \in \overline{\set{\bZ_{\tbd}}}}$.
It is noted that $\tg$ can be viewed as a function of a random subset of size $u$, $\overline{\set{\bZ_{\tbd}}}$, sampled uniformly from $[n]$ without replacement.
$\tg$ is used in the proof of Theorem~\ref{theorem:concentration-tildeg-u-greater-m} to handle the case that $u \ge m$ due to the important fact that $\tg  \overset{\textup{dist}} = g$, where both sides have the same distribution over subset of size $u$ sampled uniformly from $[n]$ without replacement. To see this, we first note that similar to $\set{\bZ_{\bd}}$ introduced in
Section~\ref{sec:roadmap-proofs-sample-alg}, $\set{{\bZ_{\tbd}}}$ is a random set of size $m$ sampled uniformly from $[n]$ without replacement.
As a result, $\overline{\set{\bZ_{\tbd}}}$ and $\set{\bZ_{\bd}}$ are both subsets of size $u$ sampled uniformly from $[n]$ without replacement, so that $\tg  \overset{\textup{dist}} = g$ holds. This indicates that if
some concentration inequality about $\tg$ happens with certain probability over $\set{\bZ_{\tbd}}$, then the same concentration inequality about $g$ happens with the same probability over $\set{\bZ_{\bd}}$. To this end, we can study the concentration inequality for $g$ through the lens of $\tg$ and $\set{\bZ_{\tbd}}$, which is the main strategy for the proof of
Theorem~\ref{theorem:concentration-tildeg-u-greater-m} in the appendix.

\subsection{Proofs of Theorem~\ref{theorem:TLC} and Theorem~\ref{theorem:TLC-nonnegative-func-class}}
\label{sec:proofs-theorem-TLC-TLC-nonnegative-func-class}

Below are the definitions and lemmas useful for the proof of Theorem~\ref{theorem:TLC}.

 For $r > 0$, define the function class
\bal\label{eq:Gr-def}
\cH^{(r)} = \set{\frac{r}{w(h)} h \colon h \in \cH},
\eal
where $w(h) \defeq \min\set{r \lambda^k \colon k\in\set{0,1,2,\ldots}, r\lambda^k \ge \tT_n(h)}$ with $\lambda > 1$.

Define the supremum of the empirical process
\bal\label{eq:Ur+}
U_r^+ \defeq \sup_{s \in \cH^{(r)}} \pth{ \cL_u(s)
-\cL_m(s) },
\eal
which can be viewed as the test-train process restricted on the
function class $\cH^{(r)}$. We then have the following lemma.

\begin{lemma}\label{lemma:TLC-supp-lemma}
Fix $\lambda > 1$, $K_0 > 1$, and $r > 0$. If $U_r^+ \le \frac{r}{\lambda K_0 }$, then
\bal\label{eq:TLC-supp-lemma}
\cL_u(h) \le \cL_m(h)+ \frac{r}{\lambda K_0 } + \frac{\tT_n(h)}{K_0}, \quad \forall h \in \cH.
\eal
\end{lemma}
\begin{proof}
If $\tT_n(h) \le r$, then $w(h) = r$ and $s = \frac{r}{w(h)} h = h$. Therefore,
$U_r^+ \le \frac{r}{\lambda K_0} \Rightarrow \cL_u(s)-\cL_m(s) \le \frac{r}{\lambda K_0}$ and (\ref{eq:TLC-supp-lemma}) holds since $\tT_n(h) \ge 0$ for all $h \in \cH$.

If $\tT_n(h) > r$, then $w(h) = r\lambda^k$ with $\tT_n(h) \in (r\lambda^{k-1},r\lambda^k]$.
Again, it follows from $U_r^+ \le \frac{r}{\lambda K_0}$ that
\bals
\cL_u(s)-\cL_m(s)  \le \frac{r}{\lambda K_0}, s = \frac{h}{\lambda^k},
\eals
and we have
\bals
\cL_u(h)-\cL_m(h)  \le \frac{r\lambda^{k-1}}{K_0}
\le \frac{\tT_n(h)}{K_0},
\eals
and (\ref{eq:TLC-supp-lemma}) still holds.
\end{proof}

\begin{proof}[\textbf{\textup{Proof of Theorem~\ref{theorem:TLC}}}]
We first show that every $s \in \cH^{(r)}$ satisfies $T_n(s) \le r$.

Let $r$ be chosen such that $r \ge r_{u,m}$. For any $h \in \cH$,
let $s=\frac{r}{w(h)} h \in \cH^{(r)}$. Then we have $T_n(s) \le r$. To see this, if $\tT_n(h) \le r$, then $w(h) = r$ and $s = h$, so $T_n(s) \le \tT_n(h) \le r$. Otherwise, if $\tT_n(h)> r$, then $s = \frac{h}{\lambda^k}$ where $k$ is such that $\tT_n(h) \in (r\lambda^{k-1}, r\lambda^{k}]$. It follows that
$T_n(s) = \frac{T_n(h)}{\lambda^{2k}} \le\frac{\tT_n(h)}{\lambda^{2k}} \le \frac{ r\lambda^{k}}{\lambda^{2k}} \le r$.
It follows that $T_n(s) \le r$ for all $s \in \cH^{(r)}$.

We then prove (\ref{eq:TLC-bound-g-upper-bound}). Let $p \defeq \min\set{u,m}$. It follows by applying Lemma~\ref{lemma:td-gd-RC} in the appendix  to the function class
$\cH^{(r)}$ that
\bal\label{eq:TLC-U+-bound}
\Expect{\bd}{U_r^+} \le \cfrakR^+_u(\cH^{(r)}) + \cfrakR^-_m(\cH^{(r)}).
\eal
If $r_{u,m}=0$, the fixed-point convention after Definition~\ref{def:sub-root-function} gives $\psi_{u,m}\equiv0$. The localized classes contain zero, so the peeling argument below gives $\cfrakR^+_u(\cH^{(r)})=\cfrakR^-_m(\cH^{(r)})=0$ for every $r>0$. Since $U_r^+\ge0$, (\ref{eq:TLC-U+-bound}) then gives $U_r^+=0$ almost surely. Apply Lemma~\ref{lemma:TLC-supp-lemma} for a sequence of positive $r$ decreasing to zero. This proves (\ref{eq:TLC-bound-g-upper-bound}) in this case. Henceforth assume $r_{u,m}>0$.
Applying (\ref{eq:concentration-gd-TLC}) in
Theorem~\ref{theorem:main-inequality-TLC} with $\alpha = 1$ to the function class $\cH^{(r)}$, we obtain, with probability at least $1-\exp(-x)-\delta$,
\bal\label{eq:TLC-seg-U+}
U_r^+ \le  \Expect{\bd}{U_r^+} +
2\sqrt{2} \cfrakR^+_{p}(\cH^{(r)}_1 ) +
4\sqrt{\frac{10rx}
{N_{u,m,\delta}}}
+  \frac {(4 H_0^2+2\sqrt{2})x}{N_{u,m,\delta}},
\eal
where $\cH^{(r)}_1 \defeq \set{s^2 \colon s \in \cH^{(r)}}$.

Define the function class $\cH(x,y) \defeq \set{h \in \cH \colon x \le \tT_n(h) \le y}$. Put $T_0 \defeq \sup_{h \in \cH} \tT_n(h)$. If $T_0>r$, let $T$ be the smallest nonnegative integer such that $r \lambda^{T+1} \ge T_0$, with $T=\infty$ when $T_0=\infty$. If $T_0\le r$, set $T=-1$ and interpret the sums below as empty. Since the zero function belongs to every localized class, its Rademacher supremum is nonnegative. Splitting the supremum over the localized class and the disjoint shells, and taking the positive part of each shell supremum, gives
\bal\label{eq:TLC-Ru+-seg1}
\cfrakR^+_u(\cH^{(r)}) &\le \Expect{}{\sup_{h \in \cH(0,r)} R^+_{u} h} +
\sum_{t=0}^{T}\lambda^{-(t+1)}\Expect{}{\max\set{0,\sup_{h\in\cH\colon r\lambda^t<\tT_n(h)\le r\lambda^{t+1}}R^+_u h}} \nonumber \\
&\le \psi_{u,m}(r) + \sum_{t=0}^{T} \lambda^{-t} \psi_{u,m}(r \lambda^{t+1}) \le
\psi_{u,m}(r)  \pth{1+\lambda^{1/2} \sum_{t=0}^{T} \lambda^{-t/2}}.
\eal
The supremum over an empty shell contributes zero. On the $t$th shell, $r/w(h)=\lambda^{-(t+1)}$. The maximum in the first line is bounded by the supremum over $\set{h\in\cH\colon\tT_n(h)\le r\lambda^{t+1}}$, because that class contains zero. The sub-root property gives
$\psi_{u,m}(r \lambda^{t+1}) \le \lambda^{(t+1)/2} \psi_{u,m}(r)$.

Setting $\lambda = 4$ on the RHS of (\ref{eq:TLC-Ru+-seg1}), we have
\bal\label{eq:TLC-Ru+-seg2}
\cfrakR^+_u(\cH^{(r)}) \le5\psi_{u,m}(r) \le 5\sqrt{r r_{u,m}}.
\eal
The last inequality follows from $\psi_{u,m}(r) \le \sqrt{\frac{r}{r_{u,m}}}\psi_{u,m}(r_{u,m}) = \sqrt{rr_{u,m}}$ because $r \ge r_{u,m}$.

Similarly, we have
\bal\label{eq:TLC-Rm--seg1}
\cfrakR^-_m(\cH^{(r)}) \le 5\psi_{u,m}(r) \le 5\sqrt{r r_{u,m}}.
\eal
Following an argument similar to (\ref{eq:TLC-Ru+-seg1}), we have
\bal\label{eq:TLC-Ru+-seg3}
\cfrakR^+_p(\cH^{(r)}_1) &\le \Expect{}{\sup_{h \in \cH(0,r)} R^+_{p} h^2} +
\sum_{t=0}^{T}\lambda^{-2(t+1)}\Expect{}{\max\set{0,\sup_{h\in\cH\colon r\lambda^t<\tT_n(h)\le r\lambda^{t+1}}R^+_p h^2}} \nonumber \\
&\le \psi_{u,m}(r) + \sum_{t=0}^{T} \lambda^{-2t} \psi_{u,m}(r \lambda^{t+1}) \le
\psi_{u,m}(r)  \pth{1+\lambda^{1/2} \sum_{t=0}^{T} \lambda^{-3t/2}}.
\eal
The zero function also belongs to the squared class. Hence each positive shell supremum is bounded by the corresponding localized supremum of $R^+_p h^2$. The final inequality again uses the sub-root property.

Again, setting $\lambda = 4$ on the RHS of (\ref{eq:TLC-Ru+-seg3}), we have
\bal\label{eq:TLC-Ru+-seg4}
\cfrakR^+_p(\cH^{(r)}_1) &\le \frac {23}7\psi_{u,m}(r) \le \frac {23}7\sqrt{r r_{u,m}}.
\eal

It follows from (\ref{eq:TLC-U+-bound}), (\ref{eq:TLC-seg-U+}),
(\ref{eq:TLC-Ru+-seg2}), (\ref{eq:TLC-Rm--seg1}), and (\ref{eq:TLC-Ru+-seg4})  that
\bal\label{eq:TLC-seg3}
U_r^+ &\le \frac {70+46 {\sqrt 2}}7   \sqrt{r r_{u,m}} +
4\sqrt{\frac{10rx}
{N_{u,m,\delta}}}
+  \frac {(4 H_0^2+2\sqrt{2})x}{N_{u,m,\delta}} \le \sqrt{r} \pth{c'_0\sqrt{r_{u,m}}+c'_1\sqrt{\frac{x}{N_{u,m,\delta}}}}+\frac {c'_2x}{N_{u,m,\delta}} \defeq P(r),
\eal
where the constants $c'_0 = \frac {70+46 {\sqrt 2}}7$,
$c'_1 = 4 {\sqrt {10}}$, $c'_2 = 4 H_0^2+2\sqrt{2}$.

Let $r_0$ be the largest solution to $P(r)= \frac{r}{\lambda K_0}$ for the fixed $K_0 > 1$. We have
\bals
r_0 \le r_1 \defeq 2\lambda K_0 \pth{\lambda K_0
\pth{c'_0\sqrt{r_{u,m}}+c'_1\sqrt{\frac{x}{N_{u,m,\delta}}}}^2
+ \frac {c'_2 x}{N_{u,m,\delta}}} .
\eals
Then $r_1 \ge  r_{u,m}$. Setting $r = r_1$ in (\ref{eq:TLC-seg3}), we have
\bals
U_{r_1}^+ \le \frac{r_1}{\lambda K_0 } =
2\lambda K_0
\pth{c'_0\sqrt{r_{u,m}}+c'_1\sqrt{\frac{x}{N_{u,m,\delta}}}}^2
+ \frac {2c'_2 x}{N_{u,m,\delta}}.
\eals
It follows from Lemma~\ref{lemma:TLC-supp-lemma} that
\bal\label{eq:TLC-g-seg1}
\cL_u(h) &\le \cL_m(h) + \frac{r_1}{\lambda K_0}
+ \frac{\tT_n(h)}{K_0} =\cL_m(h)+ \frac{\tT_n(h)}{K_0}+
2\lambda K_0
\pth{c'_0\sqrt{r_{u,m}}+c'_1\sqrt{\frac{x}{N_{u,m,\delta}}}}^2
+ \frac {2c'_2 x}{N_{u,m,\delta}} \nonumber \\
&\le \cL_m(h) + \frac{\tT_n(h)}{K_0}+ c_0 r_{u,m}+ \frac{ c_1 x}{N_{u,m,\delta}}, \quad \forall h \in \cH,
\eal
where $c_0 = 4\lambda K_0 {c'_0}^2$, $ c_1 = 4\lambda K_0 {c'_1}^2 + 2c'_2$,
$\lambda= 4$, and the last inequality  follows from the Cauchy-Schwarz inequality.
(\ref{eq:TLC-bound-g-upper-bound}) is then proved by (\ref{eq:TLC-g-seg1}).


\end{proof}

\begin{proof}[\textbf{\textup{Proof of Theorem~\ref{theorem:TLC-nonnegative-func-class}}}]
Theorem~\ref{theorem:TLC-nonnegative-func-class} is proved using an argument similar to that in the proof of Theorem~\ref{theorem:TLC}. If $r_{u,m}=0$, the same zero-complexity argument and the nonnegative version of Lemma~\ref{lemma:TLC-supp-lemma} below prove the result by letting $r\downarrow0$. Thus assume $r_{u,m}>0$.
Let $p \defeq \min\set{u,m}$.
First, we define the following function class:
\bals
\tilde \cH^{(r)} = \set{\frac{r}{\tilde w(h)} h \colon h \in \cH},
\eals
where $\tilde w(h) \defeq \min\set{r \lambda^k \colon k\in\set{0,1,2,\ldots}, r\lambda^k \ge \cL_n(h)}$ with $\lambda > 1$. We also define
\bals
\tilde U_r^+ \defeq \sup_{s \in \tilde \cH^{(r)}} \pth{ \cL_u(s)
-\cL_m(s) }.
\eals
Similar to Lemma~\ref{lemma:TLC-supp-lemma}, for fixed $\lambda > 1$, $K_0 > 1$, and $r > 0$, if $\tilde U_r^+ \le \frac{r}{\lambda K_0 }$, then it
can be verified that
\bal\label{eq:TLC-supp-lemma-nonnegative-func-class}
\cL_u(h) \le \cL_m(h)+ \frac{r}{\lambda K_0 } +
\frac{\cL_n(h)}{K_0}, \quad \forall h \in \cH.
\eal
It can be verified that $\cL_n(h) \le r$ for all $h \in \tilde \cH^{(r)} $.
Applying (\ref{eq:concentration-gd-TLC-nonnegative-func-class})
in Theorem~\ref{theorem:main-inequality-TLC-nonnegative-func-class}
with $\alpha = 1$ to the function class $\tilde \cH^{(r)}$, we obtain, with probability at least $1-\exp(-x)-\delta$,
\bal
\label{eq:TLC-seg-U+-nonnegative-func-class}
\tilde U_r^+ &\le \Expect{}{\tilde U_r^+}  +4\sqrt{\frac{2H_0 rx}
{N_{u,m,\delta}}} +
2\cfrakR^+_{p}(\tilde \cH^{(r)})  +
\frac {6 H_0 x}{N_{u,m,\delta}} \nonumber \\
&\le \cfrakR^+_{u}(\tilde \cH^{(r)})+\cfrakR^-_{m}(\tilde \cH^{(r)})+4\sqrt{\frac{2H_0 rx}
{N_{u,m,\delta}}} +
2\cfrakR^+_{p}(\tilde \cH^{(r)})  +
\frac {6 H_0 x}{N_{u,m,\delta}},
\eal
where the last inequality follows from
$\Expect{}{\tilde U_r^+} \le \cfrakR^+_{u}(\tilde \cH^{(r)})+\cfrakR^-_{m}(\tilde \cH^{(r)})$ which follows from Lemma~\ref{lemma:td-gd-RC} in the appendix.
The zero function belongs to $\tilde\cH^{(r)}$ and to every localized class $\set{h\in\cH\colon\cL_n(h)\le s}$. Splitting $\cH$ into the shells $r\lambda^t<\cL_n(h)\le r\lambda^{t+1}$, taking the positive part of each shell supremum, and using $r/\tilde w(h)=\lambda^{-(t+1)}$ gives exactly the bound (\ref{eq:TLC-Ru+-seg1}) for each of the three Rademacher processes in (\ref{eq:TLC-cond-psi-nonnegative}), with $\tT_n(h)$ replaced by $\cL_n(h)$. Hence
\bals
\max\set{\cfrakR^+_{u}(\tilde \cH^{(r)}),\cfrakR^-_{m}(\tilde \cH^{(r)}),\cfrakR^+_p(\tilde \cH^{(r)}) } &\le  5\sqrt{r r_{u,m}}.
\eals
We then apply a similar argument in
the proof of Theorem~\ref{theorem:TLC} to
(\ref{eq:TLC-seg-U+-nonnegative-func-class}), and we have
\bal\label{eq:TLC-g-nonnegative-func-class-seg1}
\cL_u(h) \le \cL_m(h) + \frac{\cL_n(h)}{K_0}+ C'_0 r_{u,m}+ \frac{ C'_1 x}{N_{u,m,\delta}}, \quad \forall h \in \cH,
\eal
where $C'_0,C'_1$ are absolute positive constants depending on $K_0 $, and $C'_1$ also depends on $H_0$.
We note that $\cL_n(h) = m/n \cdot \cL_m(h) + u/n \cdot \cL_u(h)$. Then with
$K_0 > 1$,
(\ref{eq:TLC-bound-g-upper-bound-nonnegative-func-class})
 follows from (\ref{eq:TLC-g-nonnegative-func-class-seg1})
 with $c'_0 = C'_0 K_0/(K_0-1)$, $c'_1 = C'_1 K_0/(K_0-1)$.
\end{proof}

\section{Applications}
\label{sec:applications}
In this section, we apply TLC-based bounds in Section~\ref{sec:detailed-results} to realizable transductive learning over binary-valued function classes of finite VC-dimension in
Section~\ref{sec:transductive-learning-finite-VC-dim}, and TKL in Section~\ref{sec:TKL}. Realizable transductive learning is formally defined in Theorem~\ref{theorem:TLC-delta-ell-f-excess-risk-upper-bound-VC-dim} where there exists a prediction function that renders the ground truth labels for all the points in the full sample $\bX_n$. The proofs for the results of this section are deferred to the appendix.
\subsection{Transductive Learning Over Binary-Valued Function Classes}
\label{sec:transductive-learning-finite-VC-dim}
In this section, we provide a nearly optimal TLC bound for a single static empirical risk minimizer over a binary-valued function class $\cF$ with finite VC dimension using
Theorem~\ref{theorem:TLC-nonnegative-func-class}.
\begin{theorem}
[TLC bound for a static empirical risk minimizer in realizable transductive learning]
\label{theorem:TLC-delta-ell-f-excess-risk-upper-bound-VC-dim}
Suppose $\cF$ is a class of $\set{0,1}$-valued functions with VC-dimension
$2 \le \dVC < \infty$, $\ell_f(i) = (f(i)-y_i)^2 = \indict{f(i) \neq y_i}$ with $y_i \in \cY = \set{0,1}$ for all $i \in [n]$, and $u \ge m \ge \dVC$. Suppose there exists $f^* \in \cF$ such that $f^*(i) = y_i$ for all $i \in [n]$, and such a setup is termed realizable transductive learning.
Then for every $x >0$ and every $\delta \in (0,1)$, with probability at least $1 - \exp(-x)-\delta$ over $\set{\bZ}$,
\bal\label{eq:TLC-delta-ell-f-excess-risk-upper-bound-VC-dim}
\cL_u(\ell_{\hat f_m}) \le  c''_0 \frac{\dVC \log   (me/\dVC)}{m}
+ c'_1\frac{\pth{\log_2 (4m/\delta)}x}{m},
\eal
where $c''_0,c'_1$ are absolute positive constants.
\end{theorem}
\begin{sloppypar}
We note that for realizable transductive learning, we have $0 \le \cL_m(\ell_{\hat f_m})  \le \cL_m(\ell_{f^*}) =0$,
so that
$\cL_m(\ell_{\hat f_m})  = 0$ and the excess risk of $\hat f_m$ is $\cE(\hat f_{m}) = \cL_u(\ell_{\hat f_m})$. In fact, the excess risk of
any estimator $\hat f$ is $\cE(\hat f) =  \cL_u(\ell_{\hat f})$.
When $m>4$ and $\delta = 4/m$, it follows from (\ref{eq:TLC-delta-ell-f-excess-risk-upper-bound-VC-dim}) that
\bal\label{eq:TLC-delta-ell-f-excess-risk-upper-bound-VC-dim-concrete}
\cL_u(\ell_{\hat f_m}) \le  \Theta\pth{\frac{\dVC \log   (me/\dVC)}{m}}
\eal
holds with probability at least $1-4/m-p(m,\dVC)$
where $p(m,\dVC) \defeq \exp\big(-\dVC /(2\log \dVC) \big)$. 
Since
$p(m,\dVC) \to 0$ as $m \ge \dVC \to \infty$, (\ref{eq:TLC-delta-ell-f-excess-risk-upper-bound-VC-dim-concrete}) holds with high probability for large $\dVC$ by setting 
 $x = \dVC /(2\log \dVC)$ in (\ref{eq:TLC-delta-ell-f-excess-risk-upper-bound-VC-dim}).
\end{sloppypar}
It is shown in \citep[Theorem 3]{tolstikhin2016minimaxlowerboundsrealizable} that under the conditions of Theorem~\ref{theorem:TLC-delta-ell-f-excess-risk-upper-bound-VC-dim},
the transductive learning admits the following minimax lower bound for the excess risk, $\inf_{\hat f} \sup_{\bS_{m+u}} \cL_u(\ell_{\hat f})  \ge (\dVC-1)/(16m)$, where the infimum is taken over all transductive learning estimators $\hat f$ learned from $\bS_m \bigcup \bX_u$, and the supremum is taken over all
possible labeled full samples $\bS_{m+u} = \set{(\bbx_i, y_i)}_{i=1}^{m+u}$ which are realizable.
For empirical risk minimization under a uniform random split, \citep[Theorem 7]{tolstikhin2016minimaxlowerboundsrealizable} gives an upper bound
$\cL_u(\ell_{\hat f_m}) \le 2\big(\dVC\log(ne/\dVC)+\log(1/\delta)\big)/m$ which holds with probability at least $1-\delta$. This bound can be loose for arbitrarily large $n$. Theorem~\ref{theorem:TLC-delta-ell-f-excess-risk-upper-bound-VC-dim} improves its logarithmic complexity term from $\log n$ to $\log m$ through a TLC analysis of a single static empirical risk minimizer.
This addresses the corresponding question about the ERM bound raised in \citep{tolstikhin2016minimaxlowerboundsrealizable}, without asserting that a logarithmic factor is necessary for general transductive learning.
Theorem~\ref{theorem:TLC-delta-ell-f-excess-risk-upper-bound-VC-dim} shows that the excess risk bound (\ref{eq:TLC-delta-ell-f-excess-risk-upper-bound-VC-dim-concrete}) holds with high probability, which matches the minimax lower bound with only a logarithmic gap of $\log m$ due to our TLC-based analysis. We remark that our excess risk bound (\ref{eq:TLC-delta-ell-f-excess-risk-upper-bound-VC-dim-concrete}) is as sharp as that under the inductive learning setup, where the inductive LRC-based bound~\citep[Corollary 3.7]{bartlett2005} for learning a binary-valued function class of VC-dimension $\dVC$ is of the order $\Theta(\dVC\log(em/\dVC)/m)$, where $m$ is the size of the training data.

The contribution of Theorem~\ref{theorem:TLC-delta-ell-f-excess-risk-upper-bound-VC-dim} is the improvement from $\log n$ to $\log m$ in the logarithmic complexity term of the known ERM bound under a uniform random split, obtained through TLC.
The remaining $\log(me/\dVC)$ factor pertains to this TLC bound for the empirical minimizer. Theorem~\ref{theorem:optimal-transductive-transfer} gives a bound without this factor for an appropriate randomized learner.
We further clarify this contribution through a detailed discussion of the existing literature on inductive and transductive learning, both under the BVC setup and in more general settings, as follows. 
Section~\ref{sec:discussion-transductive-learning-finite-VC-dim} of the appendix provides more technical details for such a discussion. 

Earlier works, including \citep{Blumer1989,EHRENFEUCHT1989247,DEVROYE19951011}, established the \emph{inductive} minimax lower bound for the generalization error of the order $\Omega(\dVC/m)$ for inductive learning with training data of size $m$  over a function class with VC dimension $\dVC$. The nearly optimal \emph{inductive}
generalization error rate of  $\dVC \log (1/\delta)/m$ has been achieved by the classical one-inclusion hypergraph algorithm by Haussler, Littlestone,
and Warmuth~\citep{Haussler1994-vc-dim-binary-class} (as the extended version of \citep{Haussler1988-vc-dim-binary-class}) under the BVC setup, where $\delta \in (0,1)$ is a small number such that the PAC guarantee holds with probability at least $1-\delta$. Under the same BVC setup, \citep{Hanneke16-optimal-sample-PAC} achieves the optimal generalization error bound of $\cO(\dVC/m)$ at any fixed confidence level, which matches both the inductive lower bound~\citep{Blumer1989,EHRENFEUCHT1989247,DEVROYE19951011} and the transductive lower bound~\citep[Theorem 3]{tolstikhin2016minimaxlowerboundsrealizable}.
Aden-Ali et al.~\citep{Aden-AliCSZ23-optimal-PAC-without-uniform-convergence}, using a reverse-martingale
argument~\citep{wu2023expected}, show that one can obtain an optimal high-probability cumulative error bound for the one-inclusion predictor. For realizable classification with a finite set of class labels, \citep[Theorem 2.2]{Aden-AliCSZ23-optimal-PAC-without-uniform-convergence} gives a risk bound of order $(\lceil\mu_{\cH}(m)\rceil+\log(2/\delta))/m$ with probability at least $1-\delta$, where $\mu_{\cH}(m)$ is the one inclusion hypergraph density discussed in Section~\ref{sec:one-include-graph-LOO}. In the binary case, $\lceil\mu_{\cH}(m)\rceil\le\dVC$ yields a bound of order $(\dVC+\log(2/\delta))/m$.

The cited PAC guarantees in~\citep{Haussler1994-vc-dim-binary-class,Hanneke16-optimal-sample-PAC,Aden-AliCSZ23-optimal-PAC-without-uniform-convergence} are stated for inductive risk, and the classical transductive one inclusion guarantee concerns LOO error. These formulations do not prevent an optimal bound under a general uniform random split. In particular, the forward and reverse martingale argument admits a finite population adaptation for majority aggregation of predictors trained on prefixes of a random ordering of the training sample. Under the BVC setup, this gives a bound of order $(\dVC+\log(1/\delta))/m$ for $u\ge m\ge\dVC\ge2$, with probability at least $1-\delta$ over the split and the learner's randomization, as explained in Section~\ref{sec:gap-transductive-inductive}. The argument analyzes an aggregated learner and does not require uniform control of the test-train process. Theorem~\ref{theorem:optimal-transductive-transfer} gives a separate construction with the same order of bound using an exact sampling reduction. Theorem~\ref{theorem:TLC-delta-ell-f-excess-risk-upper-bound-VC-dim} instead concerns a single static empirical risk minimizer analyzed through TLC.

On the other hand, ~\citep{Haussler1994-vc-dim-binary-class} proves that the transductive LOO error is bounded by the VC dimension $\dVC$ of the function class under the BVC setup, leading to an upper bound of $\dVC/n$ for the expectation of the test loss in the special case where the test data size is $u=1$ (see (\ref{eq:LOO-error-Lu}) in Section~\ref{sec:Haussler1988-transductive-LOO}),
and $n = m+u = m+1$ is the full sample size.
This shows that the $\dVC/m$ scale was already present in the one-inclusion/LOO literature. We emphasize, however, that the transductive minimax lower bound of  $\Theta(\dVC/m)$ in~\citep[Theorem~3]{tolstikhin2016minimaxlowerboundsrealizable} is proved in a different transductive regime, namely $u \ge m \ge \max\{9,\dVC-1\}$ and $\dVC \ge 2$, rather than in the LOO regime $u=1$. This observation shows that the expected $\Theta(\dVC/m)$ rate was already identified in~\citep{Haussler1994-vc-dim-binary-class} for the restrictive transductive LOO setup. In comparison, Theorem~\ref{theorem:TLC-delta-ell-f-excess-risk-upper-bound-VC-dim} analyzes a single static empirical risk minimizer through TLC under a uniform random split with $u \ge m \ge \dVC$, giving a nearly optimal bound within a logarithmic factor of the minimax lower bound at fixed confidence.

\subsection{TLC Excess Risk Bound for Transductive Kernel Learning}
\label{sec:TKL}
\subsubsection{Background in RKHS and Kernel Learning}

We denote by $\cH_{K}$ the Reproducing Kernel Hilbert Space (RKHS) associated with
$K$, where $K \colon \cX \times \cX \to \RR$ is a positive definite kernel defined on the compact set $\cX \times \cX$ with $\cX \subseteq \RR^d$
and $\max_{\bx \in \cX} K(\bx,\bx) = \Theta(1) < \infty$. We denote by
\bals
\cH_{\bX_n} \defeq \overbar{\set{\sum\limits_{i=1}^n
K(\cdot,\bbx_i) \alpha_i \longmid \set{\alpha_i}_{i=1}^n
\subseteq \RR}}
\eals
the usual RKHS spanned by $\set{K(\cdot,\bbx_i)}_{i=1}^n$
on the full sample $\bX_n = \set{\bbx_i}_{i=1}^n$.

It is noted that \citep{El-Yaniv2009-TRC} considers the function class where the output of any function in that class on the full sample $\bX_n$ can be expressed by $\bK \balpha$ for some vector $\balpha \in \RR^n$. Formally,
\citep{El-Yaniv2009-TRC} considers the function class $\cH_{\bX_n}(\mu)$   for Transductive Kernel Learning (TKL), where $\cH_{\bX_n}(\mu)$ denotes a ball centered at $0$ with radius $\mu$ in the Hilbert Space $\cH_{\bX_n}$, that is,
\bal\label{eq:TLC-function-class-kernel-HXn-mu}
\cH_{\bX_n}(\mu) = \set{f \in \cH_{\bX_n} \colon f = \sum\limits_{i=1}^n K(\cdot,\bbx_i) \alpha_i, \balpha = \bth{\alpha_1,\ldots,\alpha_n}^{\top}, \balpha^{\top}\bK\balpha \le \mu^2}.
\eal
For any $f \in \cH_{\bX_n}(\mu)$, because $f \in \cH_K$ and $\norm{f}{\cH_K}^2 = \balpha^{\top}\bK\balpha \le \mu^2$, we have $\cH_{\bX_n}(\mu) \subseteq \cH_K(\mu)$, where $\cH_K(\mu) \defeq \set{f \in \cH_K \colon \norm{f}{\cH_K} \le \mu}$.  In TKL, the empirical minimizer $\hat f_{m}$ and the oracle predictor $\hat f_{u}$ are defined using the function class $\cF = \cH_{\bX_n}(\mu)$.

\subsubsection{Results}
\label{sec:TKL-results}
Our improved bound for the excess risk $\cE(\hat f_{m})$ for TKL is based on Assumption~\ref{assumption:main} (1) and Assumption~\ref{assumption:Lipschitz-loss-Tn-f-Ln-ellf}. Assumption~\ref{assumption:Lipschitz-loss-Tn-f-Ln-ellf} is presented as follows.
\begin{assumption}
\label{assumption:Lipschitz-loss-Tn-f-Ln-ellf}
\begin{itemize}[leftmargin=18pt]
\item[(1)] The loss function $\ell(\cdot,\cdot)$ is $L$-Lipschitz in its first argument on the prediction range of $\cF$, that is, $|\ell(f(\bx),y) - \ell(g(\bx),y)|\le L \abth{f(\bx)-g(\bx)}$ for all $f,g \in \cF$ and all $(\bx,y)$ in the full sample.
\item[(2)] There is a constant $B'$ such that for any $f \in \cF$, $T_n\pth{f - f^*_n} \le B' \cL_n\pth{\ell_f - \ell_{f^*_n}}$.
\end{itemize}
\end{assumption}
It can be verified that Assumption~\ref{assumption:Lipschitz-loss-Tn-f-Ln-ellf} implies Assumption~\ref{assumption:main} (2), so that the former is stronger than the latter. We will present our new excess risk bound using Assumption~\ref{assumption:main} (1) and  Assumption~\ref{assumption:Lipschitz-loss-Tn-f-Ln-ellf}, which are exactly the same assumptions adopted by
\citep{TolstikhinBK14-local-complexity-TRC} for the excess risk bound for TKL. The following theorem, Theorem~\ref{theorem:TLC-kernel} as our main result for TKL, is obtained by applying Theorem~\ref{theorem:TLC}, Lemma~\ref{lemma:TLC-delta-ell-f}, Corollary~\ref{corollary:TLC-delta-ell-f-ind-rc}, and Theorem~\ref{theorem:TLC-delta-star-ell-f}. In particular,
the excess risk bound for TKL in (\ref{eq:TLC-kernel-excess-loss}) of Theorem~\ref{theorem:TLC-kernel} is proved by the same
strategy as that for the generic excess risk bound (\ref{eq:TLC-ell-f-excess-risk-upper-bound}) in Theorem~\ref{theorem:TLC-delta-ell-f-excess-risk-upper-bound}, so that (\ref{eq:TLC-kernel-excess-loss}) is an instantiation of the TLC-based excess risk bound (\ref{eq:TLC-ell-f-excess-risk-upper-bound}) for TKL.
\begin{theorem}\label{theorem:TLC-kernel}
Suppose that Assumption~\ref{assumption:main} (1) and  Assumption~\ref{assumption:Lipschitz-loss-Tn-f-Ln-ellf} hold with $\cF = \cH_{\bX_n}(\mu)$, and $K$ is a  positive definite kernel on $\cX \times \cX$. Suppose that for all $f \in \cH_{\bX_n}(\mu)$, $0 \le \ell_f(i) \le L_0$ for all $i \in [n]$, $L_0 \ge 2\sqrt 2$, and $K_0 > 1$ is an arbitrary constant.
Then for every $x > 0$ and every $\delta \in (0,1/3)$, with probability at least $1-3\exp(-x)-3\delta$ over $\set{\bZ}$,
we have the excess risk bound
\bal\label{eq:TLC-kernel-excess-loss}
&\cE(\hat f_{ m}) \le
c_5 \pth{\min_{0 \le Q \le n} r(u,m,Q) + \frac{\log_2 (4\min\set{u,m}/\delta) x}{\min\set{u,m}}},
\eal
where
\bals
r(u,m,Q) \defeq Q \pth{\frac{1}{u} + \frac{1}{m}} + \sqrt{\frac{\sum\limits_{q = Q+1}^{n}\hat \lambda_q}{u}}
+ \sqrt{\frac{\sum\limits_{q = Q+1}^{n}\hat \lambda_q}{m}},
\eals
$\hat \lambda_1 \ge \hat \lambda_2 \ldots \ge \hat \lambda_n \ge 0$ are the eigenvalues of $\bK/n \in \RR^{n \times n}$
with $\bK_{ij} = K(\bbx_i,\bbx_j)$ for all $i,j \in [n]$. $c_5$ is an absolute positive constant depending on $B',L_0,L,\mu$.
\end{theorem}

We now compare our excess risk bound (\ref{eq:TLC-kernel-excess-loss}) for TKL to the following excess risk bound obtained by the local complexity method for transductive learning
\citep{TolstikhinBK14-local-complexity-TRC}, which holds with high probability:
\bal\label{eq:local-complexity-excess-risk-TKL}
\cE(\hat f_{ m}) \le \Theta\pth{\frac nu r_m^*+ \frac nm r_u^*+\frac 1m +\frac 1u},
\eal
where
\bals
r_{s}^* \le \Theta\pth{\min_{0 \le Q \le s}\pth{ \frac{Q}{s} +  \sqrt{\frac{\sum\limits_{q = Q+1}^{n}\hat \lambda_q}{s}} }}, s = u \textup{ or } m.
\eals
It is emphasized that both our excess risk bound  (\ref{eq:TLC-kernel-excess-loss}) and (\ref{eq:local-complexity-excess-risk-TKL}) in \citep{TolstikhinBK14-local-complexity-TRC} are derived under exactly the same assumptions, Assumption~\ref{assumption:main} (1) and  Assumption~\ref{assumption:Lipschitz-loss-Tn-f-Ln-ellf}. It can be observed that our bound (\ref{eq:TLC-kernel-excess-loss}) is free of the undesirable factors of $n/u$ and $n/m$ in (\ref{eq:local-complexity-excess-risk-TKL}). We now show that the RHS of (\ref{eq:local-complexity-excess-risk-TKL}) together with the bound for $r_{s}^*$ for $s = u,m$ can diverge under well-studied scenario in kernel learning. It is well known from standard results on population and empirical inductive Rademacher complexities \citep{bartlett2005,Mendelson02-geometric-kernel-machine,RaskuttiWY14-early-stopping-kernel-regression} that when the full sample $\bX_n$ is drawn uniformly from the unit sphere with $\cX = \unitsphere{d-1}$ and the kernel $K$ is a dot-product kernel which admits a polynomial Eigenvalue Decay Rate (EDR) $\lambda_j \asymp j^{-2\alpha}$ for $\alpha > 1/2$ under the spherical uniform distribution on $\cX$, then with high probability,
\bals
\min_{0 \le Q \le n}\pth{ \frac{Q}{n} +  \sqrt{\frac{\sum\limits_{q = Q+1}^{n}\hat \lambda_q}{n}} } \asymp n^{-2\alpha/(2\alpha+1)}.
\eals
Such polynomial EDR happens for arccosine  kernel, such as $K(\bx,\bx') \asymp {\bx^{\top}}{\bx'}(\pi-\arccos({\bx^{\top}}{\bx'}))$ for all
$\bx,\bx' \in \cX = \unitsphere{d-1}$~\citep{BiettiM19,BiettiB21}, with $2\alpha = d/(d-1)$. In this well-studied case,
the RHS of (\ref{eq:local-complexity-excess-risk-TKL}) diverges with $u = o\pth{n^{1/(2\alpha+1)}}$ or $m = o\pth{n^{1/(2\alpha+1)}}$ when $u,m \to \infty$. On the other hand, for the bounded kernel considered here and fixed confidence parameters $x>0$ and $\delta\in(0,1/3)$, our excess risk bound (\ref{eq:TLC-kernel-excess-loss}) converges to $0$ as $u,m \to \infty$ because $\min_{0\le Q\le n}r(u,m,Q)\le r(u,m,0) \le \Theta\pth{1/\sqrt{u} + 1/\sqrt{m}}$.

Table~\ref{tab:comparison-sota} summarizes the main improvements of Theorem~\ref{theorem:TLC-delta-ell-f-excess-risk-upper-bound-VC-dim} and Theorem~\ref{theorem:TLC-kernel} over the closest state-of-the-art results. In particular, Theorem~\ref{theorem:TLC-delta-ell-f-excess-risk-upper-bound-VC-dim} improves the logarithmic complexity term in the known ERM bound under a uniform random split~\citep[Theorem~7]{tolstikhin2016minimaxlowerboundsrealizable} from $\log n$ to $\log m$ through TLC, while Theorem~\ref{theorem:TLC-kernel} removes the undesirable prefactors $n/u$ and $n/m$ in the current state-of-the-art transductive kernel learning bound~\citep{TolstikhinBK14-local-complexity-TRC} under the same structural assumptions as~\citep{TolstikhinBK14-local-complexity-TRC}.

\begin{table*}[t]
\centering
\caption{Comparison with representative state-of-the-art results.}
\label{tab:comparison-sota}
\renewcommand{\arraystretch}{1.18}
\footnotesize
\begin{tabular}{p{0.15\textwidth} p{0.25\textwidth} p{0.3\textwidth} p{0.22\textwidth}}
\hline
\textbf{Result} & \textbf{Setting} & \textbf{Excess Risk Bound} & \textbf{Main difference} \\
\hline
Tolstikhin--Lopez-Paz~\citep[Theorem~7]{tolstikhin2016minimaxlowerboundsrealizable}
&
Realizable VC classes with a static ERM under a uniform random split, $u\ge m\ge\dVC\ge2$
&
$\cL_u(\ell_{\hat f_m})
\lesssim
\bigl(\dVC\log(ne/\dVC)+\log(1/\delta)\bigr)/m$
&
Contains $\log n$ dependence. \\
\hline
Theorem~\ref{theorem:TLC-delta-ell-f-excess-risk-upper-bound-VC-dim} (Ours)
&
Realizable VC classes with a static ERM analyzed through TLC under a uniform random split, $u\ge m\ge\dVC\ge2$
&
$\cL_u(\ell_{\hat f_m})
\lesssim
\bigl(\dVC\log(me/\dVC)+x\log_2(4m/\delta)\bigr)/m$
with probability at least $1-e^{-x}-\delta$
&
Improves the logarithmic complexity term from $\log n$ to $\log m$ for a static ERM. \\
\hline
Tolstikhin et al.~\citep{TolstikhinBK14-local-complexity-TRC}
&
TKL; local-complexity analysis
&
$\cE(\hat f_m)
\lesssim
(n/u)r_m^*+(n/m)r_u^*+1/m+1/u$
&
Contains prefactors $n/u$ and $n/m$, which can be loose in unbalanced regimes. \\
\hline
Theorem~\ref{theorem:TLC-kernel} (Ours)
&
TKL under the same structural assumptions as~\citep{TolstikhinBK14-local-complexity-TRC}
&
$\cE(\hat f_m)
\lesssim
\min_{0\le Q\le n} r(u,m,Q)
+
\frac{\log_2(4\min\{u,m\}/\delta)x}{\min\{u,m\}}$
&
Removes the prefactors $n/u$ and $n/m$; converges in regimes where the prior bound may diverge. \\
\hline
\end{tabular}
\end{table*}

\section{Conclusion}
We extend the Transductive Local Complexity (TLC) framework of~\citep{yang2025a-concentration-sampling-without-replacement} to general training and test sample sizes. The extension uses concentration bounds obtained through chain analysis and entropy averaging under independent uniform relabeling. Combining these bounds with the peeling strategy, expectation decomposition, and surrogate variance operator of that framework gives improved excess risk bounds for transductive learning with bounded loss. As two applications, the TLC bound for a single empirical minimizer over a realizable binary class of finite VC dimension improves the logarithmic complexity dependence from $\log n$ to $\log m$, and the kernel bound extends the earlier kernel application to general sample sizes.



\newpage
\begin{appendix}

\section{Mathematical Tools}
\label{sec::math-tools}

We introduce the basic mathematical tools, such as
the entropy based concentration inequalities for empirical loss
of independent random variables~\citep{Boucheron2003-concentration-entropy-method}, which we used to develop the main results. Let $f \colon \cX^n \to \RR$ be a measurable function. We are concerned with concentration of the random variable $Z = f(X_1,X_2,\ldots, X_n)$. Let $X'_1, X'_2, \ldots, X'_n$ denote independent copies of $X_1, X_2, \ldots, X_n$, and recall that
\bals
Z^{(i)} = f(X_1,\ldots,X_{i-1},X'_i,X_{i+1},\ldots,X_n).
\eals
The upper variance for $Z$, $V_+(Z)$, is defined in (\ref{eq:V+-def}).
\begin{theorem}
{\rm (\citep[Theorem 5, Theorem 6]{Boucheron2003-concentration-entropy-method})}
\label{theorem:Boucheron2003-concentration}
Assume that $Z$ is bounded and that there exist constants $a > 0$ and $b > 0$ such that
\bals
V_+(Z) \le aZ + b.
\eals
Then for any $\lambda \in (0,1/a)$,
\bal\label{eq:Boucheron2003-concentration-thm5-logE}
\log \Expect{}{\exp\pth{\lambda(Z-\Expect{}{Z})}}
\le \frac{\lambda^2}{1-a\lambda}\pth{a \Expect{}{Z} + b}.
\eal
Furthermore, for all $t > 0$,
If
\bals
V_-(Z) \defeq \Expect{}{\sum\limits_{i=1}^n \pth{Z-Z^{(i)}}^2 \indict{Z<Z^{(i)}}
\mid X_1^n} \le v(Z)
\eals
holds for a nondecreasing function $v$, then for all $t > 0$,
\bals
\Prob{Z < \Expect{}{Z}-t} \le \exp\pth{\frac{-t^2}{4 \Expect{}{v(Z)}}}.
\eals
\end{theorem}

\begin{theorem}
({\citep[Theorem 2]{Boucheron2003-concentration-entropy-method}, the exponential version of the Efron-Stein inequality})
\label{theorem:Boucheron2003-concentration-thm2}
If $Z$ is bounded, then for all $\theta > 0$ and $\lambda \in (0,1/\theta)$,
\bal\label{eq:Boucheron2003-concentration-thm2}
\log \Expect{}{\exp\pth{\lambda\pth{Z - \Expect{}{Z}}}}
\le \frac{\lambda \theta}{1-\lambda \theta} \log\Expect{}{\exp\pth{\frac{\lambda V_+(Z)}{\theta}}}.
\eal
\end{theorem}

The following entropy inequality underlies the exponential Efron-Stein inequality \citep{Boucheron2003-concentration-entropy-method}. If $Z=f(X_1,\ldots,X_n)$ is bounded, then, for every $\lambda>0$,
\bal\label{eq:concentration-g-m-greater-u-product-entropy}
\lambda\frac{\Expect{}{Ze^{\lambda Z}}}{\Expect{}{e^{\lambda Z}}}
-\log\Expect{}{e^{\lambda Z}}
\le\lambda^2
\frac{\Expect{}{e^{\lambda Z}V_+(Z\mid X_1^n)}}{\Expect{}{e^{\lambda Z}}}.
\eal
Entropy tensorization followed by symmetrization in each coordinate proves the inequality using
$(e^{\lambda a}-e^{\lambda b})(a-b)\le
\lambda e^{\lambda a}(a-b)^2$ when $a>b$.


\begin{theorem}[Contraction property of inductive Rademacher complexity \citep{LedouxTal91-probability}]
\label{theorem:contraction-RC}
Suppose $g$ is Lipschitz continuous with $\abth{g(x) - g(y)} \le L \abth{x-y}$. Then
\bals
\cfrakR^{(\textup{ind})}_u(g \circ \cH) \le L \cfrakR^{(\textup{ind})}_u(\cH),
\quad \cfrakR^{(\textup{ind})}_m(g \circ \cH) \le L \cfrakR^{(\textup{ind})}_m(\cH).
\eals
The same bound holds conditionally on the sampled indices when the Lipschitz map in each Rademacher summand depends on that index, provided every map has Lipschitz constant at most $L$. This is the coordinatewise contraction form used for losses with different labels.
\end{theorem}

\begin{theorem}
[{\citep[Theorem 4]{Hoeffding1963-probability-inequalities},
\citep[Section D]{Gross-1001-2738}}]
\label{theorem:sample-w-wo-replacement-inequality}
Let $1\le n\le N$, and let $\set{X_i}_{i=1}^n$ and $\set{Y_i}_{i=1}^n$ be sampled from a population
$\set{c_i}_{i=1}^N \subseteq \cX \subseteq \RR^d$ without replacement and with
replacement respectively. Suppose $f\colon\RR^d\to\RR$ is continuous and convex. Then
\bal\label{eq:sample-w-wo-replacement-inequality}
\Expect{}{f\pth{\sum\limits_{i=1}^n X_i}}
\le \Expect{}{f\pth{\sum\limits_{i=1}^n Y_i}}.
\eal
\end{theorem}
\begin{remark}
As indicated in \citep{TolstikhinBK14-local-complexity-TRC},
while \citep[Theorem 4]{Hoeffding1963-probability-inequalities})
proves (\ref{eq:sample-w-wo-replacement-inequality}) only for real numbers,
 one can repeat the proof of
 \citep[Theorem 4]{Hoeffding1963-probability-inequalities})
 to show that (\ref{eq:sample-w-wo-replacement-inequality}) also holds for real vectors (e.g.,
see \citep[Section D]{Gross-1001-2738}).
\end{remark}
\begin{proposition}
[Concentration of a function class with finite VC-dimension, adapted from {\citep[Corollary 3.5]{bartlett2005}}]
\label{proposition:concentration-vc-class-ind}
Let $\cF$ be a class of $\set{0,1}$-valued functions defined on $\cX$ with VC-dimension $1\le\dVC\le n$. Let $\set{x_i}_{i=1}^n \subseteq \cX$ be an i.i.d. sample from a distribution $P$ over $\cX$. Then there exists an absolute positive constant $c$ such that for all $K_0>1$ and every $x > 0$, with probability at least $1-\exp(-x)$, simultaneously for all $f\in\cF$,
\bal\label{eq:concentration-vc-class-ind}
\Expect{P_n}{f(x)}  - \frac{K_0+1}{K_0}\Expect{x \sim P}{f(x)}
\le c K_0 \pth{\frac{\dVC \log (en/\dVC)}{n}+\frac xn}.
\eal
Here $P_n$ is the empirical distribution over $\set{x_i}_{i=1}^n$ with
$\Expect{P_n}{f(x)} = \frac{1}{n} \sum_{i=1}^n f(x_i)$.
\end{proposition}
\begin{proof}
If $\dVC\le n\le2\dVC$, the left-hand side of (\ref{eq:concentration-vc-class-ind}) is at most one and $\dVC\log(en/\dVC)/n$ is bounded below by a positive absolute constant. Thus the assertion is immediate after enlarging $c$. If $n>2\dVC$, apply the second assertion of \citep[Corollary 3.5]{bartlett2005} with the VC entropy calculation in the proof of \citep[Corollary 3.7]{bartlett2005}. That calculation gives a sub-root majorant with fixed point at most a constant times $\dVC\log(en/\dVC)/n$.
\end{proof}

\section{Detailed Proofs}
\label{sec:proofs}

We present the detailed proofs for the results of this paper, including Theorem~\ref{theorem:TC-RC}, Theorem~\ref{theorem:main-inequality-TLC-nonnegative-func-class}, Corollary~\ref{corollary:concentration-gu}, Lemma~\ref{lemma:TLC-delta-ell-f}, Theorem~\ref{theorem:TLC-delta-star-ell-f}, Theorem~\ref{theorem:TLC-delta-ell-f-excess-risk-upper-bound-VC-dim}, Theorem~\ref{theorem:TLC-kernel}, all the lemmas in
Section~\ref{sec:roadmap-proofs}, and Theorem~\ref{theorem:concentration-tildeg-u-greater-m}.

\subsection{Proof of Theorem~\ref{theorem:TC-RC}}
\begin{proof}
We prove the first bound in (\ref{eq:TRC-RC}).
We  let $\bY^{(u)} = \set{Y_1,\ldots,Y_u}$ be $u$ independent random variables with each $Y_i$ for $i \in [u]$ sampled uniformly from $[n]$ with replacement. Let ${\bY^{(u)}}' = \bth{Y'_1,\ldots,Y'_u}$ be independent copies of $\bY^{(u)}$, and $\bsigma = \set{\sigma_i}_{i=1}^{\max\set{u,m}}$ be i.i.d. Rademacher variables.
Let $\cH_0 = \set{h^{(0)}_j}_{j \ge 1} $ be a countable dense subset of the restrictions of $\cH$ to $[n]$, so that $\overline{\cH_0} = \overline{\cH}$ as subsets of $\RR^n$. We define $c_i \defeq 1/u \cdot \bth{h^{(0)}_j(i)-\cL_n(h^{(0)}_j)}_{j \in [M]} \in \RR^M$ for $i \in [n]$, and let $\set{Q_i}_{i \in [u]}$ and $\set{Q'_i}_{i \in [u]}$ be sampled from $\set{c_i}_{i \in [n]}$ without replacement and with replacement respectively.
Let $f(\bx) \defeq \max_{j \in [M]} \bx_j$ for $\bx \in \RR^M$.
Since $f$ is continuous and convex on $\RR^M$, Theorem~\ref{theorem:sample-w-wo-replacement-inequality} gives
 $\Expect{}{f\pth{\sum\limits_{i=1}^u Q_i}}
\le \Expect{}{f\pth{\sum\limits_{i=1}^u Q'_i}}$, which means that
\bals
\Expect{}{\max_{j \in [M]} \pth{\cL_u(h^{(0)}_j)-\cL_n(h^{(0)}_j)}} \le \Expect{\bY^{(u)}}{\max_{j \in [M]} \pth{\frac 1u \sum\limits_{i=1}^u h^{(0)}_j(Y_i)- \cL_n(h^{(0)}_j)}}.
\eals
 We note that both sequences,
\bals
\bigg\{\max_{j \in [M]}
\Big(\cL_u(h^{(0)}_j)-\cL_n(h^{(0)}_j)\Big)\bigg\}_{M \ge 1}
\textup{ and }
\bigg\{\max_{j \in [M]} \Big(\frac 1u \sum\limits_{i=1}^u h^{(0)}_j(Y_i)- \cL_n(h^{(0)}_j)\Big)\bigg\}_{M \ge 1},
\eals
are nondecreasing in terms of $M$. Letting $M \to \infty$, it then follows from the Levi's monotone convergence theorem and the fact that the first elements of both sequences are integrable that
\bals
\Expect{}{\sup_{h \in \cH_0} \pth{\cL_u(h) - \cL_n(h)}} \le \Expect{\bY^{(u)}}{\sup_{h \in \cH_0} \pth{\frac 1u \sum\limits_{i=1}^u h(Y_i)- \cL_n(h)}}.
\eals
Because $\cH_0$ is dense in $\cH$, we have
\bal\label{eq:TC-RC-seg1}
\Expect{}{\sup_{h \in \cH} \pth{\cL_u(h) - \cL_n(h)}} \le \Expect{\bY^{(u)}}{\sup_{h \in \cH} \pth{\frac 1u \sum\limits_{i=1}^u h(Y_i)- \cL_n(h)}}.
\eal

As a result, it follows from the standard proof of the symmetry
inequality for the inductive Rademacher complexity that
\bal\label{eq:TC-RC-seg2}
\cfrakR^+_u(\cH) &= \Expect{}{\sup_{h \in \cH} \pth{\cL_u(h) - \cL_n(h)}} \nonumber \\
&\stackrel{\circled{1}}{\le} \Expect{\bY^{(u)}}{\sup_{h \in \cH} \pth{\frac 1u \sum\limits_{i=1}^u h(Y_i)- \cL_n(h)}} \stackrel{\circled{2}}{=}\Expect{\bY^{(u)}}{\sup_{h \in \cH} \pth{\frac 1u \sum\limits_{i=1}^u h(Y_i)- \Expect{{\bY^{(u)}}'}{\frac 1u \sum\limits_{i=1}^u h(Y'_i)}}} \nonumber \\
&\stackrel{\circled{3}}{\le}  \Expect{\bY^{(u)},{\bY^{(u)}}'}{\frac 1u \sup_{h \in \cH} \pth{
 \sum\limits_{i=1}^u h(Y_i) -  \sum\limits_{i=1}^u h(Y'_i) }}
\stackrel{\circled{4}}{=}\Expect{\bY^{(u)},{\bY^{(u)}}',\bsigma}{\frac 1u \sup_{h \in \cH} \pth{
\sum\limits_{i=1}^u \sigma_i \pth{h(Y_i)-h(Y'_i)}
 }} \nonumber \\
&\le \Expect{\bY^{(u)},\bsigma}{\frac 1u \sup_{h \in \cH}
\sum\limits_{i=1}^u \sigma_i h(Y_i)  } +
 \Expect{{\bY^{(u)}}',\bsigma}{\frac 1u \sup_{h \in \cH}
\sum\limits_{i=1}^u \sigma_i h(Y'_i)  } =2\cfrakR^{(\textup{ind})}_u(\cH).
\eal
Here $\circled{1}$
follows from (\ref{eq:TC-RC-seg1}). $\circled{2}$ is due to $\Expect{{\bY^{(u)}}'}{1/u \cdot \sum\limits_{i=1}^u h(Y'_i)}= \cL_n(h)$. $\circled{3}$ is due to the Jensen's inequality, and $\circled{4}$ is due to the definition of the Rademacher variables. All the other bounds for $\cfrakR^-_u(\cH)$,
$\cfrakR^+_m(\cH)$, and  $\cfrakR^-_m(\cH)$ in (\ref{eq:TRC-RC}) can be proved in a similar manner.

\end{proof}

\subsection{Lemmas for the Proof of Theorem~\ref{theorem:concentration-g-m-greater-u}}
\label{sec:lemmas-concentration-g-m-greater-u}

We present the proof of Lemma~\ref{lemma:uniform-draw-diff} and all the other lemmas used in the proof of Theorem~\ref{theorem:concentration-g-m-greater-u}.

\subsubsection{Proof of Lemma~\ref{lemma:uniform-draw-diff}}
We need the following claim in the proof of
Lemma~\ref{lemma:uniform-draw-diff}.

\begin{claim}
\label{claim:singleton-pi-qi}
The set $\set{i' \in [i+1,u] \colon \bZ_{\bd^{(i)}}(i') =  i_0}$ and
the set
$\{i' \in [i+1,u]\colon \bZ_{\bd}(i') =  i_0\}$ are either
empty or they can have at most $1$ element.
\end{claim}
\begin{proof}
We prove this claim for the set
$\set{i' \in [i+1,u] \colon \bZ_{\bd^{(i)}}(i') = i_0}$, and the proof
for the set $\{i' \in [i+1,u]\colon \bZ_{\bd}(i') = i_0 \}$ follows a similar
argument.

If the set $\set{i' \in [i+1,u] \colon \bZ_{\bd^{(i)}}(i') = i_0}$
is not empty, suppose that there are two different elements
$s,s' \in \set{i' \in [i+1,u] \colon \bZ_{\bd^{(i)}}(i') = i_0}$ with
$s < s'$. Because $\bZ_{\bd^{(i)}}(s) = i_0$, it indicates
that the element $i_0$ is never in the vector $\bI^{(i)}(s+1:n)$ after
the $s$-th iteration of Algorithm~\ref{alg:randperm}. It follows
that $\bZ_{\bd^{(i)}}(s') \neq  i_0$ for all $s' > s$, which contradicts
the definition of $s'$. This contradiction proves that if the set
$\set{i' \in [i+1,u] \colon \bZ_{\bd^{(i)}}(i') = i_0}$
is not empty, it can have at most $1$ element.
\end{proof}

\begin{proof}
[\textbf{Proof of Lemma~\ref{lemma:uniform-draw-diff}}]
We present this proof by running the two instances
of Algorithm~\ref{alg:randperm} introduced above.
Throughout this paper, when we refer to the $i$-th iteration,
it means that we run both instances of Algorithm~\ref{alg:randperm}
for $i$ iterations so that $\bZ_{\bd}(1:i)$ and
$\bZ_{\bd^{(i)}}(1:i)$ have been generated.

\noindent \textbf{Case 1.} When $d_i \neq d'_i$ and $q(i) \le u, p(i) > u$, we
aim to prove the claim that $\set{\bZ_{\bd}} = \set{\bZ_{\bd}(i)} \cup \bA$ and
$\set{\bZ_{\bd^{(i)}}}= \set{i_0} \cup \bA$ where $\bA
\subseteq [n]$ is a subset of size $u-1$. To prove this claim,
we first note that $\bZ_{\bd}(1:i-1) =
\bZ_{\bd^{(i)}}(1:i-1)$ according to
Algorithm~\ref{alg:randperm}.

We define $\cC_{r}(\bI,\bI^{(i)}) \subseteq [n]$ for $r \in [n-1]$
to be the set of indices greater than $r$ where $\bI$ and
$\bI^{(i)}$  have different elements, that is,
\bals
\cC_{r}(\bI,\bI^{(i)})
\defeq \set{j \in [r+1:n] \colon \bI(j) \neq \bI^{(i)}(j)}.
\eals
It is noted that $\cC_{r}(\bI,\bI^{(i)})$ can change across
the  iterations of Algorithm~\ref{alg:randperm}.

We use $\ind(i_0)$ to denote the index of element $i_0$ in
$\bI$, and use $\ind^{(i)}(i_0)$ to denote the index of
element $i_0$ in $\bI^{(i)}$.  That is,
$\bI(\ind(i_0)) = \bI^{(i)}(\ind^{(i)}(i_0)) = i_0$.
It is also noted that $\ind(i_0)$ and $\ind^{(i)}(i_0)$
can vary across the iterations of
Algorithm~\ref{alg:randperm}. It can be verified by Algorithm~\ref{alg:randperm} that after the $i$-th
iteration and before the $q(i)$-th iteration,
$\cC_{i}(\bI,\bI^{(i)}) \subseteq
\set{\ind(i_0),\ind^{(i)}(i_0)}$, which is marked by the two red boxes in the
top part of Figure~\ref{fig:uniform-draw-diff-lemma-proof-qi}.
It is also noted that
$\cC_{i}(\bI,\bI^{(i)}) =
\set{\ind^{(i)}(i_0)}$ when
$\bd(i) = i$.



It can be verified that $\bZ_{\bd}(t) = \bI(\bd(t))$ and $\bZ_{\bd^{(i)}}(t) = \bI^{(i)}(\bd^{(i)}(t))$ for every $t \in [u]$, where each array is evaluated after its corresponding draw has been sampled and before the swap in iteration $t$. For $t\ne i$, these coordinates agree. We now show that
$\ind^{(i)}(i_0) \neq \bd(t)$
and $\ind(i_0) \neq \bd(t)$ for every $t \in [i+1\relcolon q(i)-1]$ right after the $(t-1)$-th iteration of Algorithm~\ref{alg:randperm}.
To see this, if $\ind^{(i)}(i_0) = \bd(t)$ for some $t \in [i+1\relcolon q(i)-1]$ right after the $(t-1)$-th iteration of Algorithm~\ref{alg:randperm}, then $\bZ_{\bd^{(i)}}(t) = i_0$, contradicting
Claim~\ref{claim:singleton-pi-qi}.
Moreover, if
$\ind(i_0) =  \bd(t)$ for some $t \in [i+1\relcolon q(i)-1]$ right after the $(t-1)$-th iteration of Algorithm~\ref{alg:randperm}, then
$\bZ_{\bd}(t) = i_0$ for some $t \in [i+1:q(i)-1]$, and then
it follows from Claim~\ref{claim:singleton-pi-qi} that
$p(i) = t\le u$, contradicting the given condition that
 $p(i) > u$.

As a result, we have $\bd(t) \notin \cC_{t-1}(\bI,\bI^{(i)})$  for every $t \in [i+1\relcolon q(i)-1]$ right after the $(t-1)$-th iteration of Algorithm~\ref{alg:randperm}. It follows that $\bI(\bd(t)) = \bI^{(i)}(\bd(t)) = \bZ_{\bd}(t) = \bZ_{\bd^{(i)}}(t)$ for every $t \in [i+1\relcolon q(i)-1]$, that is, $\bZ_{\bd}(i+1:q(i)-1) =
\bZ_{\bd^{(i)}}(i+1:q(i)-1)$.

Right after the
$q(i)$-th iteration of Algorithm~\ref{alg:randperm},
we have
$\bZ_{\bd^{(i)}}(q(i)) = i_0$ and
$\bZ_{\bd}(q(i)) = \bZ_{\bd^{(i)}}(i)$, which is
illustrated in the
bottom part of
Figure~\ref{fig:uniform-draw-diff-lemma-proof-qi}.
It can also be verified by Algorithm~\ref{alg:randperm} that
$\cC_{q(i)}(\bI,\bI^{(i)}) \subseteq
\set{\ind( i_0)}$ holds after the $q(i)$-th iteration,
which is marked by the
red box in the
bottom part of
Figure~\ref{fig:uniform-draw-diff-lemma-proof-qi}.
It is also noted that
$\cC_{q(i)}(\bI,\bI^{(i)}) = \emptyset$ when
$\bd(i) = i$.
Following the same argument above
using Claim~\ref{claim:singleton-pi-qi}, we can verify
that $\ind( i_0) \neq \bd(t)$ for every $t \in [q(i)+1\relcolon u]$ right after the $(t-1)$-th iteration of Algorithm~\ref{alg:randperm}. Otherwise,
we have $\bZ_{\bd}(t) = i_0$ for
some $t\in [q(i)+1:u]$ so that $p(i)=t\le u$,
contradicting the given condition that
$p(i) > u$. It
then follows from the same argument above
that $\bZ_{\bd}(q(i)+1:u) =
\bZ_{\bd^{(i)}}(q(i)+1:u)$.

In summary, after we run
Algorithm~\ref{alg:randperm} twice to generate
$\bZ_{\bd}$ and $\bZ_{\bd^{(i)}}$,
$\bZ_{\bd}$ and $\bZ_{\bd^{(i)}}$ can only differ
at the $i$-th element and the $q(i)$-th element.
Since $\bZ_{\bd}(q(i)) = \bZ_{\bd^{(i)}}(i)$, which
is also shown in the bottom part of
Figure~\ref{fig:uniform-draw-diff-lemma-proof-qi},
we have
$\set{\bZ_{\bd}} = \set{\bZ_{\bd}(i)} \cup \bA$ and
$\set{\bZ_{\bd^{(i)}}}= \set{i_0} \cup \bA$,
where
\bals
\bA={}&\set{\bZ_{\bd}(1:i-1)}
\cup\set{\bZ_{\bd^{(i)}}(i)}
\cup\set{\bZ_{\bd}(i+1:q(i)-1)}\\
&\cup\set{\bZ_{\bd}(q(i)+1:u)}.
\eals
Empty index ranges contribute the empty set. Thus $\bA$ contains exactly the original outputs other than $\bZ_{\bd}(i)$, and $|\bA|=u-1$. It follows that
\bals
\cL_h(\set{\bZ_{\bd}}) -
\cL_h(\set{\bZ_{\bd^{(i)}}}) =
\frac 1u \pth{h(\bZ_{\bd}(i)) - h(i_0)}.
\eals
Similarly,
$ \cL_h(\overline{\bZ_{\bd^{(i)}}}) -\cL_h(\overline{\set{\bZ_{\bd}}})
=1/m \cdot \pth{h(\bZ_{\bd}(i)) - h(i_0)}$
since
$\cL_h(\overline{\bZ_{\bd^{(i)}}}) = 1/m
\cdot (n \cL_n(h) - u \cL_h({\bZ_{\bd^{(i)}}}))$
and $\cL_h(\overline{\set{\bZ_{\bd}}}) = 1/m
\cdot (n \cL_n(h) - u \cL_h({\bZ_{\bd}}))$. Since
$i_0 =\bZ_{\bd^{(i)}}(q(i))$, Case 1 is proved.

\begin{figure}[!tb]
\centering

\begin{subfigure}[t]{0.48\textwidth}
    \centering
    \includegraphics[width=\linewidth]{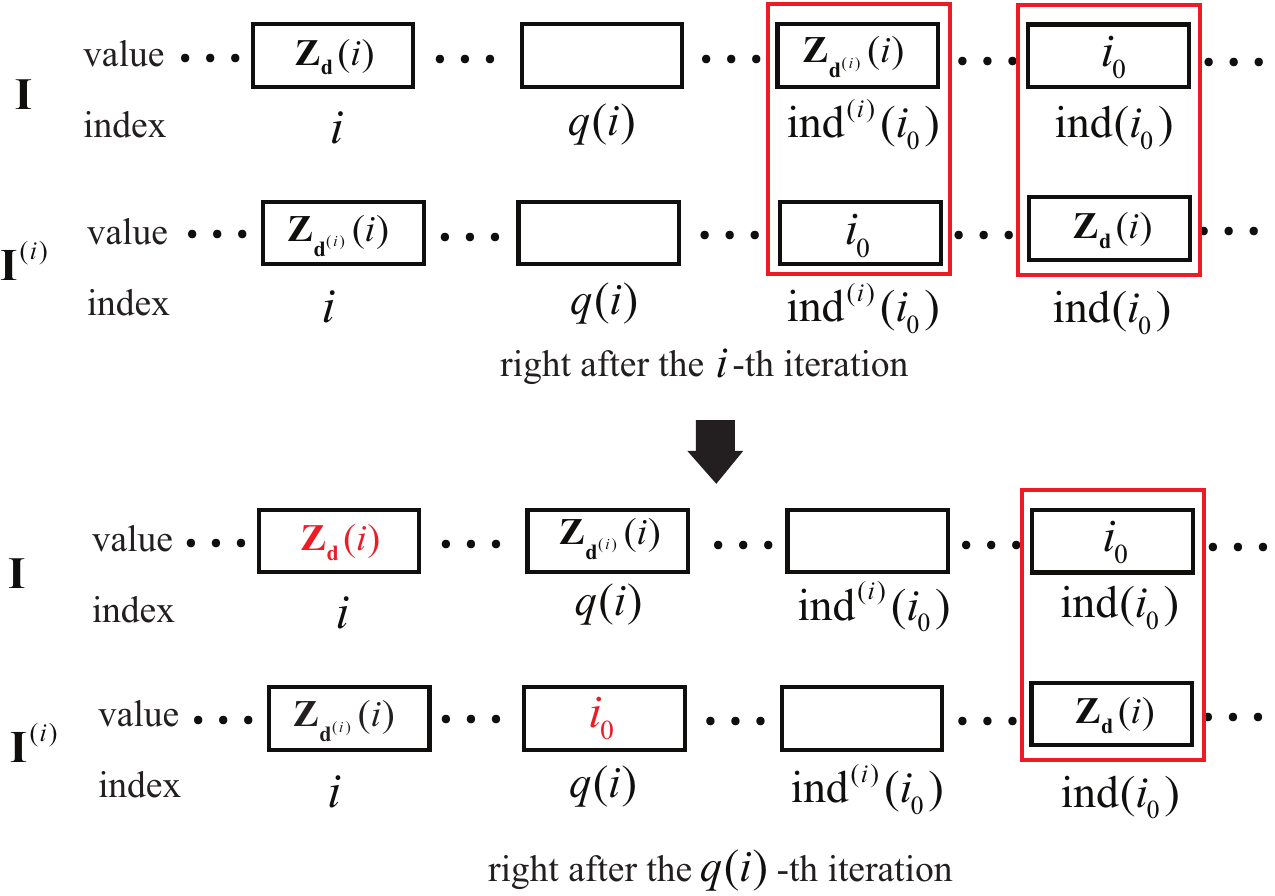}
    \caption{Case 1: $d_i \neq d'_i$, $q(i) \le u$, $p(i) > u$.}
    \label{fig:uniform-draw-diff-lemma-proof-qi}
\end{subfigure}
\hfill
\begin{subfigure}[t]{0.48\textwidth}
    \centering
    \includegraphics[width=\linewidth]{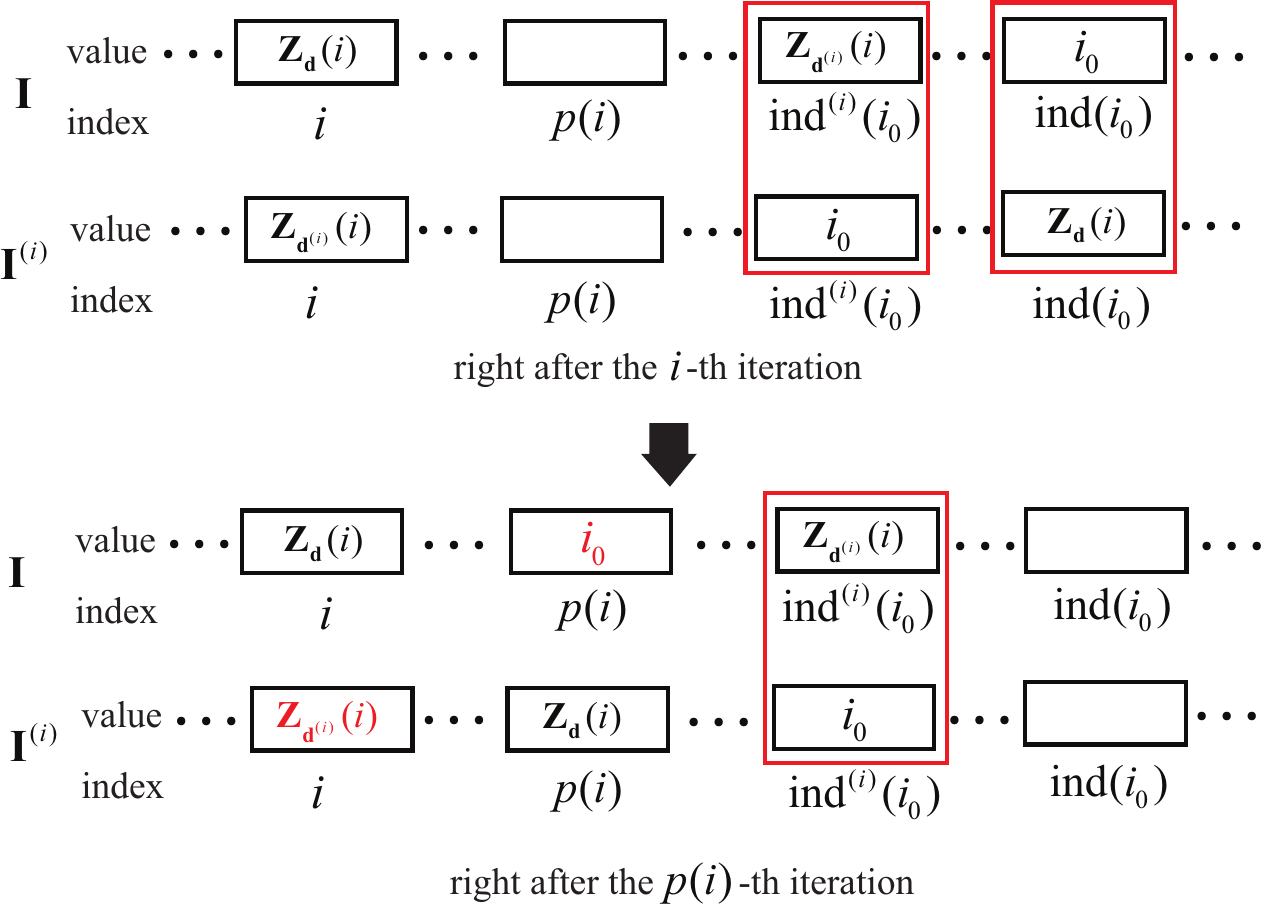}
    \caption{Case 2: $d_i \neq d'_i$, $p(i) \le u$, $q(i) > u$.}
    \label{fig:uniform-draw-diff-lemma-proof-pi}
\end{subfigure}

\caption{Illustration of Case 1 and Case 2 in the proof of Lemma~\ref{lemma:uniform-draw-diff}.}
\label{fig:uniform-draw-diff-lemma-proof}
\end{figure}

\noindent \textbf{Case 2.} By the symmetry between
$\bd$ and $\bd^{(i)}$ as they only differ at the
$i$-th element ($\bd(j) = \bd^{(i)}(j)$ for all $j \in [u] \setminus \set{i}$), we can apply the argument in Case 1
and swap $\pth{\bI,\bd,p(i)}$ and
$\pth{\bI^{(i)},\bd^{(i)},q(i)}$
to prove Case 2.
Figure~\ref{fig:uniform-draw-diff-lemma-proof-pi}
illustrates $\bI$ and $\bI^{(i)}$ after the $i$-th iteration
and the $p(i)$-th iteration of Algorithm~\ref{alg:randperm}.

\begin{figure}[!tb]
\begin{center}
\includegraphics[width=0.7\textwidth]
{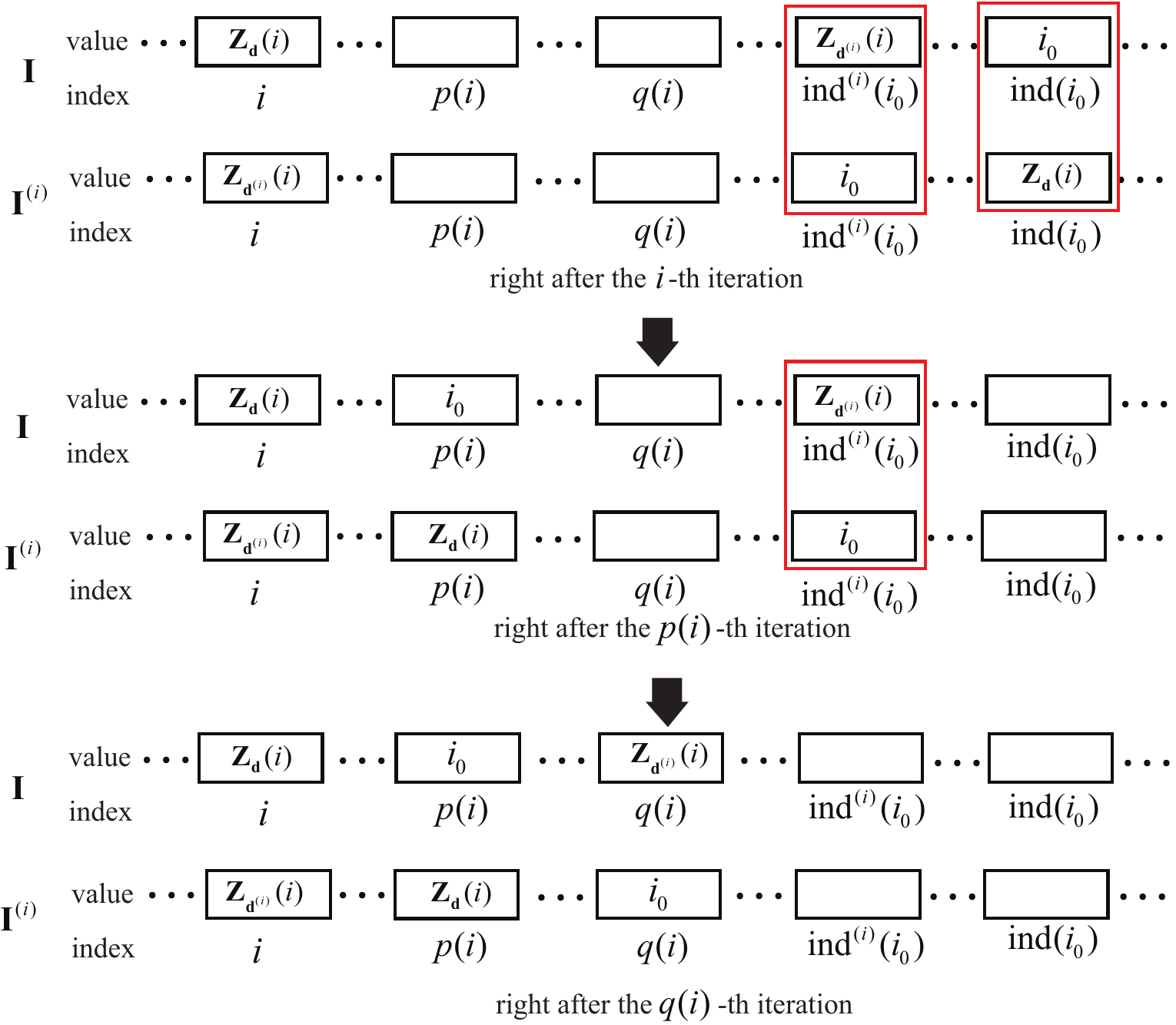}
\end{center}
\caption{Illustration of Case 4: $d_i = d'_i$ or $p(i),q(i) \le u$}
\label{fig:uniform-draw-diff-lemma-proof-pi-qi}
\end{figure}

\noindent \textbf{Case 4.} If $d_i = d'_i$,
we have $\bd = \bd^{(i)}$ so that
$\bZ_{\bd} = \bZ_{\bd^{(i)}}$ and
$E(h,\bd,\bd^{(i)}) = 0$. We now consider the case
when $p(i),q(i) \le u$ and $d_i \neq d'_i$.
Without loss of generality we assume that
$p(i) < q(i)$. The case that $p(i) \ge q(i)$ can be proved by
a similar argument. Figure~\ref{fig:uniform-draw-diff-lemma-proof-pi-qi}
illustrates $\bI$ and $\bI^{(i)}$ after the $i$-th iteration,
the $p(i)$-th iteration, and the $q(i)$-th iteration
of Algorithm~\ref{alg:randperm}.

By repeating the argument in Case 1, after the $i$-th iteration and before the $p(i)$-th iteration
 of Algorithm~\ref{alg:randperm}, we
have $\cC_{i}(\bI,\bI^{(i)}) \subseteq
\set{\ind(i_0),\ind^{(i)}(i_0)}$, and
$\ind(i_0) \neq \bd(t)$,
$\ind^{(i)}(i_0)\neq \bd(t)$ for every $t \in [i+1\relcolon p(i)-1]$ right after the $(t-1)$-th iteration of Algorithm~\ref{alg:randperm}. As a result,
$\bZ_{\bd}(i+1:p(i)-1) = \bZ_{\bd^{(i)}}(i+1:p(i)-1)$.

After the $p(i)$-th iteration and before the $q(i)$-th iteration,
it can be verified by Algorithm~\ref{alg:randperm} that $\cC_{p(i)}(\bI,\bI^{(i)}) = \set{\ind^{(i)}(i_0)}$. It
follows from Claim~\ref{claim:singleton-pi-qi} that
$\ind^{(i)}(i_0) \neq \bd(t)$ for every $t \in [p(i)+1\relcolon q(i)-1]$ right after the $(t-1)$-th iteration of Algorithm~\ref{alg:randperm}. As a result,
$\bZ_{\bd}(p(i)+1:q(i)-1) =
\bZ_{\bd^{(i)}}(p(i)+1:q(i)-1)$.

After the $q(i)$-th iteration,
it can be verified by Algorithm~\ref{alg:randperm} that $\cC_{q(i)}(\bI,\bI^{(i)}) = \emptyset$. As a result,
we have $\bZ_{\bd}(q(i)+1:u) =
\bZ_{\bd^{(i)}}(q(i)+1:u)$.

We have
\bals
\set{\bZ_{\bd}(\set{i,p(i),q(i)})}
=
\set{\bZ_{\bd^{(i)}}(\set{i,p(i),q(i)})}
=\set{i_0,\bZ_{\bd}(i),\bZ_{\bd^{(i)}}(i)},
\eals
which is also illustrated in the bottom part of
Figure~\ref{fig:uniform-draw-diff-lemma-proof-pi-qi}.
Combining all the above arguments, we have
$\set{\bZ_{\bd}} = \set{\bZ_{\bd^{(i)}}}$,  so that $E(h,\bd,\bd^{(i)}) = 0$.

\noindent \textbf{Case 3.} Suppose
 $d_i \neq d'_i$, and $p(i) >u, q(i) > u$.
By repeating the argument in Case 1, we have
$\cC_{i}(\bI,\bI^{(i)}) \subseteq
\set{\ind(i_0),\ind^{(i)}(i_0)}$ after the $i$-th iteration of Algorithm~\ref{alg:randperm}. Since
$p(i) >u, q(i) > u$, we have
$\ind(i_0) \neq \bd(t)$ and
$\ind^{(i)}(i_0) \neq \bd(t)$ for every $t \in [i+1\relcolon u]$ right after the $(t-1)$-th iteration of Algorithm~\ref{alg:randperm}.
As a result,
$\bZ_{\bd}(i+1:u) = \bZ_{\bd^{(i)}}(i+1:u)$.
It follows that
\bals
\set{\bZ_{\bd}}
&=\set{\bZ_{\bd}(1:i-1)}\cup\set{\bZ_{\bd}(i)}
\cup\set{\bZ_{\bd}(i+1:u)},\\
\set{\bZ_{\bd^{(i)}}}
&=\set{\bZ_{\bd}(1:i-1)}\cup\set{\bZ_{\bd^{(i)}}(i)}
\cup\set{\bZ_{\bd}(i+1:u)}.
\eals
We then have
$\cL_h(\set{\bZ_{\bd}}) -
\cL_h(\set{\bZ_{\bd^{(i)}}}) =
1/u \cdot \pth{h(\bZ_{\bd}(i)) - h(\bZ_{\bd^{(i)}}(i))}$
and $\cL_h(\overline{\bZ_{\bd^{(i)}}}) -\cL_h(\overline{\set{\bZ_{\bd}}})
=1/m \cdot \pth{h(\bZ_{\bd}(i)) - h(\bZ_{\bd^{(i)}}(i))}$,
which proves Case 3.

\end{proof}

\subsubsection{Proofs of Lemma~\ref{lemma:zd-pi-qi-repeat-number}-Lemma~\ref{lemma:Zdi-random}}

We then prove Lemma~\ref{lemma:zd-pi-qi-repeat-number}-Lemma~\ref{lemma:Zdi-random}.

\begin{proof}[Proof of Lemma~\ref{lemma:zd-pi-qi-repeat-number}]
Consider the value occupying position $i$ immediately before iteration $i$ of Algorithm~\ref{alg:randperm}. This is the value $i_0$ in Lemma~\ref{lemma:uniform-draw-diff}. Fix a value that occupies such positions at iterations
$j_1<\cdots<j_k$. After iteration $j_\ell$, this value is at position $d_{j_\ell}$ unless it has been output. If a later iteration strictly before $d_{j_\ell}$ selected that position, the value would be output and could not occupy a current position again. Thus, for consecutive occurrences of the same value,
$d_{j_\ell}=j_{\ell+1}$. In particular, its positions form a chain, and on $\Omega(Q)^c$ each value occurs as $i_0$ at most $Q-1$ times.

If $p(i)\le u$, then $\bZ_{\bd}(p(i))=i_0$. Therefore equal finite values of $p(i)$ imply equal values of $i_0$. If $q(i)\le u$, then
$\bZ_{\bd^{(i)}}(q(i))=i_0$, so equality of these outputs also implies equality of the corresponding values of $i_0$. The preceding chain property proves both inclusions in
(\ref{eq:zd-pi-qi-repeat-number-prob-Omega12}), for every realization of the independent replacement draws.

A chain of length at least $Q$ contains a chain of length $Q$. For each fixed
$j_1<\cdots<j_Q$ in $[u]$, the coordinates of $\bd$ are independent, and
\bal
\Prob{d_{j_\ell}=j_{\ell+1}\textup{ for all }\ell\in[Q-1]}
=\prod_{\ell=1}^{Q-1}\frac1{n-j_\ell+1}.
\eal
A union bound over all such chains gives
\bal\label{eq:zd-pi-qi-repeat-number-chain-prob}
\Prob{\Omega(Q)}
\le\sum_{1\le j_1<\cdots<j_Q\le u}
\prod_{\ell=1}^{Q-1}\frac1{n-j_\ell+1}.
\eal
The sequence $\{(n-i+1)^{-1}\}_{i=1}^u$ is increasing. Maclaurin's inequality for its first $j-1$ entries and the fact that a prefix average of an increasing sequence does not exceed its full average show that the right side is at most
\bal
&\sum_{j=Q}^u\binom{j-1}{Q-1}
\pth{\frac1{j-1}\sum_{i=1}^{j-1}\frac1{n-i+1}}^{Q-1}
\nonumber\\
&\le
\binom{u}{Q}\pth{\frac1u\sum_{i=1}^u\frac1{n-i+1}}^{Q-1}.
\eal
Here the last step also uses
$\sum_{j=Q}^u\binom{j-1}{Q-1}=\binom{u}{Q}$.
Since $m\ge u$,
\bal
\sum_{i=1}^u\frac1{n-i+1}
\le\sum_{k=u+1}^{2u}\frac1k
<\int_u^{2u}\frac{dt}{t}=\ln2.
\eal
This proves the first line of (\ref{eq:omega-Q-prob-bound}).

For its second line, start a chain at a fixed $j_1\le u-Q+1$. At each subsequent vertex $j\in[u]$, the conditional probability of another strictly increasing step within $[u]$ is
$(u-j)/(n-j+1)\le1/2$, because the unused coordinate $d_j$ is independent of the preceding coordinates and $m\ge u$. The probability of $Q-1$ such steps is therefore at most $2^{-(Q-1)}$. A union bound over the at most $u-Q+1$ possible starting points proves the second line. For $Q=2$, the same argument with
$(u-j)/(n-j+1)\le u/m$ gives $\Prob{\Omega(2)}\le u^2/m$.

It remains to prove (\ref{eq:omega-Q-weighted-prob-bound}). For $u=2,3$, the stated choice has $Q=u$, so the left side is $u\le\log_2(4u)$. For $u=4$, $Q=3$ and
$(u-Q)\Prob{\Omega(Q)}\le4/25<1=\log_2(4u)-Q$.
For $u=5,6,7$, $Q=4$ and, respectively,
\[
(u-Q)\Prob{\Omega(Q)}
\le\frac5{216},\quad \frac{30}{343},\quad \frac{105}{512}.
\]
Each is less than $1/4$, whereas
$\log_2(4u)-Q\ge\log_2(5/4)>1/4$.
For $u=8$, $Q=4$ and
$(u-Q)\Prob{\Omega(Q)}\le280/729<1=\log_2(4u)-Q$.

Now suppose $u\ge9$, so $Q\ge5$. The choice of $Q$ gives
$2^{Q-2}<u\le2^{Q-1}$. Using
$\binom{u}{Q}\le u^Q/Q!$ in the first line of
(\ref{eq:omega-Q-prob-bound}), we obtain
\bal
(u-Q)\Prob{\Omega(Q)}
\le\frac{u^2(\ln2)^{Q-1}}{Q!}.
\eal
Also,
\bal
\log_2(4u)-Q
=\frac{\ln(u/2^{Q-2})}{\ln2}
\ge\frac{u-2^{Q-2}}{u\ln2}.
\eal
It therefore suffices to show that
$u^3(\ln2)^Q/Q!\le u-2^{Q-2}$.
For fixed $Q$, the difference between the left and right sides is a convex function of real $u$, so it is enough to check
$u=2^{Q-2}+1$ and $u=2^{Q-1}$.
At $Q=5$, these two inequalities reduce respectively to
$9^3(\ln2)^5/5!<1$ and
$2^9(\ln2)^5/5!<1$.
They continue to hold for every $Q>5$: on increasing $Q$ by one, the left side of the first endpoint inequality grows by a factor smaller than $8\ln2/(Q+1)<1$, and the left side of the second grows by $4\ln2/(Q+1)<1$.
Thus $(u-Q)\Prob{\Omega(Q)}\le\log_2(4u)-Q$ for all $u\ge9$, completing the proof.
\end{proof}

\begin{proof}[Proof of Lemma~\ref{lemma:zd-pi-upper-bound}]

First, since $\bd \notin \Omega(Q)$, it follows from Lemma~\ref{lemma:zd-pi-qi-repeat-number} that
$\bd \notin \Omega^{(1)}(Q) \cup \Omega^{(2)}(Q)$.

We denote the unique values in
$\set{p(i) \colon i \in [u], p(i) \le u}$ by
$\set{i_1,i_2,\ldots,i_{u'}}$ with $u' \le u$, and let $\cP(i_k)
\defeq \set{i \in [u] \colon p(i) = i_k}$ for $k \in [u']$.  It follows from Lemma~\ref{lemma:zd-pi-qi-repeat-number} that
$\abth{\cP(i_k)} \le Q-1$ since $\bd \notin \Omega^{(1)}(Q)$. As a result,
we have
\bals
\sum\limits_{i \in \cA_{h,2} } \pth{h(\bZ_{\bd}(p(i)))}^2
&\le \sum\limits_{k=1}^{u'} \abth{\cP(i_k)} \pth{h(\bZ_{\bd}(i_k))}^2
\le (Q-1) \sum\limits_{k=1}^{u'} \pth{h(\bZ_{\bd}(i_k))}^2
\le (Q-1) \sum\limits_{i=1}^{u} \pth{h(\bZ_{\bd}(i))}^2.
\eals
\end{proof}

\begin{proof}[Proof of Lemma~\ref{lemma:Zdi-random}]
For all $i \in [u] $, we have
\bals
\Expect{}{\pth{h(\bZ_{\bd^{(i)}}(i))}^2 \longmid \bd} &=
\frac{1}{n-i+1} \sum\limits_{j \notin \set{\bZ_{\bd}(1:i-1)}} h^2(j)
\le \frac{n}{n-i+1} T_n(h).
\eals
We then have
\bals
\Expect{}{\sum\limits_{i =1}^u \pth{h(\bZ_{\bd^{(i)}}(i))}^2 \longmid \bd} &\le
\sum\limits_{i =1}^u \frac{n}{n-i+1} T_n(h)
\le \frac{nu}{m} T_n(h),
\eals
which proves (\ref{eq:Zdi-random-u-i}).

To prove (\ref{eq:Zdi-random-u-q}), fix $\bd$ and, for each $i\in[u]$, let $i_0$ be the value occupying position $i$ immediately before iteration $i$. This value is fixed conditional on $\bd$, although it may depend on $i$. After iteration $i$ in the replacement run, $i_0$ occupies position $d'_i$, which is uniform on $[i:n]$. With the later draws fixed, exactly $u-i$ of these $n-i+1$ positions lead to selection during iterations $i+1,\ldots,u$. Therefore,
\bals
\Prob{q(i)\le u\mid\bd}
=\frac{u-i}{n-i+1}\le\frac{u}{m}.
\eals
On the event $\{i\in\cA_{h,1}\}$, we have $q(i)\le u$ and $\bZ_{\bd^{(i)}}(q(i))=i_0$. Consequently,
\bal\label{eq:Zdi-random-seg1}
&\Expect{}{\sum_{i\in\cA_{h,1}}
h^2(\bZ_{\bd^{(i)}}(q(i)))\mid\bd}\nonumber\\
&=\sum_{i=1}^u h^2(i_0)\Prob{i\in\cA_{h,1}\mid\bd}
\le\frac{u}{m}\sum_{i=1}^u h^2(i_0)\nonumber\\
&\le\frac{(Q-1)u}{m}\sum_{j=1}^n h^2(j)
=\frac{(Q-1)nu}{m}T_n(h).
\eal
Here each summand uses the value $i_0$ associated with its index $i$. The last inequality follows from the chain argument in the proof of Lemma~\ref{lemma:zd-pi-qi-repeat-number}, since $\bd\notin\Omega(Q)$ implies that each value occurs as $i_0$ at most $Q-1$ times. This proves the stronger bound above and hence (\ref{eq:Zdi-random-u-q}).
\end{proof}

\subsection{Proof of Theorem~\ref{theorem:concentration-tildeg-u-greater-m}}

To apply the entropy inequality (\ref{eq:concentration-g-m-greater-u-product-entropy}) in Section~\ref{sec::math-tools}, we let $\tbd' = [\td'_1,\ldots,\td'_m]$ be an independent copy of $\tbd$, and $\tbd^{(i)} = [\td_1,\ldots,\td_{i-1},\td'_i, \td_{i+1},\ldots,\td_m]$ for every $i \in [m]$.


For every $h \in \cH$, we define the following change of the test-train loss,
\bal\label{eq:tE-def}
\tE(h,\tbd,\tbd^{(i)}) \defeq
\cL_h(\set{\bZ_{\tbd}})-\cL_h(\overline{\set{\bZ_{\tbd}}}) - \cL_h(\bZ_{\tbd^{(i)}}) +  \cL_h(\overline{\bZ_{\tbd^{(i)}}}).
\eal
Similar to the four cases in Lemma~\ref{lemma:uniform-draw-diff}, it follows by repeating the argument in the proof of Lemma~\ref{lemma:uniform-draw-diff} that we have the following four cases for $\tE$ as stated in the following lemma.

\begin{lemma}
\label{lemma:uniform-draw-diff-tE}
For any $h \in \cH$, there are four cases for the value of $\tE(h,\tbd,\tbd^{(i)})$ for $i \in [m]$.
\begin{itemize}[leftmargin=40pt]
\item[Case 1:] $\tE(h,\tbd,\tbd^{(i)}) = \pth{\frac 1u + \frac 1m}  \pth{h(\bZ_{\tbd}(i)) - h(\bZ_{\tbd^{(i)}}(\tilde q(i)))}$,

if $\td_i \neq \td'_i$, $\tilde q(i) \le m, \tilde p(i) > m$,
\item[Case 2:] $\tE(h,\tbd,\tbd^{(i)}) = \pth{\frac 1u + \frac 1m} \pth{h(\bZ_{\tbd}(\tilde p(i))) - h(\bZ_{\tbd^{(i)}}(i))}$,

if $\td_i \neq \td'_i$, $\tilde p(i) \le m, \tilde q(i) > m$,
\item[Case 3:]  $\tE(h,\tbd,\tbd^{(i)}) = \pth{\frac 1u + \frac 1m} \pth{h(\bZ_{\tbd}(i)) - h(\bZ_{\tbd^{(i)}}(i))}$,

if $\td_i \neq \td'_i$, $\tilde p(i) >m, \tilde q(i) > m$,

\item[Case 4:]  $\tE(h,\tbd,\tbd^{(i)}) = 0$,

if $\td_i = \td'_i$ or $\tilde p(i),\tilde q(i) \le m$,
\end{itemize}
where
\bal
\tilde q(i) &\defeq \min\set{i' \in [i+1,m] \colon \bZ_{\tbd^{(i)}}(i') = i_0}, \label{eq:tqi}
\\
\tilde p(i) &\defeq \min\set{i' \in [i+1,m]\colon \bZ_{\tbd}(i') = i_0}, \label{eq:tpi}
\eal
where $i_0$ is the value occupying position $i$ immediately before iteration $i$ of the RANDPERM run on $\tbd$, as in Lemma~\ref{lemma:uniform-draw-diff}. The $\min$ over an empty set returns $+\infty$.
\end{lemma}

\noindent \textbf{Useful Definitions.} We need the following definitions
before the proof of Theorem~\ref{theorem:concentration-tildeg-u-greater-m}.
These definitions are ``symmetric'' to the definitions in
Section~\ref{sec:proof-concentration-test-train-process-m-greater-u} for the proof
of Theorem~\ref{theorem:concentration-g-m-greater-u}, and they are expressed in terms of $\tbd$ instead of $\bd$.

For any $h \in \cH$, define
\bal\label{eq:set-tilde-A}
\tilde \cA_{h,p} = \set{i \in [m] \colon \tbd \textup{ and } \tbd^{(i)}
\textup{ satisfy Case $p$ in Lemma~\ref{lemma:uniform-draw-diff-tE}}}, \quad p \in [4].
\eal

For a function class $H'$, we define
\bal
t_m(\tbd,\cH') &\defeq \sup_{h \in \cH'} \pth{\cL_h(\set{\bZ_{\tbd}}) - \cL_n(h)} \label{eq:tdm}.
\eal

\begin{definition}
Define
\bal\label{eq:omega-tQ-def}
\tOmega(Q) \defeq
\set{\tbd \colon \textup{there exists a chain } \set{j_1,\ldots,j_{Q'}}
\textup{ in } [m] \textup{ associated with } \tbd, Q' \in [Q:m]}
\eal
as the subset of $\tbd$ for which the length of the longest chain is not less than $Q$.
\end{definition}

We are then ready to present the proof of Theorem~\ref{theorem:concentration-tildeg-u-greater-m}.

\begin{proof}[\textbf{Proof of Theorem~\ref{theorem:concentration-tildeg-u-greater-m}}]
The processes $\tg(\tbd)$ and $g(\bd)$ have the same distribution: the complement of the uniformly sampled $m$-subset $\set{\bZ_{\tbd}}$ is a uniformly sampled $u$-subset. In particular,
\bal\label{eq:equiv-tildeg-g}
\Expect{\tbd}{\tg(\tbd)}=\Expect{\bd}{g(\bd)}.
\eal
As in the proof of Theorem~\ref{theorem:concentration-g-m-greater-u}, we may pass to the closure of the restrictions of $\cH$ to $[n]$, so that all suprema are attained. Let $\tbd^{(i)}$ be obtained by replacing $\td_i$ with an independent copy $\td'_i$.

\noindent \textbf{Step 1: Conditional upper variances.}
Suppose first that $m\ge2$ and fix $Q\in[2:m]$. For $h^*$ attaining $\tg(\tbd)$, the definition of $\tg$ and Lemma~\ref{lemma:uniform-draw-diff-tE} give, for every replacement, including one crossing the boundary of $\tOmega(Q)$,
\bal\label{eq:concentration-tildeg-u-greater-m-direct-diff}
\pth{\tg(\tbd)-\tg(\tbd^{(i)})}_+
\le\pth{-\tE(h^*,\tbd,\tbd^{(i)})}_+.
\eal
For each nonzero exchange, the square on the right is at most
$2(1/u+1/m)^2$ times the sum of the squared values of $h^*$ at the old and new endpoints. The old endpoint is $\bZ_{\tbd}(i)$ in Cases 1 and 3 and $\bZ_{\tbd}(\tilde p(i))$ in Case 2. By Lemma~\ref{lemma:ztd-pi-upper-bound}, its squared values sum to at most
$Q\sum_{j=1}^m (h^*)^2(\bZ_{\tbd}(j))$ on $\tOmega(Q)^c$. The same bound with $Q$ replaced by $m$ holds everywhere.

In Case 1 the new endpoint is the value $i_0$ occupying position $i$ immediately before iteration $i$. Conditional on $\tbd$, it is selected in the replacement run with probability
\bal\label{eq:concentration-tildeg-u-greater-m-q-prob}
\Prob{\tilde q(i)\le m\mid\tbd}
=\frac{m-i}{n-i+1}\le\frac{m}{u}.
\eal
Indeed, with the later draws fixed, exactly $m-i$ of the $n-i+1$ possible positions of $i_0$ cause its selection after iteration $i$. Its multiplicity over $i$ is at most $Q-1$ on $\tOmega(Q)^c$ and at most $m$ everywhere, by the chain argument in the proof of Lemma~\ref{lemma:ztd-pi-qi-repeat-number}. The new endpoint in Cases 2 and 3 is $\bZ_{\tbd^{(i)}}(i)$, which is conditionally uniform among the $n-i+1$ indices not selected before iteration $i$. Thus, with $h\in\cH$ fixed conditional on $\tbd$,
\bal\label{eq:concentration-tildeg-u-greater-m-new-indices}
\Expect{}{\sum_{i\in\tilde\cA_{h,1}}h^2(\bZ_{\tbd^{(i)}}(\tilde q(i)))\mid\tbd}
&\le\begin{cases}
\frac{(Q-1)mn}{u}T_n(h),&\tbd\notin\tOmega(Q),\\
\frac{m^2n}{u}T_n(h),&\tbd\in\tOmega(Q),
\end{cases}\\
\Expect{}{\sum_{i=1}^m h^2(\bZ_{\tbd^{(i)}}(i))\mid\tbd}
&\le\frac{mn}{u}T_n(h).
\eal
Since $(1/u+1/m)^2\le4/m^2$, $n/u\le2$, and $T_n(h^*)\le r$, these estimates imply
\bal\label{eq:concentration-tildeg-u-greater-m-direct-variance}
V_+(\tg\mid\tbd)
\le\begin{cases}
\frac{8Q}{m}\pth{t_m(\tbd,\cH^2)+3r},&\tbd\notin\tOmega(Q),\\[2mm]
8\pth{t_m(\tbd,\cH^2)+4r},&\tbd\in\tOmega(Q).
\end{cases}
\eal
The second line follows more precisely from
$8\cL_{(h^*)^2}(\set{\bZ_{\tbd}})+8(2+2/m)T_n(h^*)$.

Applying the same one-exchange argument to a maximizer of $t_m(\tbd,\cH^2)$, using nonnegativity of $h^2$ and $h^4\le H_0^2h^2$, yields
\bal\label{eq:concentration-tildeg-u-greater-m-square-variance}
V_+(t_m(\tbd,\cH^2)\mid\tbd)
\le\begin{cases}
\frac{QH_0^2}{m}\pth{t_m(\tbd,\cH^2)+r},&\tbd\notin\tOmega(Q),\\[2mm]
H_0^2\pth{t_m(\tbd,\cH^2)+r},&\tbd\in\tOmega(Q).
\end{cases}
\eal
Both right-hand sides are nonnegative since $t_m(\tbd,\cH^2)\ge-r$. If $m=1$, then $\tOmega(1)$ is the whole sample space. Directly resampling the sole coordinate proves the second lines of the two variance bounds, with $Q=1$.

\noindent \textbf{Step 2: Averaging the entropy inequalities.}
Let $\pi$ be an independent uniform permutation of $[n]$. For each fixed $\pi$, apply the preceding bounds to the relabeled class $\set{h\circ\pi:h\in\cH}$. Its population quantities and bounds are unchanged. Conditional on $\tbd$, the set $\pi\set{\bZ_{\tbd}}$ is a uniform $m$-subset and is therefore independent of $\tbd$. The entropy inequality (\ref{eq:concentration-g-m-greater-u-product-entropy}) in Section~\ref{sec::math-tools}, applied to the $m$ independent coordinates of $\tbd$, can consequently be averaged over $\pi$, making each indicator of $\tOmega(Q)$ independent of the exponentially tilted empirical process. Take
$Q=\min\{m,\lceil\log_2(4m)\rceil-1\}$ if $m\ge2$, and $Q=1$ if $m=1$. Lemma~\ref{lemma:ztd-pi-qi-repeat-number} gives the following bound for $m\ge2$. For $m=1$, its left side is $1$ and the bound is immediate:
\bal\label{eq:concentration-tildeg-u-greater-m-weight}
Q\Prob{\tOmega(Q)^c}+m\Prob{\tOmega(Q)}\le\log_2(4m).
\eal
It follows from (\ref{eq:concentration-tildeg-u-greater-m-direct-variance})--(\ref{eq:concentration-tildeg-u-greater-m-weight}) that, for $\lambda>0$,
\bal\label{eq:concentration-tildeg-u-greater-m-entropy-g}
\lambda\frac{\Expect{\tbd}{\tg(\tbd)e^{\lambda\tg(\tbd)}}}{\Expect{\tbd}{e^{\lambda\tg(\tbd)}}}
-\log\Expect{\tbd}{e^{\lambda\tg(\tbd)}}
\le\frac{8\log_2(4m)\lambda^2}{m}
\pth{\frac{\Expect{\tbd}{t_m(\tbd,\cH^2)e^{\lambda\tg(\tbd)}}}{\Expect{\tbd}{e^{\lambda\tg(\tbd)}}}+4r},
\eal
The corresponding inequality for $t_m(\tbd,\cH^2)$ is
\bal\label{eq:concentration-tildeg-u-greater-m-entropy-tm}
\lambda\frac{\Expect{\tbd}{t_m(\tbd,\cH^2)e^{\lambda t_m(\tbd,\cH^2)}}}{\Expect{\tbd}{e^{\lambda t_m(\tbd,\cH^2)}}}
-\log\Expect{\tbd}{e^{\lambda t_m(\tbd,\cH^2)}}
\le\frac{H_0^2\log_2(4m)\lambda^2}{m}
\pth{\frac{\Expect{\tbd}{t_m(\tbd,\cH^2)e^{\lambda t_m(\tbd,\cH^2)}}}{\Expect{\tbd}{e^{\lambda t_m(\tbd,\cH^2)}}}+r}.
\eal

\noindent \textbf{Step 3: Closing the moment bounds.}
Set $\theta=H_0^2\log_2(4m)/m$. Integrating the entropy inequality for $t_m$, exactly as in (\ref{eq:concentration-g-m-greater-u-mgf-tu}), gives
\bal\label{eq:concentration-tildeg-u-greater-m-mgf-tm}
\log\Expect{\tbd}{e^{\lambda t_m(\tbd,\cH^2)}}
\le\frac{\lambda\pth{\cfrakR_m^+(\cH^2)+\theta\lambda r}}
{1-\theta\lambda},\qquad 0<\lambda<1/\theta.
\eal
The expectation $\cfrakR_m^+(\cH^2)$ is nonnegative. Applying the relative-entropy variational formula to the tilt by $e^{\lambda\tg}$, then inserting (\ref{eq:concentration-tildeg-u-greater-m-mgf-tm}) in (\ref{eq:concentration-tildeg-u-greater-m-entropy-g}), gives the same differential inequality as (\ref{eq:concentration-g-m-greater-u-entropy-closed}), with $u,g,t_u$ replaced by $m,\tg,t_m$. The use of $H_0^2\ge8$ is legitimate under the assumptions of the theorem. Its integration gives
\bal\label{eq:concentration-tildeg-u-greater-m-mgf-g}
\log\Expect{\tbd}{e^{\lambda(\tg(\tbd)-\Expect{\tbd}{\tg(\tbd)})}}
\le\frac{8\log_2(4m)\lambda^2
\pth{\cfrakR_m^+(\cH^2)+4r}}
{m(1-\theta\lambda)},\qquad 0<\lambda<1/\theta.
\eal
The Chernoff bound therefore yields, with probability at least $1-e^{-x}$,
\bal\label{eq:concentration-tildeg-u-greater-m-direct-tail}
\tg(\tbd)\le\Expect{\tbd}{\tg(\tbd)}
+2\sqrt{\frac{8\log_2(4m)\pth{\cfrakR_m^+(\cH^2)+4r}x}{m}}
+\frac{H_0^2\log_2(4m)x}{m}.
\eal
Finally, $\sqrt{a+b}\le\sqrt a+\sqrt b$, $8\sqrt2\le4\sqrt{10}$, and
$\log_2(4m)\le\log_2(4m/\delta)$ show that this bound implies (\ref{eq:concentration-g-u-greater-m}) after optimizing the $\alpha$ term. The equality in distribution at the start transfers the probability statement from $\tbd$ to $\bd$.
\end{proof}

\subsection{Lemmas for the Proof of
Theorem~\ref{theorem:concentration-tildeg-u-greater-m}}
\label{sec:lemmas-concentration-tildeg-u-greater-m}
We present the lemmas used in the proof of Theorem~\ref{theorem:concentration-tildeg-u-greater-m}, along with the proofs of these lemmas.

First, we have the following lemmas, Lemma~\ref{lemma:ztd-pi-qi-repeat-number}, Lemma~\ref{lemma:ztd-pi-upper-bound}, and Lemma~\ref{lemma:Ztdi-random} for sampling with $\tbd$ instead of $\bd$ , which are the counterparts corresponding to Lemma~\ref{lemma:zd-pi-qi-repeat-number}, Lemma~\ref{lemma:zd-pi-upper-bound}, and Lemma~\ref{lemma:Zdi-random},  respectively.

\begin{lemma}\label{lemma:ztd-pi-qi-repeat-number}
Suppose $u\ge m$ and $Q\in[2:m]$. Define
\bal
\tOmega^{(1)}(Q)
&\defeq\set{\tbd:\textup{there exists }\cP\subseteq[m],\ |\cP|\ge Q,
\ \tilde p(s)=\tilde p(t)\le m\textup{ for all }s,t\in\cP},
\label{eq:ztd-pi-repeat-set}\\
\tOmega^{(2)}(Q)
&\defeq\{\tbd:\textup{there exists }\cP\subseteq[m],\ |\cP|\ge Q,
\ \tilde q(s)\le m\textup{ for all }s\in\cP,\nonumber\\
&\hspace{35mm}\bZ_{\tbd^{(s)}}(\tilde q(s))
=\bZ_{\tbd^{(t)}}(\tilde q(t))\textup{ for all }s,t\in\cP\}.
\label{eq:ztd-qi-repeat-set}
\eal
The second event is understood on the joint probability space of $\tbd$ and the independent replacement draws, with $\tOmega(Q)$ lifted to that space. For every realization of the replacement draws,
\bal\label{eq:ztd-pi-qi-repeat-number-prob-Omega12}
\tOmega^{(1)}(Q)\subseteq\tOmega(Q),\qquad
\tOmega^{(2)}(Q)\subseteq\tOmega(Q).
\eal
Moreover,
\bal\label{eq:tomega-Q-prob-bound}
\Prob{\tOmega(Q)}
&\le\binom mQ\pth{\frac1m\sum_{i=1}^m\frac1{n-i+1}}^{Q-1}
<\binom mQ\pth{\frac{\ln2}{m}}^{Q-1},\nonumber\\
\Prob{\tOmega(Q)}
&\le\frac{m-Q+1}{2^{Q-1}}<\frac m{2^{Q-1}}.
\eal
In particular, $\Prob{\tOmega(2)}\le m^2/u$. For $m\ge2$, the choice
$Q=\min\{m,\lceil\log_2(4m)\rceil-1\}$ satisfies
\bal\label{eq:tomega-Q-weighted-prob-bound}
Q\Prob{\tOmega(Q)^c}+m\Prob{\tOmega(Q)}\le\log_2(4m).
\eal
\end{lemma}
\begin{proof}
The chain argument of Lemma~\ref{lemma:zd-pi-qi-repeat-number} applies with the number of draws changed from $u$ to $m$. In particular, equal finite values of $\tilde p(i)$, or equal outputs $\bZ_{\tbd^{(i)}}(\tilde q(i))$ for $\tilde q(i)\le m$, imply that the corresponding values $i_0$ coincide. Consecutive occurrences of the same value $i_0$ form a chain. This proves the inclusions, including their asserted pathwise interpretation.

For fixed $j_1<\cdots<j_Q$ in $[m]$, the probability of the chain
$\td_{j_\ell}=j_{\ell+1}$ for every $\ell\in[Q-1]$ equals
$\prod_{\ell=1}^{Q-1}(n-j_\ell+1)^{-1}$. Summing over chains and applying Maclaurin's inequality to the increasing sequence $(n-i+1)^{-1}$ gives the first bound in (\ref{eq:tomega-Q-prob-bound}). Since $u\ge m$,
$\sum_{i=1}^m(n-i+1)^{-1}<\ln2$. Alternatively, from any fixed chain vertex $j\in[m]$, the conditional probability of a further step in $[m]$ is
$(m-j)/(n-j+1)\le1/2$. A union bound over the at most $m-Q+1$ possible starting vertices proves the second bound. For $Q=2$, using
$(m-j)/(n-j+1)\le m/u$ instead gives $m^2/u$.

Finally, the proof of (\ref{eq:omega-Q-weighted-prob-bound}) uses only these two probability estimates and the size of the sampled set. Replacing $u$ there by $m$ gives (\ref{eq:tomega-Q-weighted-prob-bound}), including the cases $m=2,\ldots,8$ treated explicitly in that proof.
\end{proof}

\begin{lemma}\label{lemma:ztd-pi-upper-bound}
Suppose $\tbd \notin \tOmega(Q)$ where $\tOmega(Q)$ is defined in
(\ref{eq:omega-tQ-def}) and $Q \in [2:m]$. Then For any $h \in \cH$, we have
\bal
\sum\limits_{i \in \tilde \cA_{h,2} } \pth{h(\bZ_{\tbd}(\tilde p(i)))}^2 &\le (Q-1) \sum\limits_{i=1}^m \pth{h(\bZ_{\tbd}(i))}^2. \label{eq:ztd-pi-upper-bound}
\eal
\end{lemma}
\begin{proof}
(\ref{eq:ztd-pi-upper-bound}) can be proved by applying the argument for proof of (\ref{eq:zd-pi-upper-bound}) in Lemma~\ref{lemma:zd-pi-upper-bound} to $\tbd$ and $\bZ_{\tbd}$.
\end{proof}

\begin{lemma}
\label{lemma:Ztdi-random}
Suppose $\tbd \notin \tOmega(Q)$ where $ \tOmega(Q)$ is defined in
(\ref{eq:omega-tQ-def}) and $Q \in [2:m]$. Then for any $h \in \cH$, we have
\bal
\Expect{}{\sum\limits_{i =1}^m \pth{h(\bZ_{\tbd^{(i)}}(i))}^2 \longmid \tbd} &\le \frac{nm}{u} T_n(h), \label{eq:Ztdi-random-m-i} \\
 \Expect{}{\sum\limits_{i \in \tilde \cA_{h,1} }
 \pth{h(\bZ_{\tbd^{(i)}}(\tilde q(i)))}^2
\longmid \tbd  } &\le\frac{2(Q-1)nm}{u}T_n(h) \label{eq:Ztdi-random-m-q}.
\eal
\end{lemma}
\begin{proof}
The proof of Lemma~\ref{lemma:Zdi-random} applies with $\bd,u,m$ replaced by $\tbd,m,u$, respectively. In particular, $\Prob{\tilde q(i)\le m\mid\tbd}=(m-i)/(n-i+1)\le m/u$, and each value $i_0$ occurs at most $Q-1$ times on $\tOmega(Q)^c$. Thus (\ref{eq:Ztdi-random-m-i}) holds, and the second conditional expectation is at most $(Q-1)nmT_n(h)/u$, which implies (\ref{eq:Ztdi-random-m-q}).
\end{proof}

\begin{lemma}
\label{lemma:td-gd-RC}
\bal\label{eq:td-RC}
\Expect{\bd}{t_u(\bd,\cH')}  = \cfrakR^+_u(\cH'), \quad \Expect{\tbd}{t_m(\tbd,\cH')} = \cfrakR^+_m(\cH').
\eal
Moreover, for $g(\bd)$ defined in (\ref{eq:def-g}), we have
\bal\label{eq:gd-RC}
\Expect{\bd}{g(\bd)} \le \cfrakR^+_u(\cH) + \cfrakR^-_m(\cH).
\eal
\end{lemma}
\begin{proof}

$\Expect{\bd}{t_u(\bd)}  = \cfrakR^+_u(\cH')$ follows from the definition of Transductive Complexity in (\ref{eq:TC-def}). We have
\bals
\Expect{\tbd}{t_m(\tbd,\cH')} &= \Expect{\tbd}{\sup_{h \in \cH'} \pth{\cL_h(\set{\bZ_{\tbd}}) - \cL_n(h)}} 
=\Expect{\bd}{\sup_{h \in \cH'} \pth{\cL_h(\overline{\set{\bZ_{\bd}}}) - \cL_n(h)}} = \cfrakR^+_m(\cH').
\eals
where the second last equality is due to the fact that
$\set{\bZ_{\tbd}}$ and  $\overline{\set{\bZ_{\bd}}}$ have the same distribution, that is, they are sets of size $m$ sampled uniformly from $[n]$ without replacement, which proves (\ref{eq:td-RC}).

We now prove (\ref{eq:gd-RC}). We first have
\bals
\Expect{\bd}{g(\bd)} &=
\Expect{\bd}{\sup_{h \in \cH} \pth{\cL_u(h)-\cL_h(\overline{\set{\bZ_{\bd}}}) }} =\Expect{\bd}{\sup_{h \in \cH} \pth{\cL_u(h)-\cL_n(h)+\cL_n(h)-\cL_h(\overline{\set{\bZ_{\bd}}}) }}  \nonumber \\
&\stackrel{\circled{1}}{\le} \underbrace{\Expect{\bd}{\sup_{h \in \cH} \pth{\cL_u(h)-\cL_n(h) }}}_{\cfrakR^+_u(\cH)} + \underbrace{\Expect{\bd}{\sup_{h \in \cH} \pth{\cL_n(h)-\cL_h(\overline{\set{\bZ_{\bd}}}) }}}_{\cfrakR^-_m(\cH)} 
=\cfrakR^+_u(\cH)+ \cfrakR^-_m(\cH)
\eals
Here $\circled{1}$ follows from the sub-additivity of supremum.

\end{proof}

\subsection{Special Case of Theorem~\ref{theorem:main-inequality-TLC}}
\label{sec:special-case-prelim-version-appendix}
For comparison, the $Q=2$ concentration statement in the preliminary version
\citep{yang2025a-concentration-sampling-without-replacement} applies in the
unbalanced regime. When $m\ge u$ and $m\gg u^2$, a sharper specialization of the argument for
Theorem~\ref{theorem:concentration-g-m-greater-u}, retaining the exact exchange factor $(1/u+1/m)^2$ and using the expected chain multiplicity after averaging over an independent relabeling, yields the following bound:
with probability at least $1-\exp(-x) - u^2/m$ over $\bd$,
\bal\label{eq:concentration-g-m-greater-u-square}
g(\bd) \le &\Expect{\bd}{g(\bd)} + 8\sqrt{\frac{5rx}{u}}
+2\sqrt{2} \inf_{\alpha > 0} \pth{\frac{\cfrakR^+_u(\cH^2) }{\alpha} + \frac{2\alpha x}{u}}  +  \frac {8H_0^2x}{u},
\eal
Similarly, when
$u \gg m^2$, the corresponding sharper specialization of the argument for Theorem~\ref{theorem:concentration-tildeg-u-greater-m}, together with the equality in distribution of $\tg(\tbd)$ and $g(\bd)$, gives the following bound:
with probability at least $1-\exp(-x) - m^2/u$ over $\bd$,
\bal\label{eq:concentration-g-u-greater-m-square}
g(\bd) \le &\Expect{\bd}{g(\bd)} + 8\sqrt{\frac{5rx}{m}}
+2\sqrt{2} \inf_{\alpha > 0} \pth{\frac{\cfrakR^+_m(\cH^2) }{\alpha} + \frac{2\alpha x}{m}}  +  \frac {8H_0^2x}{m},
\eal
The direct proof of Theorem~\ref{theorem:concentration-g-m-greater-u}
instead yields (\ref{eq:concentration-g-m-greater-u-direct-tail}) for every
$m\ge u$.

\subsection{Proof of Theorem~\ref{theorem:main-inequality-TLC-nonnegative-func-class}
and Corollary~\ref{corollary:concentration-gu}}

The proof of Theorem~\ref{theorem:main-inequality-TLC-nonnegative-func-class}
uses the same direct entropy averaging over independent relabelings, with the nonnegativity of the class used to bound one-exchange differences.
\begin{proof}
[\textbf{Proof of Theorem~\ref{theorem:main-inequality-TLC-nonnegative-func-class}}]
As before, passing to the closure of the restrictions of $\cH$ to $[n]$ allows us to take maximizers without changing any supremum. The one-exchange identities used below are combinatorial and do not require the lower bound on $H_0$ imposed in Theorem~\ref{theorem:main-inequality-TLC}. We use the entropy inequality (\ref{eq:concentration-g-m-greater-u-product-entropy}) in Section~\ref{sec::math-tools} and an independent uniform relabeling $\pi$ of $[n]$. The relabeling preserves $0\le h\le H_0$ and $\cL_n(h)\le r$, and makes the sampled subset independent of the chain event generated by the draws.

\noindent \textbf{The case $m\ge u$.}
Suppose first that $u\ge2$ and fix $Q\in[2:u]$. Let $h^*$ attain $g(\bd)$. By (\ref{eq:concentration-g-m-greater-u-direct-diff}), the nonnegative part of the one-coordinate change is bounded by $(1/u+1/m)$ times the value of $h^*$ at the old endpoint of Lemma~\ref{lemma:uniform-draw-diff}: the new endpoint contributes with a minus sign and $h^*\ge0$. The multiplicity of each old endpoint is at most $Q$ on $\Omega(Q)^c$ and at most $u$ everywhere. Since $h^2\le H_0h$,
\bal\label{eq:concentration-g-m-greater-u-nonnegative-direct-variance}
V_+(g\mid\bd)
\le\begin{cases}
\frac{4H_0Q}{u}\pth{t_u(\bd,\cH)+r},&\bd\notin\Omega(Q),\\[2mm]
4H_0\pth{t_u(\bd,\cH)+r},&\bd\in\Omega(Q).
\end{cases}
\eal
Applying the same old-endpoint argument to a maximizer of $t_u(\bd,\cH)$ gives
\bal\label{eq:concentration-tu-nonnegative-direct-variance}
V_+(t_u(\bd,\cH)\mid\bd)
\le\begin{cases}
\frac{H_0Q}{u}\pth{t_u(\bd,\cH)+r},&\bd\notin\Omega(Q),\\[2mm]
H_0\pth{t_u(\bd,\cH)+r},&\bd\in\Omega(Q).
\end{cases}
\eal
These inequalities also hold for $u=1$ with only their second lines and $Q=1$, since $\Omega(1)$ is the whole sample space.

Take $Q=\min\{u,\lceil\log_2(4u)\rceil-1\}$ for $u\ge2$ and $Q=1$ for $u=1$. For each fixed relabeling, apply (\ref{eq:concentration-g-m-greater-u-product-entropy}) to the coordinates of $\bd$, then average over $\pi$ as in Step 2 of the proof of Theorem~\ref{theorem:concentration-g-m-greater-u}. Lemma~\ref{lemma:zd-pi-qi-repeat-number} yields
$Q\Prob{\Omega(Q)^c}+u\Prob{\Omega(Q)}\le\log_2(4u)$ for $u\ge2$. The same inequality is immediate for $u=1$. Consequently, with $\theta=H_0\log_2(4u)/u$, for every $\lambda>0$,
\bal\label{eq:concentration-g-m-greater-u-nonnegative-entropy}
\lambda\frac{\Expect{\bd}{g(\bd)e^{\lambda g(\bd)}}}{\Expect{\bd}{e^{\lambda g(\bd)}}}
-\log\Expect{\bd}{e^{\lambda g(\bd)}}
&\le4\theta\lambda^2
\pth{\frac{\Expect{\bd}{t_u(\bd,\cH)e^{\lambda g(\bd)}}}{\Expect{\bd}{e^{\lambda g(\bd)}}}+r},\\
\lambda\frac{\Expect{\bd}{t_u(\bd,\cH)e^{\lambda t_u(\bd,\cH)}}}{\Expect{\bd}{e^{\lambda t_u(\bd,\cH)}}}
-\log\Expect{\bd}{e^{\lambda t_u(\bd,\cH)}}
&\le\theta\lambda^2
\pth{\frac{\Expect{\bd}{t_u(\bd,\cH)e^{\lambda t_u(\bd,\cH)}}}{\Expect{\bd}{e^{\lambda t_u(\bd,\cH)}}}+r}.
\eal
Integrating the second inequality and using
$\Expect{\bd}{t_u(\bd,\cH)}=\cfrakR_u^+(\cH)\ge0$ gives
\bal\label{eq:concentration-g-m-greater-u-nonnegative-mgf-tu}
\log\Expect{\bd}{e^{\lambda t_u(\bd,\cH)}}
\le\frac{\lambda\pth{\cfrakR_u^+(\cH)+\theta\lambda r}}
{1-\theta\lambda},\qquad 0<\lambda<1/\theta.
\eal
The relative-entropy variational formula, applied to the exponential tilt by $g$, bounds the tilted expectation of $t_u$ in the first inequality of (\ref{eq:concentration-g-m-greater-u-nonnegative-entropy}) by its entropy term plus $\log\Expect{\bd}{e^{\lambda t_u}}$, divided by $\lambda$. Inserting (\ref{eq:concentration-g-m-greater-u-nonnegative-mgf-tu}) and integrating the resulting differential inequality, for $0<\lambda<1/(4\theta)$, yields
\bal\label{eq:concentration-g-m-greater-u-nonnegative-mgf-g}
\log\Expect{\bd}{e^{\lambda(g(\bd)-\Expect{\bd}{g(\bd)})}}
\le\frac{4\theta\lambda^2\pth{\cfrakR_u^+(\cH)+r}}
{1-4\theta\lambda}.
\eal
Indeed, before integration the entropy term is bounded by
$4\theta\lambda^2(\cfrakR_u^+(\cH)+r)/
\{(1-4\theta\lambda)(1-\theta\lambda)\}$, which is at most the same numerator divided by $(1-4\theta\lambda)^2$. Chernoff's bound therefore shows that, with probability at least $1-e^{-x}$,
\bal\label{eq:concentration-gd-TLC-nonnegative-func-class-m-greater-u}
g(\bd)\le\Expect{\bd}{g(\bd)}
+4\sqrt{\frac{H_0\log_2(4u)\pth{\cfrakR_u^+(\cH)+r}x}{u}}
+\frac{4H_0\log_2(4u)x}{u}.
\eal

\noindent \textbf{The case $u\ge m$.}
Use the process $\tg(\tbd)$ from (\ref{eq:test-train-g-td}), which has the same distribution as $g(\bd)$. If $h^*$ attains $\tg(\tbd)$, then (\ref{eq:concentration-tildeg-u-greater-m-direct-diff}) and $h^*\ge0$ show that the positive change is at most $(1/u+1/m)$ times $h^*$ at the new endpoint. In Case 1 of Lemma~\ref{lemma:uniform-draw-diff-tE}, this is the value $i_0$ selected at $\tilde q(i)$ in the replacement run. In Cases 2 and 3, it is $\bZ_{\tbd^{(i)}}(i)$. The conditional selection probability and multiplicity argument used in (\ref{eq:concentration-tildeg-u-greater-m-q-prob})--(\ref{eq:concentration-tildeg-u-greater-m-new-indices}) applies to every nonnegative function. In particular, conditional on $\tbd$,
\bal
\Expect{}{\sum_{i\in\tilde\cA_{h^*,1}}h^*(i_0)\mid\tbd}
+\Expect{}{\sum_{i\in\tilde\cA_{h^*,2}\cup\tilde\cA_{h^*,3}}h^*(\bZ_{\tbd^{(i)}}(i))\mid\tbd}
\le\begin{cases}
\frac{Qmn}{u}\cL_n(h^*),&\tbd\notin\tOmega(Q),\\
\frac{m^2n}{u}\cL_n(h^*),&\tbd\in\tOmega(Q).
\end{cases}
\eal
For the second line, fix $i$ and condition on $\tbd$. As $\td'_i$ ranges over $[i:n]$, each value can be the endpoint $\bZ_{\tbd^{(i)}}(i)$ for at most one draw. The exceptional endpoint $i_0$ can instead occur at $\tilde q(i)$ for at most $m-i$ draws. Hence each value contributes at most $m-i+1$ times to the sum over the $n-i+1$ possible draws. Since $h^*\ge0$, the conditional expectation for this $i$ is at most
$(m-i+1)n\cL_n(h^*)/(n-i+1)\le mn\cL_n(h^*)/u$. Summing over $i$ gives the displayed bound. Using $h^2\le H_0h$, $(1/u+1/m)^2\le4/m^2$, and $n/u\le2$, we obtain
\bal\label{eq:concentration-g-u-greater-m-nonnegative-direct-variance}
V_+(\tg\mid\tbd)
\le\begin{cases}
\frac{8H_0Q}{m}r,&\tbd\notin\tOmega(Q),\\[2mm]
8H_0r,&\tbd\in\tOmega(Q).
\end{cases}
\eal
Here $\cL_n(h^*)\le r$ was used. The second line also covers $m=1$ with $Q=1$.

Again apply (\ref{eq:concentration-g-m-greater-u-product-entropy}) to the $m$ independent coordinates of $\tbd$ for each fixed uniform relabeling and average over the relabeling. Choose $Q$ as in Lemma~\ref{lemma:ztd-pi-qi-repeat-number}, or $Q=1$ when $m=1$. The weighted probability bound for $\tOmega(Q)$, with its immediate $m=1$ case, gives, for all $\lambda>0$,
\bal\label{eq:concentration-g-u-greater-m-nonnegative-entropy}
\lambda\frac{\Expect{\tbd}{\tg(\tbd)e^{\lambda\tg(\tbd)}}}{\Expect{\tbd}{e^{\lambda\tg(\tbd)}}}
-\log\Expect{\tbd}{e^{\lambda\tg(\tbd)}}
\le\frac{8H_0\log_2(4m)r\lambda^2}{m}.
\eal
Integrating from zero and applying Chernoff's bound proves that, with probability at least $1-e^{-x}$,
\bal\label{eq:concentration-gd-TLC-nonnegative-func-class-u-greater-m}
\tg(\tbd)\le\Expect{\tbd}{\tg(\tbd)}
+4\sqrt{\frac{2H_0\log_2(4m)rx}{m}}.
\eal
The equality in distribution transfers this bound to $g(\bd)$.

Finally, use $\sqrt{a+b}\le\sqrt a+\sqrt b$ in (\ref{eq:concentration-gd-TLC-nonnegative-func-class-m-greater-u}),
$2\inf_{\alpha>0}(\cfrakR_p^+(\cH)/\alpha+\alpha H_0x/N_{u,m,\delta})
=4\sqrt{\cfrakR_p^+(\cH)H_0x/N_{u,m,\delta}}$ for $p=\min\{u,m\}$, and
$\log_2(4p)\le\log_2(4p/\delta)$. The displayed bound in (\ref{eq:concentration-gd-TLC-nonnegative-func-class}) follows in both cases. Its failure probability is at most $e^{-x}$, and hence at most $e^{-x}+\delta$.
\end{proof}

\begin{proof}[\textbf{\textup{Proof of Corollary~\ref{corollary:concentration-gu}} }]
We have $g^+_u(\bd)  = \frac mn g(\bd)$. Therefore,
\bals
g^+_u(\bd) - \Expect{\bd}{g^+_u(\bd)} = \frac mn \pth{g(\bd) - \Expect{\bd}{g(\bd)}},
\eals
where $g$ is defined in (\ref{eq:def-g}). Then the bound for
$g^+_u - \Expect{}{g^+_u}$ in (\ref{eq:concentration-gu})
follows from (\ref{eq:concentration-gd-TLC}) in
Theorem~\ref{theorem:main-inequality-TLC}.

We further have
\bal\label{eq:concentration-gu--gd}
g^-_u(\bd) =
\frac mn \sup_{h \in \cH}
\pth{\cL_h(\overline{\set{\bZ_{\bd}}})-\cL_h(\set{\bZ_{\bd}})}
\overset{\textup{dist}}{=}\frac mn \sup_{h \in \cH} \pth{\cL_h(\set{\bZ_{\tbd}})-\cL_h(\overline{\set{\bZ_{\tbd}}}) },
\eal
where $\overset{\textup{dist}}{=}$ indicates that the random variables on both sides follow the same distribution, since both
$\set{\overline{\set{\bZ_{\bd}}}}$ and $\set{\bZ_{\tbd}}$ are both subsets
of size $m$ sampled uniformly from $[n]$ without replacement.
Taking $\set{\bbx_i}_{i \in \overline{\bZ_{\tbd}}}$ as the training features and $\set{\bbx_i}_{i \in \bZ_{\tbd}}$ as the test features, then we can
repeat the proof of Theorem~\ref{theorem:main-inequality-TLC}  to obtain the following concentration inequality: for every $x > 0$ and $\delta \in (0,1)$, with probability at least $1-\exp(-x) - \delta$ over $\tbd$,
\bal\label{eq:concentration-gd-u-m-swapped}
&\sup_{h \in \cH} \pth{\cL_h(\set{\bZ_{\tbd}})-\cL_h(\overline{\set{\bZ_{\tbd}}})} -
\Expect{\tbd}{\sup_{h \in \cH} \pth{\cL_h(\set{\bZ_{\tbd}})-\cL_h(\overline{\set{\bZ_{\tbd}}})}} \nonumber \\
&\le  4\sqrt{\frac{10rx}
{N_{u,m,\delta}}} +2\sqrt{2} \inf_{\alpha > 0}
\pth{\frac{\cfrakR^+_{\min\set{u,m}}(\cH^2) }{\alpha} +
\frac{\alpha x}{N_{u,m,\delta}}}
+  \frac {4 H_0^2x}{N_{u,m,\delta}}.
\eal
Then the bound for
$g^-_u - \Expect{}{g^-_u}$ in (\ref{eq:concentration-gu})
follows from (\ref{eq:concentration-gu--gd}) and
(\ref{eq:concentration-gd-u-m-swapped}).
\end{proof}

\subsection{Proof of Lemma~\ref{lemma:TLC-delta-ell-f} and
Theorem~\ref{theorem:TLC-delta-star-ell-f} }
\begin{proof}[\textbf{\textup{Proof of Lemma~\ref{lemma:TLC-delta-ell-f}}}]

We only need to prove that $T_n(h) \le \tT_n(h)$. For any $h \in \cH$, by the definition of $\tT_n(h)$ in (\ref{eq:tTn-def}), for every $\eps > 0$, there exist $f_1,f_2 \in \cF$ such that $h = \ell_{f_1}-\ell_{f_2}$, and
$2B \cL_n(\ell_{f_1} - \ell_{f^*_n}) + 2B \cL_n(\ell_{f_2} - \ell_{f^*_n})
< \tT_n(h) + \eps$. Therefore,
\bals
T_n(h) &= \cL_n\pth{(\ell_{f_1}-\ell_{f_2})^2} 
\le 2 T_n\pth{\ell_{f_1} - \ell_{f^*_n}} + 2 T_n\pth{\ell_{f_2} - \ell_{f^*_n}} \\
&\le 2B \cL_n(\ell_{f_1} - \ell_{f^*_n}) + 2B \cL_n(\ell_{f_2} - \ell_{f^*_n})
< \tT_n(h) +\eps,
\eals
where the first  inequality follows from the Cauchy-Schwarz inequality, and the second inequality is due to Assumption~\ref{assumption:main}(2).
It follows that $T_n(h) \le \tT_n(h)$.
As a result, we can apply Theorem~\ref{theorem:TLC} with $\tT_n(\cdot)$ defined in this corollary.
\end{proof}
%
%
%

\begin{lemma}\label{lemma:concentration-gu-gm-Delta*}
Suppose Assumption~\ref{assumption:main} holds. Furthermore, let $\psi^*_{u,m}$ be a sub-root function and $r^*$ be the fixed point of $\psi^*_{u,m}$.
Assume that for all $r \ge r^*$,
\bal
\psi^*_{u,m} (r) \ge  \max\Bigg\{&\Expect{}{\sup_{h \colon h \in\Delta^*_{\cF},B \cL_n(h) \le r} R^-_{u} h}, \Expect{}{\sup_{h \colon h \in\Delta^*_{\cF},B \cL_n(h) \le r} R^-_{m} h}, \Expect{}{\sup_{h \colon h \in\Delta^*_{\cF},B \cL_n(h) \le r} R^+_{\min\set{u,m}} h^2}\Bigg\}, \label{eq:TLC-cond-um-delta-star-ell-f}
\eal
Then for any fixed constant $K_0 > 1$, there exists an absolute positive constant $\hat c_2$ depending on $K_0,L_0$ such that for every $x > 0$
and every $\delta \in (0,1)$, with probability at least $1-\exp(-x)-\delta$ over $\set{\bZ}$,
\bal\label{eq:concentration-gm-Delta*-fixedpoint}
\cL_n(h) - \cL_m(h)\le\frac{B\cL_n(h)}{K_0} + \hat c_2\pth{r^* + \frac{x}{N_{u,m,\delta}}}, \quad \forall h \in \Delta^*_{\cF}.
\eal
Similarly, with probability at least $1-\exp(-x)-\delta$ over $\set{\bZ}$,
\bal\label{eq:concentration-gu--Delta*-fixedpoint}
\cL_n(h) - \cL_u(h) \le\frac{B\cL_n(h)}{K_0} + \hat c_2\pth{r^* +
\frac{x}{N_{u,m,\delta}}},  \quad \forall h \in \Delta^*_{\cF}.
\eal
where $N_{u,m,\delta}$ is defined in (\ref{eq:N-um-delta-def}).
\end{lemma}
\begin{proof}
We first establish concentration bounds for an arbitrary deterministic class $\cH$ containing zero and satisfying $|h(i)|\le L_0$ for every $h\in\cH$ and $i\in[n]$, and $\sup_{h\in\cH}T_n(h)\le r$, where $r>0$. Define
\bal\label{eq:gm-def}
g^-_m(\bd)\defeq\sup_{h\in\cH}\pth{\cL_n(h)-\cL_m(h)}.
\eal
Then $g^-_m(\bd)=(u/n)g(\bd)$ and $\Expect{\bd}{g^-_m(\bd)}=\cfrakR^-_m(\cH)$. Theorem~\ref{theorem:main-inequality-TLC}, with $H_0=L_0$ and $\alpha=1$, therefore gives, with probability at least $1-\exp(-x)-\delta$,
\bal\label{eq:concentration-gu-gm-seg1}
g^-_m(\bd)
&\le\cfrakR^-_m(\cH)
+2\sqrt2\,\cfrakR^+_{\min\set{u,m}}(\cH^2)
+4\sqrt{\frac{10rx}{N_{u,m,\delta}}}
+\frac{(4L_0^2+2\sqrt2)x}{N_{u,m,\delta}}.
\eal
Here we used $u/n\le1$ and the nonnegativity of the deviation bound. Likewise, $g^-_u(\bd)=(m/n)\sup_{h\in\cH}(\cL_m(h)-\cL_u(h))$. Applying the same theorem to $\set{-h:h\in\cH}$, whose squared class is $\cH^2$, gives, separately with probability at least $1-\exp(-x)-\delta$,
\bal\label{eq:concentration-gu-gm-seg2}
g^-_u(\bd)
&\le\cfrakR^-_u(\cH)
+2\sqrt2\,\cfrakR^+_{\min\set{u,m}}(\cH^2)
+4\sqrt{\frac{10rx}{N_{u,m,\delta}}}
+\frac{(4L_0^2+2\sqrt2)x}{N_{u,m,\delta}}.
\eal

We now take $\cH=\Delta^*_{\cF}$ and set $\tT_n(h)=B\cL_n(h)$. Assumption~\ref{assumption:main} gives $0\le T_n(h)\le\tT_n(h)$, and $0\in\cH$ with $\tT_n(0)=0$.
If $r^*=0$, the fixed-point convention after Definition~\ref{def:sub-root-function} gives $\psi^*_{u,m}\equiv0$. Since $0\le B\cL_n(h)\le BL_0$ for every $h\in\cH$, (\ref{eq:TLC-cond-um-delta-star-ell-f}) at $r=BL_0$ implies $\cfrakR^-_m(\cH)=\cfrakR^-_u(\cH)=0$. Both $g^-_m$ and $g^-_u$ are nonnegative because $0\in\cH$, so they vanish almost surely. This proves both asserted inequalities in this case. Henceforth assume $r^*>0$.

For $r\ge r^*$, use $w(h)$ and $\cH^{(r)}$ from (\ref{eq:Gr-def}), with $\lambda=4$. The argument at the start of the proof of Theorem~\ref{theorem:TLC} gives $|s(i)|\le L_0$ and $T_n(s)\le r$ for every $s\in\cH^{(r)}$. Consequently, (\ref{eq:concentration-gu-gm-seg1}) and (\ref{eq:concentration-gu-gm-seg2}) apply to this normalized class.
The positive-part shell argument in (\ref{eq:TLC-Ru+-seg1})--(\ref{eq:TLC-Ru+-seg4}), applied to the three empirical processes in (\ref{eq:TLC-cond-um-delta-star-ell-f}), gives
\bals
\max\set{\cfrakR^-_u(\cH^{(r)}),\cfrakR^-_m(\cH^{(r)})}
&\le5\psi^*_{u,m}(r)\le5\sqrt{rr^*},\\
\cfrakR^+_{\min\set{u,m}}\pth{\set{s^2:s\in\cH^{(r)}}}
&\le\frac{23}{7}\psi^*_{u,m}(r)\le\frac{23}{7}\sqrt{rr^*}.
\eals
Indeed, each localized class contains zero. Thus the positive part of each shell supremum is bounded by the corresponding localized supremum in (\ref{eq:TLC-cond-um-delta-star-ell-f}). The shell weights are $4^{-(t+1)}$ for $h$ and $4^{-2(t+1)}$ for $h^2$, exactly as in that argument. The last inequalities use the sub-root property and $r\ge r^*>0$.

Applying (\ref{eq:concentration-gu-gm-seg1}) to $\cH^{(r)}$ therefore bounds $\sup_{s\in\cH^{(r)}}(\cL_n(s)-\cL_m(s))$ by
\bals
P(r)\defeq\sqrt r\pth{c'_0\sqrt{r^*}
+c'_1\sqrt{\frac{x}{N_{u,m,\delta}}}}
+\frac{c'_2x}{N_{u,m,\delta}},
\eals
with probability at least $1-\exp(-x)-\delta$, where $c'_0=(70+46\sqrt2)/7$, $c'_1=4\sqrt{10}$, and $c'_2=4L_0^2+2\sqrt2$. The coefficient of $\sqrt{rr^*}$ obtained above is $5+46\sqrt2/7$, which is at most $c'_0$. By (\ref{eq:concentration-gu-gm-seg2}), the same bound holds separately for $\sup_{s\in\cH^{(r)}}(\cL_n(s)-\cL_u(s))$.
Set
\bals
r_1\defeq8K_0\pth{4K_0\pth{c'_0\sqrt{r^*}
+c'_1\sqrt{\frac{x}{N_{u,m,\delta}}}}^2
+\frac{c'_2x}{N_{u,m,\delta}}}.
\eals
Then $r_1\ge r^*$, and the quadratic estimate used in the proof of Theorem~\ref{theorem:TLC} gives $P(r_1)\le r_1/(4K_0)$. The scaling argument of Lemma~\ref{lemma:TLC-supp-lemma}, now applied to the linear functional $\cL_n-\cL_m$, yields
\bals
\cL_n(h)-\cL_m(h)
&\le\frac{B\cL_n(h)}{K_0}+\frac{r_1}{4K_0}\\
&\le\frac{B\cL_n(h)}{K_0}+c_0r^*+\frac{c_1x}{N_{u,m,\delta}},
\qquad h\in\Delta^*_{\cF},
\eals
where $c_0=16K_0(c'_0)^2$ and $c_1=16K_0(c'_1)^2+2c'_2$. Applying the same argument to $\cL_n-\cL_u$ gives the second assertion with its own probability at least $1-\exp(-x)-\delta$. Taking $\hat c_2=\max\set{c_0,c_1}$ proves the lemma with a constant depending only on $K_0,L_0$.

\end{proof}

\begin{proof}[\textbf{\textup{Proof of Theorem~\ref{theorem:TLC-delta-star-ell-f}}}]
Let $\cH =  \Delta^*_{\cF}$. It follows from
 (\ref{eq:concentration-gu--Delta*-fixedpoint}) in
 Lemma~\ref{lemma:concentration-gu-gm-Delta*} that with high probability, for all $h \in \cH$,
\bals
\cL_n(h) - \cL_u(h) &\le \frac{B\cL_n(h)}{K_0} + \hat c_2\pth{r^* + \frac{x}{N_{u,m,\delta}}}
\eals
holds for a fixed constant $K_0 > 1$, and $\hat c_2$ depends on $K_0$ and $L_0$. We set $h = \ell_{\hat f_u} - \ell_{f^*_n}$ in the above inequality, and note that
$\cL_u(h) = \cL_{\ell_{\hat f_{u}}}(\bZ_{\bd}) - \cL_{\ell_{f^*_n}}(\bZ_{\bd}) \le 0$ due to the optimality of $\hat f_{u}$. Let $K_0 >B$. Then the first upper bound in (\ref{eq:TLC-bound-delta-star-ell-f}) is proved by
the above inequality with constant $c_2 = \hat c_2/\pth{1-B/K_0}$.

Moreover, it follows from (\ref{eq:concentration-gm-Delta*-fixedpoint}) in Lemma~\ref{lemma:concentration-gu-gm-Delta*}
that with high probability,
for all $h \in \cH$,
\bals
\cL_n(h) - \cL_m(h)  \le\frac{B\cL_n(h)}{K_0} + \hat c_2\pth{r^* + \frac{x}{N_{u,m,\delta}}}.
\eals
We set $h = \ell_{\hat f_{m}} - \ell_{f^*_n}$ in the above inequality, and note that
$\cL_m(h) = \cL_m(\ell_{\hat f_{m}}) - \cL_m(\ell_{f^*_n}) \le 0$ due to the optimality of $\hat f_{m}$. Let $K_0 > B$. Then the second upper bound in (\ref{eq:TLC-bound-delta-star-ell-f}) is proved by the above inequality with the same constant $c_2 = \hat c_2/\pth{1- B/K_0}$.

\end{proof}

\subsection{Proof of Theorem~\ref{theorem:TLC-delta-ell-f-excess-risk-upper-bound-VC-dim}}

We need the following lemma on the VC dimension and VC-subgraph dimension of the translated function class $\cF - f^*$ in the proof of Theorem~\ref{theorem:TLC-delta-ell-f-excess-risk-upper-bound-VC-dim}. We use the standard definitions of VC dimension and VC-subgraph dimension; see, e.g., \citep{van1996weak}.

\begin{lemma}
\label{lemma:VC-dim-translation}
Suppose the VC dimension of the $\set{0,1}$-valued function class $\cF$ defined on $\bX_n$ is $\dVC$, and let $f^*\in\cF$ be fixed.
Then the VC dimension of the following function class $\cF'$,
\bals
\cF' \defeq \set{\abth{g} \colon g \in \cG},
\quad \textup{where }\cG \defeq \cF - f^*,
\eals
denoted as $\textup{VC-dim}(\cF')$, is also $\dVC$. Furthermore, the function class $\cG$ is a VC-subgraph class, and its VC-subgraph dimension, denoted as $\textup{VC-subgraph}(\cG)$, is $\dVC$.
\end{lemma}

\begin{proof}
[\textbf{Proof of
Theorem~\ref{theorem:TLC-delta-ell-f-excess-risk-upper-bound-VC-dim}}]

We apply Theorem~\ref{theorem:TLC-nonnegative-func-class} to prove
this theorem.

First, we define the function class
\bals
\cH = \set{\ell_f \colon f \in \cF}
\eals
with $\ell_f(i) = (f(\bbx_i)-y_i)^2 = \indict{f(\bbx_i) \neq y_i}$ for all $i \in [n]$. The zero function belongs to $\cH$ because the realizable predictor $f^*$ has zero loss on the full sample.
It follows from Theorem~\ref{theorem:TC-RC} that if
$\psi_{u,m}(r)$ is a sub-root function which  satisfies
\bal\label{eq:psi-u-m-goal}
\psi_{u,m}(r) \ge 2\max\set{
\Expect{}{\sup_{h \colon h \in \cH,\cL_n(h) \le r} R^{(\textup{ind})}_{\bsigma,\bY^{(u)}} h},
\Expect{}{\sup_{h \colon h \in \cH,\cL_n(h) \le r} R^{(\textup{ind})}_{\bsigma,\bY^{(m)}} h}}
\eal
where both $\bY^{(u)} = \set{Y_1,\ldots,Y_u}$ and
$\bY^{(m)} = \set{Y_1,\ldots,Y_m}$ are sampled uniformly and independently from $[n]$ with replacement, then such
$\psi_{u,m}$ meets the condition (\ref{eq:TLC-cond-psi-nonnegative}) for
Theorem~\ref{theorem:TLC-nonnegative-func-class}.
To this end, we now derive the upper bounds for
$\Expect{}{\sup_{h \colon h \in \cH,\cL_n(h) \le r} R^{(\textup{ind})}_{\bsigma,\bY^{(u)}} h}$ and
$\Expect{}{\sup_{h \colon h \in \cH,\cL_n(h) \le r} R^{(\textup{ind})}_{\bsigma,\bY^{(m)}} h}$.

Let $\cF-f^*$ denote the function class $\set{f - f^* \colon f \in \cF}$. Because the VC-dimension of $\cF$ is $\dVC$, it follows from Lemma~\ref{lemma:VC-dim-translation}
that the binary-valued function class
$\set{g^2 \colon g \in \cF - f^*}
=
\set{\abth{g} \colon g \in \cF - f^*}$
has VC dimension $\dVC$.
As a result, it  follows from Proposition~\ref{proposition:concentration-vc-class-ind}
with $K_0=2$
that for an absolute positive constant $c > 0$, for every $x' > 0$, with probability at least $1-\exp(-x')$ over $\bY^{(u)}$,
\bal\label{eq:TLC-delta-ell-f-excess-risk-upper-bound-VC-dim-seg1}
\frac{1}{u}\sum\limits_{i=1}^u f^2(Y_i)  - \frac{3}{2}T_n(f)
\le \underbrace{2c \pth{\frac{\dVC \log (ue/\dVC)}{u}+\frac {x'}u}}_{\defeq C(u,\dVC,x')},
\quad \forall f \in\cF - f^*.
\eal
We define
\bals
\cA \defeq \set{\bY^{(u)} \colon
\textup{ (\ref{eq:TLC-delta-ell-f-excess-risk-upper-bound-VC-dim-seg1}) holds}}
\eals
as the set of the values of $\bY^{(u)}$ such that
(\ref{eq:TLC-delta-ell-f-excess-risk-upper-bound-VC-dim-seg1}) holds, so we have $\Prob{\overline{\cA}} \le \exp(-x')$.

We then derive the upper bound for
$\Expect{}{\sup_{h \colon h \in \cH,\cL_n(h) \le r} R^{(\textup{ind})}_{\bsigma,\bY^{(u)}} h}$ as follows.

\bal\label{eq:TLC-delta-ell-f-excess-risk-upper-bound-VC-dim-u}
\Expect{}{\sup_{h \in \cH,\cL_n(h)\le r}
R^{(\textup{ind})}_{\bsigma,\bY^{(u)}}h}
\le C'_1\sqrt{\frac{\dVC\log(ue/\dVC)}{u}}\sqrt r
+C'_2\frac{\dVC\log(ue/\dVC)}{u}.
\eal
To prove this bound, condition on \(\bY^{(u)}\). The restrictions of
\(\cF-f^*\) to these \(u\) coordinates form a finite class of
cardinality at most
\(\sum_{j=0}^{\dVC}\binom{u}{j}\le(ue/\dVC)^{\dVC}\)
by the Sauer--Shelah lemma. Since \(f^*\in\cF\), this class contains
zero. On \(\cA\), every \(f\in\cF-f^*\) with \(T_n(f)\le r\) has
empirical squared norm at most \(3r/2+C(u,\dVC,x')\) by
(\ref{eq:TLC-delta-ell-f-excess-risk-upper-bound-VC-dim-seg1}).
The finite-class Rademacher bound, conditional on \(\bY^{(u)}\), is
therefore
\[
\Expect{\bsigma}{\sup_{f\in\cF-f^*,\,T_n(f)\le r}
R^{(\textup{ind})}_{\bsigma,\bY^{(u)}}f}
\le
\sqrt{\frac{2\dVC\log(ue/\dVC)}{u}
\pth{\frac32r+C(u,\dVC,x')}}.
\]
On \(\overline{\cA}\), the same supremum is at most one because
\(f\in\set{-1,0,1}\) on the full sample. Apply the contraction
inequality to \(\ell_f\), whose coordinatewise map as a function of
\(f-f^*\) is $2$-Lipschitz and vanishes at zero. Taking expectations
over \(\bY^{(u)}\), using
\(\Prob{\overline{\cA}}\le e^{-x'}\), and choosing
\(x'=\dVC\log(ue/\dVC)\) proves
(\ref{eq:TLC-delta-ell-f-excess-risk-upper-bound-VC-dim-u})
with absolute constants \(C'_1,C'_2\).

Now repeating the argument for
(\ref{eq:TLC-delta-ell-f-excess-risk-upper-bound-VC-dim-u}) and
replacing $u$ with $m$,
we have
\bal\label{eq:TLC-delta-ell-f-excess-risk-upper-bound-VC-dim-m}
\Expect{}{\sup_{h \in \cH \colon \cL_n(h) \le r} R^{(\textup{ind})}_{\bsigma,\bY^{(m)}} h}
\le C'_1  \sqrt{\frac{\dVC \log (me/\dVC)}{m}} \sqrt{r}
+C'_2 \frac{\dVC \pth{\log (me/\dVC)}}{m}.
\eal
It then follows from the fact that $u \ge m$,
(\ref{eq:TLC-delta-ell-f-excess-risk-upper-bound-VC-dim-u}),
(\ref{eq:TLC-delta-ell-f-excess-risk-upper-bound-VC-dim-m}),
and (\ref{eq:psi-u-m-goal}) that we can set
\bals
\psi_{u,m}(r)
=2C'_1  \sqrt{\frac{\dVC \log  (me/\dVC)}{m}} \sqrt{r}
+2C'_2 \frac{\dVC \log  (me/\dVC)}{m}.
\eals
For any $r \ge 0$ such that
$r \le r_{u,m}$, we have $r \le \psi_{u,m}(r)$. By solving for $r$
in this inequality, we have
\bal\label{eq:TLC-delta-ell-f-excess-risk-upper-bound-VC-dim-seg3-post}
r \lsim  \frac{\dVC \log   (me/\dVC)}{m}.
\eal
It follows from
(\ref{eq:TLC-bound-g-upper-bound-nonnegative-func-class})
and (\ref{eq:TLC-delta-ell-f-excess-risk-upper-bound-VC-dim-seg3-post})
that with probability at least $1-\exp(-x)-\delta$, for every $h \in \cH$,
\bal\label{eq:TLC-delta-ell-f-excess-risk-upper-bound-VC-dim-seg4}
\cL_u(h)
&\le \frac{K_0+1}{K_0-1}\cL_m(h) + c''_0 \frac{\dVC \log   (me/\dVC)}{m}
+ \frac{c'_1x}{N_{u,m,\delta}} \nonumber \\
&=\frac{K_0+1}{K_0-1}\cL_m(h) + c''_0 \frac{\dVC \log   (me/\dVC)}{m}
+c'_1 \frac{\pth{\log_2 (4m/\delta)}x}{m},
\eal
where $c''_0$ is an absolute positive constant.
We set $h= \ell_f$ to the empirical minimizer
$h = \ell_{\hat f_m}$ which minimizes the training loss
$\cL_m(h)$ in (\ref{eq:TLC-delta-ell-f-excess-risk-upper-bound-VC-dim-seg4}). Then
$0 \le \cL_m(\ell_{\hat f_m}) \le \cL_m(\ell_{f^*}) = 0$
so that $\cL_m(\ell_{\hat f_m}) = 0$. (\ref{eq:TLC-delta-ell-f-excess-risk-upper-bound-VC-dim})
then follows from (\ref{eq:TLC-delta-ell-f-excess-risk-upper-bound-VC-dim-seg4}).

\end{proof}

\begin{proof}
[\textbf{Proof of Lemma~\ref{lemma:VC-dim-translation}}]

We first show that $\textup{VC-dim}(\cF') = \dVC$.
We fix an arbitrary finite set
$S \defeq \set{x_1,\dots,x_{n'}} \subseteq \bX_n$.
Since $f^*$ is binary-valued, for each $i \in [n']$ we have
\bals
\abth{f(x_i)-f^*(x_i)}
=
\begin{cases}
f(x_i), & \text{if } f^*(x_i)=0, \\
1-f(x_i), & \text{if } f^*(x_i)=1.
\end{cases}
\eals
As a result, on the sample $S$, the class $\cF'$ is obtained from $\cF$ by flipping the labels in exactly those coordinates where $f^*(x_i)=1$. This coordinate-wise flip is a bijection on $\set{0,1}^{n'}$. Therefore, $S$ is shattered by $\cF$ if and only if it is shattered by $\cF'$. Hence
$\mathrm{VCdim}(\cF') = \mathrm{VCdim}(\cF) = \dVC$.

We next show that $\cG$ is a VC-subgraph class, and then show that $\textup{VC-subgraph}(\cG) = \dVC$.
For any $g \in \cG$, with $g = f-f^*$ for some $f \in \cF$, its subgraph satisfies
\bals
\set{(x,t) : t < g(x)}
=
\set{(x,t) : t + f^*(x) < f(x)}.
\eals
Now define the map $\Phi(x,t) \defeq \pth{x,\, t + f^*(x)}$.
Because $f^*$ is fixed, $\Phi$ is a bijection on $\bX_n \times \mathbb{R}$. Moreover, $\Phi$ sends the subgraph of $g$ exactly onto the subgraph of $f$. Therefore, the family of subgraphs of $\cG$ is in one-to-one correspondence with the family of subgraphs of $\cF$. Since VC-subgraph dimension is preserved under such fixed translations, $\cG$ is a VC-subgraph class, and its VC-subgraph dimension is the same as the VC-subgraph dimension of $\cF$, that is, $\textup{VCsub-dim}(\cG) = \textup{VCsub-dim}(\cF)$.

We also recall the standard fact that if $\cF$ is binary-valued, then the VC-subgraph dimension of $\cF$ is equal to the VC dimension of $\cF$, to be justified later, that is,
$\textup{VCsub-dim}(\cF)=\textup{VC-dim}(\cF)$.  As a result,
\bals
\textup{VCsub-dim}(\cG)=\textup{VCsub-dim}(\cF)=\textup{VC-dim}(\cF) = \dVC.
\eals
We now show the standard fact that $\textup{VCsub-dim}(\cF)=\textup{VC-dim}(\cF)$. Suppose a finite set $\set{(x_i,t_i)}_{i=1}^{n'}$
is shattered by the subgraphs of $\cF$. Since every $f \in \cF$ takes values in $\set{0,1}$, each $t_i$ must lie in $[0,1)$. Otherwise, if $t_i<0$, then $(x_i,t_i)$ belongs to every subgraph, and if $t_i \ge 1$, then $(x_i,t_i)$ belongs to no subgraph, so the set cannot be shattered. For $t_i \in [0,1)$, we have
\bals
t_i < f(x_i) \quad \Longleftrightarrow \quad f(x_i)=1.
\eals
Hence $\set{x_i}_{i=1}^{n'}$ is shattered by $\cF$, and therefore
$\textup{VCsub-dim}(\cF)\le \textup{VC-dim}(\cF)$.
Conversely, if $\set{x_i}_{i=1}^n$ is shattered by $\cF$, then $\set{(x_i,1/2)}_{i=1}^n$ is shattered by the subgraphs of $\cF$, so
$\textup{VCsub-dim}(\cF) \ge \textup{VC-dim}(\cF)$.
As a result, the two dimensions are equal.
\end{proof}

\subsection{Proof of Theorem~\ref{theorem:TLC-kernel}: TLC Excess Risk Bound for Transductive Kernel Learning}
\label{sec:TKL-proof}

Before presenting the proof of Theorem~\ref{theorem:TLC-kernel}, we introduce more background on the space $\cH_{\bX_n}$ and $\cH_K$.
We recall that $K \colon \cX \times \cX \to \RR$ is a positive definite kernel on the compact set $\cX \times \cX$, with $\cX \subseteq \RR^d$. If $K$ is continuous, the integral operator
$T_K \colon L^2(\cX,\mu^{(P)})\to L^2(\cX,\mu^{(P)})$ defined by
$(T_Kf)(\bx)=\int_{\cX}K(\bx,\bx')f(\bx')\diff\mu^{(P)}(\bx')$
is positive, self-adjoint, and compact, where $\mu^{(P)}$ is the probability measure of $P$. Its eigenfunctions with positive eigenvalues span the closure of its range. The proof below uses only the empirical operator $\hat T_n$ and does not require a population eigenbasis. We write $\max_{\bx\in\cX}K(\bx,\bx)\defeq\tau_0^2<\infty$.
We use $\cH_K(\mu)=\set{f\in\cH_K\colon\norm{f}{\cH_K}\le\mu}$ as defined above.

We let the gram matrix of $K$ over the full sample be $\bK \in \RR^{n \times n}, \bK_{ij} = K(\bbx_i,\bbx_j)$ for all $i,j \in [n]$, and $\bK_n \defeq \bK/n$. 
Then $\hat \lambda_1 \ge \hat \lambda_2 \ldots \ge \hat \lambda_n \ge 0$ are the eigenvalues of $\bK_n$,
and $\hat \lambda_1 \le \tr{\bK_n} \le \tau_0^2$.

In TKL, we introduce the operator $\hat T_n \colon \cH_K \to \cH_K$ which is important for our analysis. $\hat T_n$ is defined by
\bals
\hat T_n g \defeq \frac 1n \sum_{i=1}^n K(\cdot,\bbx_i) g(\bbx_i), \quad g \in \cH_K.
\eals
The nonzero eigenvalues of $\hat T_n$ are precisely the positive eigenvalues among $\set{\hat\lambda_q}_{q=1}^n$. For each $q\in[n]$ with $\hat\lambda_q>0$, put
$\Phi_q=(n\hat\lambda_q)^{-1/2}\sum_{j=1}^n K(\cdot,\bbx_j)[\bU^{(q)}]_j$,
where $\set{\bU^{(q)}}_{q=1}^n$ is an orthonormal eigenbasis of $\bK_n$ satisfying $\bK_n\bU^{(q)}=\hat\lambda_q\bU^{(q)}$ for every $q\in[n]$. These $\Phi_q$ form an orthonormal basis of $\cH_{\bX_n}$. For indices with $\hat\lambda_q=0$, set $\Phi_q=0$ in the sums below. No division by $\hat\lambda_q$ is made at such indices.

We also introduce the definition about convex and symmetric sets below as well as
Lemma~\ref{lemma:TLC-kernel-ind}, which lay the foundation of the proof of Theorem~\ref{theorem:TLC-kernel}.
\begin{definition}[Convex and Symmetric Set]
\label{def:convex-symmetric-space}
A set $X$ is convex if $\alpha X+(1-\alpha)X \subseteq X$ for all $\alpha \in [0,1]$. $X$ is symmetric if $-X \subseteq X$.
\end{definition}
\begin{lemma}\label{lemma:TLC-kernel-ind}
Let $\cF= \cH_{\bX_n}(\mu)$. For every $r > 0$,
\bal\label{eq:TLC-kernel-u-ind}
& \Expect{\bY^{(u)},\bsigma}{\sup_{f\in \cF \colon T_n (f)  \le r} R^{(\textup{ind})}_{\bsigma,\bY^{(u)}}f}\le {\tilde \varphi_u}(r),
\eal%
where
\bal\label{eq:varphi-TLC-kernel-u-ind}
{\tilde \varphi_u}(r) \defeq
\min_{Q \colon 0 \le Q \le n} \pth{  \sqrt{\frac {rQ}u} + \mu
\sqrt{\frac{\sum\limits_{q = Q+1}^{n}\hat \lambda_q}{u}}}.
\eal
Similarly, for every $r > 0$,
\bal\label{eq:TLC-kernel-m-ind}
& \Expect{\bY^{(m)},\bsigma}{\sup_{f\in \cF \colon T_n (f)  \le r} R^{(\textup{ind})}_{\bsigma,\bY^{(m)}}f} \le {\tilde \varphi_m}(r),
\eal%
where
\vspace{-5pt}
\bal\label{eq:varphi-TLC-kernel-m-ind}
{\tilde \varphi_m}(r) \defeq
\min_{Q \colon 0 \le Q \le n } \pth{  \sqrt{\frac {rQ}m} + \mu
\sqrt{\frac{\sum\limits_{q = Q+1}^{n}\hat \lambda_q}{m}}}.
\eal
\end{lemma}
\begin{proof}

We have
\bal\label{eq:lemma-TLC-empirical-NN-seg1}
R^{(\textup{ind})}_{\bsigma,\bY^{(u)}}f = \frac 1u \sum\limits_{i=1}^u {\sigma_i}{f(\bbx_{Y_i})} =
\iprod{f}
{\frac 1u \sum\limits_{i=1}^u {\sigma_i}{K(\cdot,\bbx_{Y_i})}}_{\cH_K}.
\eal

The minimum in (\ref{eq:varphi-TLC-kernel-u-ind}) can be taken over $Q=0$ and indices $Q$ with $\hat\lambda_Q>0$. Indeed, if $\hat\lambda_Q=0$, replacing $Q$ by the number of positive eigenvalues leaves the tail sum unchanged and decreases the first term. Because the $\Phi_q$ with $\hat\lambda_q>0$ form an orthonormal basis of $\cH_{\bX_n}$, for such $Q$ we further express the first term
on the RHS of (\ref{eq:lemma-TLC-empirical-NN-seg1}) as
\bal\label{eq:lemma-TLC-empirical-NN-seg2}
\iprod{f}{\frac 1u \sum\limits_{i=1}^u {\sigma_i}{K(\cdot,\bbx_{Y_i})}}_{\cH_K}
&=\iprod{\sum\limits_{q=1}^{Q} \sqrt{\hat \lambda_q}  \iprod{f}
{\Phi_q}_{\cH_K}\Phi_q }
{v^{(Q)}(\bY^{(u)},\bsigma)}_{\cH_K} +\iprod{\bar f}
{{\bar v}^{(Q)}(\bY^{(u)},\bsigma)}_{\cH_K},
\eal
where
\bals
\bar f &\defeq f - \sum\limits_{q=1}^{Q}   \iprod{f}
{\Phi_q}_{\cH_K}\Phi_q, \\
v^{(Q)}(\bY^{(u)},\bsigma) &\defeq \frac 1u\sum\limits_{q=1}^{Q} \frac{1}{\sqrt{\hat \lambda_q}}\iprod{\sum\limits_{i=1}^u {\sigma_i}{K(\cdot,\bbx_{Y_i})}}
{\Phi_q}_{\cH_K}\Phi_q, \\
{\bar v}^{(Q)}(\bY^{(u)},\bsigma) &\defeq
\frac 1u \sum\limits_{q = Q+1}^{n} \iprod{\sum\limits_{i=1}^u {\sigma_i}{K(\cdot,\bbx_{Y_i})}}{\Phi_q}_{\cH_K}\Phi_q.
\eals
Define the operator $\hat T_n \colon \cH_K \to \cH_K$ by $\hat T_n f = 1/n \cdot \sum_{i=1}^n K(\cdot,\bbx_{i}) f(\bbx_{i})$ for any $f \in \cH_K$.
For each $q\in[n]$ with $\hat\lambda_q>0$, $\Phi_q$ is a normalized eigenfunction of $\hat T_n$ with eigenvalue $\hat\lambda_q$.
We have
\bals
\iprod{\hat T_n f}{f}_{\cH_K} = \iprod{\frac 1n \sum_{i=1}^n K(\cdot,\bbx_{i}) f(\bbx_{i})}{f}_{\cH_K} = T_n({f)}.
\eals
As a result,
\bal\label{eq:lemma-TLC-empirical-NN-seg3}
\norm{\sum\limits_{q=1}^Q \sqrt{\hat \lambda_q}  \iprod{f}
{\Phi_q}\Phi_q }{\cH_{K}}^2
=\sum\limits_{q=1}^Q \hat \lambda_q
\iprod{f}{\Phi_q}_{\cH_K}^2 \le \sum\limits_{q = 1}^n \hat \lambda_q
\iprod{f}{\Phi_q}_{\cH_K}^2 = \iprod{\hat T_n f}{f}_{\cH_K}
=T_n (f) \le r,
\eal
which holds for all $f$ such that $T_n (f) \le r$.

Combining (\ref{eq:lemma-TLC-empirical-NN-seg1})-(\ref{eq:lemma-TLC-empirical-NN-seg3}), we have
\bal\label{eq:lemma-TLC-empirical-NN-seg4}
\Expect{\bY^{(u)},\bsigma}{\sup_{f\in \cF \colon T_n (f)  \le r} 
R^{(\textup{ind})}_{\bsigma,\bY^{(u)}}f} 
&\stackrel{\circled{1}}{\le} \sup_{f\in \cF \colon T_n (f)  \le r}
\norm{\sum\limits_{q=1}^Q \sqrt{\hat \lambda_q}  \iprod{f}
{\Phi_q}_{\cH_K}\Phi_q}{\cH_K}
\cdot
\Expect{\bY^{(u)},\bsigma}{\norm{v^{(Q)}(\bY^{(u)},\bsigma)}{\cH_K}} \nonumber \\
&\phantom{=}+ \sup_{f\in\cF\colon T_n(f)\le r}\norm{\bar f}{\cH_K} \cdot
\Expect{\bY^{(u)},\bsigma}{\norm{{\bar v}^{(Q)}(\bY^{(u)},\bsigma)}{\cH_K}}
\nonumber \\
&\le \sqrt{r} \Expect{\bY^{(u)},\bsigma}{\norm{v^{(Q)}(\bY^{(u)},\bsigma)}{\cH_K}}
+ \mu \Expect{\bY^{(u)},\bsigma}{\norm{{\bar v}^{(Q)}(\bY^{(u)},\bsigma)}{\cH_K}}.
\eal
where $\circled{1}$ is due to the Cauchy-Schwarz inequality.

We have
\bal\label{eq:lemma-TLC-empirical-NN-seg5}
&\frac 1u \Expect{\bY^{(u)},\bsigma}{\iprod{
\sum\limits_{i=1}^u {\sigma_i}{K(\cdot,\bbx_{Y_i})}}{\Phi_q}_{\cH_K}^2} \stackrel{\circled{1}}{=}\frac 1u \Expect{\bY^{(u)}}{\sum\limits_{i=1}^u
\iprod{K(\cdot,\bbx_{Y_i})}{\Phi_q}_{\cH_K}^2}  \nonumber \\
&= \frac 1u \Expect{\bY^{(u)}}{\sum\limits_{i=1}^u
\Phi_q(\bbx_{Y_i})^2}
=\frac 1n \sum\limits_{i=1}^n \Phi^2_q(\bbx_i)  = \iprod{\hat T_n\Phi_q}{\Phi_q} = \hat \lambda_q.
\eal
Here $\circled{1}$ is due to the fact that $\Expect{}{\sigma_i} = 0$ for all $i \in [n]$.
It follows from (\ref{eq:lemma-TLC-empirical-NN-seg5}) that
\bal\label{eq:lemma-TLC-empirical-NN-seg6}
&\Expect{\bY^{(u)},\bsigma}{\norm{v^{(Q)}(\bY^{(u)},\bsigma)}{\cH_K}}
= \frac 1{\sqrt u} \Expect{\bY^{(u)},\bsigma}{\sqrt{\frac 1u \sum\limits_{q=1}^Q \frac{1}{\hat \lambda_q} \iprod{\sum\limits_{i=1}^u {\sigma_i}{K(\cdot,\bbx_{Y_i})}}{\Phi_q}_{\cH_K}^2} } \nonumber \\
&\stackrel{\circled{1}}{\le} \frac 1 {\sqrt u} \sqrt{ \frac 1u
\Expect{\bY^{(u)},\bsigma}{\sum\limits_{q=1}^Q \frac{1}{\hat \lambda_q}\iprod{
\sum\limits_{i=1}^u {\sigma_i}{K(\cdot,\bbx_{Y_i})}}{\Phi_q}_{\cH_K}^2}} \stackrel{\circled{2}}{=}\sqrt{\frac Qu},
\eal
where $\circled{1}$ is due to the Jensen's inequality, $\circled{2}$ follows from (\ref{eq:lemma-TLC-empirical-NN-seg5}). Similarly, we have
\bal\label{eq:lemma-TLC-empirical-NN-seg7}
&\Expect{\bY^{(u)},\bsigma}{\norm{{\bar v}^{(Q)}(\bY^{(u)},\bsigma)}{\cH_K}} = \frac 1{\sqrt u}\Expect{\bY^{(u)},\bsigma}{\sqrt{\frac 1u \sum\limits_{q = Q+1}^{n} \iprod{
\sum\limits_{i=1}^u {\sigma_i}{K(\cdot,\bbx_{Y_i})}}{\Phi_q}_{\cH_K}^2}} \nonumber \\
&\le \frac 1{\sqrt u} \sqrt{ \frac 1u
\Expect{\bY^{(u)},\bsigma}{\sum\limits_{q = Q+1}^{n} \iprod{\sum\limits_{i=1}^u
{\sigma_i}{K(\cdot,\bbx_{Y_i})}}{\Phi_q}_{\cH_K}^2 }}
=\sqrt{\frac{\sum\limits_{q = Q+1}^{n}\hat \lambda_q}{u}}.
\eal
It follows from (\ref{eq:lemma-TLC-empirical-NN-seg4}), (\ref{eq:lemma-TLC-empirical-NN-seg6}),
 and (\ref{eq:lemma-TLC-empirical-NN-seg7}) that
\bals
\Expect{\bY^{(u)},\bsigma}{\sup_{f\in \cF \colon T_n (f)  \le r}\iprod{f}
{\frac 1u \sum\limits_{i=1}^u {\sigma_i}{K(\cdot,\bbx_{Y_i})}} }
\le \min_{Q \colon 0 \le Q \le n} \pth{\sqrt{\frac {rQ}u} +
\mu
\sqrt{\frac{\sum\limits_{q = Q+1}^{n}\hat \lambda_q}{u}}},
\eals
which completes the proof of (\ref{eq:TLC-kernel-u-ind}). (\ref{eq:TLC-kernel-m-ind}) can be proved by a similar argument.
\end{proof}

\begin{proof}
[\textbf{\textup{Proof of Theorem~\ref{theorem:TLC-kernel}}}]
If $\bK=0$, then $\cF=\cH_{\bX_n}(\mu)=\set{0}$, so the excess risk is zero and the conclusion is immediate. Hence suppose that $\bK$ has at least one positive eigenvalue. It follows from Assumption~\ref{assumption:Lipschitz-loss-Tn-f-Ln-ellf} that for all $h \in \Delta^*_{\cF}$, $T_n(h) \le B'L^2 \cL_n(h)$. To see this, let $h = \ell_{f_1}-\ell_{f^*_n}$ with $f_1 \in \cF$. Then
$T_n(h) = T_n(\ell_{f_1}-\ell_{f^*_n}) \le L^2 T_n (f_1-f^*_n) \le B'L^2 \cL_n(\ell_{f_1}-\ell_{f^*_n}) = B'L^2 \cL_n(h)$. This inequality indicates that  Assumption~\ref{assumption:main} (2) holds with $B =  B'L^2$. As a result,  Assumption~\ref{assumption:main} holds.

We now apply Theorem~\ref{theorem:TLC}, Lemma~\ref{lemma:TLC-delta-ell-f}, and Corollary~\ref{corollary:TLC-delta-ell-f-ind-rc}
with the function class $\cF = \cH_{\bX_n}(\mu)$ and $\tT_n(\cdot)$ defined in (\ref{eq:tTn-def})
to prove
that for every $x > 0$ and every $\delta \in (0,1)$, with probability at least $1-\exp(-x)-\delta$ over $\set{\bZ}$,
\bal\label{eq:TLC-kernel-general-loss}
\cL_u(h) &\le \cL_m(h) + \frac {2L^2B'\pth{\cL_n(\ell_{f_1} - \ell_{f^*_n}) + \cL_n(\ell_{f_2} - \ell_{f^*_n})}}{K_0}  +c'_3\min_{0 \le Q \le n} r(u,m,Q) + \frac{c_1x}{N_{u,m,\delta}},
\quad \forall f_1,f_2\in\cF,\quad h=\ell_{f_1}-\ell_{f_2},
\eal
where $c'_3$ is an absolute positive number depending on $c_0,L_0,L,\mu$, and
\bals
r(u,m,Q) = Q \pth{\frac{1}{u} + \frac{1}{m}} + \sqrt{\frac{\sum\limits_{q = Q+1}^{n}\hat \lambda_q}{u}}
+ \sqrt{\frac{\sum\limits_{q = Q+1}^{n}\hat \lambda_q}{m}}.
\eals

Let $r>0$ and $h = \ell_{f_1}-\ell_{f_2} \in \Delta_{\cF}$ with $\tT_n(h) \le r$. By the definition of $\tT_n$, for each such $h$ there exist $f_1,f_2 \in \cF$ such that $h = \ell_{f_1}-\ell_{f_2}$ and $2B \cL_n(\ell_{f_1} - \ell_{f^*_n}) + 2B \cL_n(\ell_{f_2} - \ell_{f^*_n}) \le r'$ for any $r' > r$. Set $r' = 1.1r$. At $r=0$, Assumption~\ref{assumption:main} implies that every $h$ in this localized class vanishes on the full sample, so the complexity bounds below hold trivially.
Let $\bsigma = \set{\sigma_i}_{i=1}^{\max\set{u,m}}$ be i.i.d. Rademacher variables. We have
\bal\label{eq:TLC-kernel-seg1}
&2\Expect{\bY^{(u)},\bsigma}{\sup_{h \colon h \in \Delta_{\cF},\tT_n(h) \le r} R^{(\textup{ind})}_{\bsigma,\bY^{(u)}}h } \stackrel{\circled{1}}{\le} 2\Expect{\bY^{(u)},\bsigma}{\sup_{f_1,f_2 \in \cF \colon 2B \cL_n(\ell_{f_1} - \ell_{f^*_n}) + 2B \cL_n(\ell_{f_2} - \ell_{f^*_n}) \le r'} R^{(\textup{ind})}_{\bsigma,\bY^{(u)}}(\ell_{f_1}-\ell_{f_2})} \nonumber \\
&\le 2\Expect{\bY^{(u)},\bsigma}{\sup_{f_1\in \cF \colon  \cL_n(\ell_{f_1} - \ell_{f^*_n}) \le 1.1r/2B} R^{(\textup{ind})}_{\bsigma,\bY^{(u)}} \pth{\ell_{f_1}- \ell_{f^*_n}}} \nonumber \\
&\quad\quad+ 2\Expect{\bY^{(u)},\bsigma }{\sup_{f_2 \in \cF \colon  \cL_n(\ell_{f_2} - \ell_{f^*_n}) \le 1.1r/2B} R^{(\textup{ind})}_{\bsigma,\bY^{(u)}}\pth{\ell_{f_2}- \ell_{f^*_n}}} \nonumber \\
&\stackrel{\circled{2}}{\le} 4L\Expect{\bY^{(u)},\bsigma}{\sup_{f\in \cF \colon
T_n\pth{f - f^*_n} \le rB_1/2B} R^{(\textup{ind})}_{\bsigma,\bY^{(u)}} \pth{f- f^*_n}}  \nonumber \\
&\stackrel{\circled{3}}{\le} 8L\Expect{\bY^{(u)},\bsigma}{\sup_{f\in \cF \colon T_n (f)  \le rB_1/8B} R^{(\textup{ind})}_{\bsigma,\bY^{(u)}}f} \stackrel{\circled{4}}{\le}8L {\tilde \varphi_u}\pth{\frac {rB_1}{8B}} = 8L {\tilde \varphi_u}\pth{\frac {1.1r}{8L^2}}.
\eal
Here $\circled{1}$ is due to the definition of $\tT_n$. The sign of the second Rademacher sum is reversed by symmetry of the Rademacher variables. $\circled{2}$ is due to the contraction property in Theorem~\ref{theorem:contraction-RC} and the fact that the loss function $\ell(\cdot,\cdot)$ is $L$-Lipschitz continuous, and $B_1$ is a positive constant such that $B_1 = 1.1B'$.
$\circled{3}$ follows from the fact that $(f-f^*_n)/2 \in \cF$ because $\cF$ is symmetric and convex. ${\tilde \varphi_u}$ in $\circled{4}$ is defined in (\ref{eq:varphi-TLC-kernel-u-ind}) in Lemma~\ref{lemma:TLC-kernel-ind}.

It follows from (\ref{eq:TLC-kernel-seg1}) that
\bals
2\Expect{\bY^{(u)},\bsigma }{\sup_{h \colon h \in \Delta_{\cF},\tT_n(h) \le r} R^{(\textup{ind})}_{\bsigma,\bY^{(u)}}h }\le  8L {\tilde \varphi_u}\pth{\frac {rB_1}{8B}}.
\eals
For the squared class, the map
\((a,b)\mapsto(a-b)^2\) is \(2\sqrt{2}L_0\)-Lipschitz
on \([0,L_0]^2\) in the Euclidean norm. Apply the vector
contraction inequality \citep[Theorem 3]{Maurer2016-vector-rademacher}
to the pairs \((\ell_{f_1},\ell_{f_2})\), first on a countable dense subclass and then by continuity, and enlarge their
joint localization to the two individual constraints
\(\cL_n(\ell_{f_j}-\ell_{f^*_n})\le 1.1r/(2B)\).
The two independent Rademacher sums on the right each have the
same bound. Scalar contraction and the fact that
\((f-f^*_n)/2\in\cF\) yield
\bal\label{eq:TLC-kernel-seg1-post}
2\Expect{\bY^{(u)},\bsigma}{\sup_{h\in\Delta_{\cF},\tT_n(h)\le r}
R^{(\textup{ind})}_{\bsigma,\bY^{(u)}}h^2}
&\le 16L_0\Expect{\bY^{(u)},\bsigma}{
\sup_{f\in\cF\colon\cL_n(\ell_f-\ell_{f^*_n})\le 1.1r/(2B)}
R^{(\textup{ind})}_{\bsigma,\bY^{(u)}}(\ell_f-\ell_{f^*_n})}
\nonumber\\
&\le 32L_0L\tilde\varphi_u\pth{\frac{1.1r}{8L^2}}.
\eal
Here the factor \(16L_0\) includes the factor two in the
left-hand side, the \(\sqrt2\) in vector contraction, and the
two coordinate processes. This argument applies to signed
Rademacher sums and does not use a pointwise ordering of the
squared losses.

Define $\varphi_u(r) \defeq \max\set{8L {\tilde \varphi_u}\pth{\frac {1.1r}{8L^2}},32L_0 L {\tilde \varphi_u}\pth{\frac {1.1r}{8L^2}}} =  L' {\tilde \varphi_u}\pth{\frac {1.1r}{8L^2}}$ with $L' \defeq \max\set{8L,32L_0 L}$. It can be verified that $\varphi_u$ is a sub-root function by checking the definition of the sub-root function.

Similarly, we have
\bals
2\Expect{\bY^{(m)},\bsigma }{\sup_{h \colon h \in \Delta_{\cF},\tT_n(h) \le r} R^{(\textup{ind})}_{\bsigma,\bY^{(m)}}h } &\le 8L {\tilde \varphi_m}\pth{\frac {rB_1}{8B}}
= 8L {\tilde \varphi_m} \pth{\frac {1.1r}{8L^2}}, \\
2\Expect{\bY^{(m)},\bsigma }{\sup_{h \colon h \in \Delta_{\cF},\tT_n(h) \le r} R^{(\textup{ind})}_{\bsigma,\bY^{(m)}} h^2 }
&\le 32L_0  L{\tilde \varphi_m}\pth{\frac {rB_1}{8B}} = 32L_0 L{\tilde \varphi_m} \pth{\frac {1.1r}{8L^2}},
\eals
and $\varphi_m(r) \defeq \max\set{8L {\tilde \varphi_m} \pth{\frac {1.1r}{8L^2}}, 32L_0 L{\tilde \varphi_m} \pth{\frac {1.1r}{8L^2}}} = L'{\tilde \varphi_m}\pth{\frac {1.1r}{8L^2}}$ is also a sub-root function. Let $r_u,r_m$ be the fixed points of  $\varphi_u$ and $\varphi_m$ respectively. We define $\varphi_{u,m}(r) \defeq \varphi_u(r) + \varphi_m(r)$.
Then $\varphi_{u,m}$ satisfies condition (\ref{eq:TLC-cond-psi-general}) in Theorem~\ref{theorem:TLC} according to Corollary~\ref{corollary:TLC-delta-ell-f-ind-rc}.
$\varphi_{u,m}$ is also a sub-root function. Let $r_{u,m}$ be the fixed point
of $\varphi_{u,m}$. Since both $\varphi_u$ and $\varphi_m$ are sub-root
functions, we have $r_{u,m} \ge \max\set{r_u,r_m}$, and
\bals
r_{u,m} = \varphi_{u,m}(r_{u,m})
= \varphi_u(r_{u,m}) + \varphi_m(r_{u,m}) \ge \varphi_u(r_u) +  \varphi_m(r_m)
= r_u + r_m.
\eals
It then follows from Theorem~\ref{theorem:TLC} and Lemma~\ref{lemma:TLC-delta-ell-f} that, for all $h \in \Delta_{\cF}$, we have
\bal\label{eq:TLC-kernel-seg2-pre}
\cL_u(h) \le \cL_h(\overline{\set{\bZ_{\bd}}}) + \frac {2B}{K_0} \pth{\cL_n(\ell_{f_1} - \ell_{f^*_n}) + \cL_n(\ell_{f_2} - \ell_{f^*_n})} +c_0r_{u,m} + \frac{c_1x}{N_{u,m,\delta}}.
\eal
Let $0 \le r \le r_{u,m}$. Then it follows from \citep[Lemma 3.2]{bartlett2005} that
$0 \le r \le \varphi_{u,m}(r)$. Therefore, by the definition of ${\tilde \varphi_u}$ in (\ref{eq:varphi-TLC-kernel-u-ind}) and ${\tilde \varphi_m}$ in (\ref{eq:varphi-TLC-kernel-m-ind}),
for every $0 \le Q \le n$ we have
\bals
\frac{r}{ L' } \le \sqrt{\frac {1.1r Q}{8L^2u}} +
\sqrt{\frac {1.1rQ}{8L^2m}} +
\mu \sqrt{\frac{\sum\limits_{q = Q+1}^{n}\hat \lambda_q}{u}}
+
\mu \sqrt{\frac{\sum\limits_{q = Q+1}^{n}\hat \lambda_q}{m}}.
\eals
Solving the above quadratic inequality for $r$, we have
\bal\label{eq:TLC-kernel-seg2}
r\le \hat c_3 Q \pth{\frac{1}{u} + \frac{1}{m}} + \hat c_3
\pth{
 \sqrt{\frac{\sum\limits_{q = Q+1}^{n}\hat \lambda_q}{u}}
+
 \sqrt{\frac{\sum\limits_{q = Q+1}^{n}\hat \lambda_q}{m}}}
 = \hat c_3 r(u,m,Q),
\eal
where $\hat c_3$ is a positive constant depending on $L_0,L,\mu$.
(\ref{eq:TLC-kernel-seg2}) holds for every $0 \le Q \le n$,
so it follows from
(\ref{eq:TLC-kernel-seg2-pre})
and (\ref{eq:TLC-kernel-seg2})
that (\ref{eq:TLC-kernel-general-loss}) holds with $c'_3 = c_0\hat c_3$.

We then apply
Theorem~\ref{theorem:TLC-delta-star-ell-f} and (\ref{eq:TLC-kernel-general-loss}) to prove (\ref{eq:TLC-kernel-excess-loss}). With $h = \ell_{\hat f_{m}} - \ell_{\hat f_{u}}$, then we can set $f_1 = \hat f_{m}$  and $f_2 = \hat f_{u}$ in (\ref{eq:TLC-kernel-general-loss}).
We now find the sub-root function $\psi^*_{u,m}$ satisfying (\ref{eq:TLC-cond-um-delta-star-ell-f-psi-star})  in Theorem~\ref{theorem:TLC-delta-star-ell-f}. Applying Theorem~\ref{theorem:TC-RC}, we need to find the sub-root function $\psi^*_{u,m}$ such that
\bals
\psi^*_{u,m} (r) \ge  2\max\Bigg\{&\Expect{}{\sup_{h \colon h \in\Delta^*_{\cF},B \cL_n(h) \le r} R^{(\textup{ind})}_{\bsigma,\bY^{(u)}}h}, \Expect{}{\sup_{h \colon h \in\Delta^*_{\cF},B \cL_n(h) \le r} R^{(\textup{ind})}_{\bsigma,\bY^{(m)}}h}, \nonumber \\
&\Expect{}{\sup_{h \colon h \in\Delta^*_{\cF},B \cL_n(h) \le r}
R^{(\textup{ind})}_{\bsigma,\bY^{(\min\set{u,m})}}h^2}\Bigg\},
\eals
For $h=\ell_f-\ell_{f^*_n}\in\Delta^*_{\cF}$ with $B\cL_n(h)\le r$, the same function belongs to $\Delta_{\cF}$ with $\tT_n(h)\le 2r$, by taking $f_1=f$ and $f_2=f^*_n$. Thus (\ref{eq:TLC-kernel-seg1}) and (\ref{eq:TLC-kernel-seg1-post}) show that the explicit sub-root function $\psi^*_{u,m}(r)\defeq\varphi_{u,m}(2r)$ satisfies (\ref{eq:TLC-cond-um-delta-star-ell-f-psi-star}). Let $r^*$ be its fixed point. For $0\le r\le r^*$, the inequality $r\le\varphi_{u,m}(2r)$ and the same quadratic calculation as in (\ref{eq:TLC-kernel-seg2}) give $r\le\hat c_3\min_{0\le Q\le n}r(u,m,Q)$ after enlarging $\hat c_3$ if necessary. As a result, (\ref{eq:TLC-bound-delta-star-ell-f}) in Theorem~\ref{theorem:TLC-delta-star-ell-f} implies that
with probability at least
$1-2\exp(-x)-2\delta$,
\bal
\cL_n(\ell_{\hat f_{u}} - \ell_{f^*_n}) &\le  c_2 \pth{\hat c_3\min_{0 \le Q \le n} r(u,m,Q) + \frac{x}{N_{u,m,\delta}}}, \label{eq:TLC-kernel-seg4} \\
\cL_n(\ell_{\hat f_{m}} - \ell_{f^*_n}) &\le  c_2 \pth{\hat c_3\min_{0 \le Q \le n} r(u,m,Q)+ \frac{x}{N_{u,m,\delta}}}. \label{eq:TLC-kernel-seg5}
\eal
We note that $\cL_{ \ell_{\hat f_{m}} - \ell_{\hat f_{u}}}(\overline{\set{\bZ_{\bd}}}) \le 0$ due to the optimality of $\hat f_{m}$. Applying the upper bounds in (\ref{eq:TLC-kernel-seg4}) and (\ref{eq:TLC-kernel-seg5}) to (\ref{eq:TLC-kernel-general-loss}) proves (\ref{eq:TLC-kernel-excess-loss}).

\end{proof}

\section{Discussions about Related Works for Theorem~\ref{theorem:TLC-delta-ell-f-excess-risk-upper-bound-VC-dim} }
\label{sec:discussion-transductive-learning-finite-VC-dim}

\subsection{Inductive risk and risk under a uniform random split}
\label{sec:gap-transductive-inductive}
The risk, or generalization error, in inductive learning is the expectation of a loss function over an unknown data distribution, whereas the risk in transductive learning is the averaged loss on the finite test data.
This distinction requires care when transferring a guarantee between the two settings, but does not preclude an optimal bound for an appropriate learner under a uniform random split.
Theorem~\ref{theorem:TLC-delta-ell-f-excess-risk-upper-bound-VC-dim} concerns a single static empirical risk minimizer and improves the logarithmic complexity term in the known ERM bound~\citep[Theorem 7]{tolstikhin2016minimaxlowerboundsrealizable} from $\log n$ to $\log m$ through TLC. The finite population martingale argument discussed below, and the separate construction in Theorem~\ref{theorem:optimal-transductive-transfer}, give bounds without this factor for appropriate randomized learners.

We first distinguish the two risks and then discuss how the martingale argument can be adapted to a finite population. Using the notations defined in Section~\ref{sec:transductive-basic-setup}, suppose that $P$ is an unknown distribution over the input $\bx \in \cX$ with the corresponding label $y \in \cY$. The inductive learning risk $R^{\textup{(ind)}}$ and the transductive learning risk $\cR^{\textup{(trans)}}$, given a loss function $\ell \colon \cY \times \cY \to \RR$ with a predictor $f\colon \cX \to \cY$, are expressed as
\bals
\cR^{\textup{(ind)}}(f) = \Expect{(\bx,y)\sim P}{\ell(f(\bx),y)},
\quad
\cR^{\textup{(trans)}}(f) = \frac{1}{u} \sum\limits_{\bx \in \bX_u} \ell(f(\bx),y(\bx)),
\eals
where $\bX_u$ is the test set of size $\abth{\bX_u} = u$,  and $y(\bx)$ is the label of $\bx \in \bX_u$.
For the zero one loss in the TLSII setting of \citep{tolstikhin2016minimaxlowerboundsrealizable}, suppose that the test observations are i.i.d. from $P$ and independent of the training data, and that $f$ depends only on the training data and independent learner randomization. Conditional on the training data and this randomization, Bernstein's inequality gives, for every $x>0$, with probability at least $1-\exp(-x)$,
\bals
\cR^{\textup{(trans)}}(f)
&\le \cR^{\textup{(ind)}}(f)+\sqrt{\frac{2\cR^{\textup{(ind)}}(f)x}{u}}+\frac{2x}{3u}
\le 2\cR^{\textup{(ind)}}(f)+\frac{2x}{u}.
\eals
Consequently, an inductive bound of order $(\dVC+x)/m$ with failure probability at most $\exp(-x)$ yields the same transductive order when $u\ge m$, with failure probability at most $2\exp(-x)$, as in \citep[Theorem 9]{tolstikhin2016minimaxlowerboundsrealizable}. This argument uses independence of the test sample and does not directly apply to a fixed full sample with random split in the standard transductive learning setup considered in this paper. Theorem~\ref{theorem:optimal-transductive-transfer} supplies an exact sampling reduction for that setting.

The cited PAC guarantees in~\citep{Haussler1988-vc-dim-binary-class,Hanneke16-optimal-sample-PAC,Aden-AliCSZ23-optimal-PAC-without-uniform-convergence} concern inductive risk $\cR^{\textup{(ind)}}$. In particular, the classification strategy in \citep{Haussler1994-vc-dim-binary-class} achieves a nearly optimal risk bound of order $\dVC \log(1/\delta)/m$, and the majority voting classifier in \citep{Hanneke16-optimal-sample-PAC} achieves the optimal order $\cO(\dVC/m)$ at fixed confidence.

\noindent \textbf{Finite population adaptation of the martingale argument.} The PAC theorem in \citep{Aden-AliCSZ23-optimal-PAC-without-uniform-convergence} is stated for \emph{inductive} risk. In particular, the average risk bound in
\citep[Theorem 2.1]{Aden-AliCSZ23-optimal-PAC-without-uniform-convergence}, as the first main result of \citep{Aden-AliCSZ23-optimal-PAC-without-uniform-convergence}, concerns the inductive risk of a predictor trained on prefixes of the training data $\set{S_{\le t}}$ with $S_{\le t} = \set{x_i,y_i}_{i=1}^t$, where the original training data is $S = \set{x_i,y_i}_{i=1}^n$. Let $\cR^{\textup{(ind)}}(f,S)$ be the risk of $f$ trained on the data $S$. Suppose that $P$ is realizable, the loss takes values in $[0,1]$, the predictor is invariant under permutations of its training data, and its LOO error is at most $M_t$ on every realizable sample of size $t$, where $M_t$ is nondecreasing in $t$. If $S$ is an i.i.d. sample from $P$ and $n/4$ is an integer, then, for every $\delta\in(0,1)$, \citep[Theorem 2.1]{Aden-AliCSZ23-optimal-PAC-without-uniform-convergence} gives, with probability at least $1-\delta$ over $S$,
\bals
\frac{4}{3n}\sum\limits_{t = n/4}^{n-1} \cR^{\textup{(ind)}}(f,S_{\le t})
\le 4.82\pth{ \frac{M_n}{n} + \frac{1}{n}\log\pth{\frac{2}{\delta}} },
\eals
where the LOO error is defined in (\ref{eq:LOO-error}). For the one inclusion predictor under the BVC setup, one can take $M_n=\dVC$. The proof combines the forward martingale bound in \citep[Lemma 4.2]{Aden-AliCSZ23-optimal-PAC-without-uniform-convergence} with a reverse martingale bound. In the forward argument, $\cR^{\textup{(ind)}}(f,S_{\le t})$ is the conditional expected loss on the next observation given the training prefix. Subtracting this conditional mean from the next prediction loss gives a martingale difference. Under sampling without replacement, the corresponding conditional mean is the average loss over the remaining population. Subtracting this mean from the next prediction loss still gives a forward martingale difference. For a permutation invariant one inclusion predictor under the BVC setup, the reverse martingale argument also applies to a uniform ordering of the fixed population, because the LOO bound controls the conditional error when the last point of a prefix is exposed. For $u\ge m\ge\dVC\ge2$, the loss on the test set is at most twice the average loss over the remaining population at each training prefix. Combining the two martingale bounds and taking a majority vote over prefixes of lengths $\lceil m/2\rceil,\ldots,m-1$ therefore gives a bound of order $(\dVC+\log(1/\delta))/m$, with probability at least $1-\delta$ over the split and the random ordering of the training sample. This construction analyzes an aggregated learner and does not require a uniform bound for $g(\cH)$. Theorem~\ref{theorem:optimal-transductive-transfer} gives a separate proof of the same existence conclusion through an exact sampling reduction. The remaining logarithmic factor in Theorem~\ref{theorem:TLC-delta-ell-f-excess-risk-upper-bound-VC-dim} pertains to its TLC analysis of a single static empirical risk minimizer.

\subsection{The Classical Transductive LOO Guarantee}
\label{sec:Haussler1988-transductive-LOO}
The classical transductive guarantee in \citep{Haussler1994-vc-dim-binary-class} concerns a LOO setup. In particular, in the transductive LOO setup, there is a fixed realizable sample $S=\set{(x_i,y_i)}_{i=1}^n$, and the learner is given access to the labels of $n-1$ randomly chosen points and is tasked with predicting the label of the remaining point. With $S^{-i}=((x_j,y_j))_{j\ne i}$, the goal is to design a predictor $\hat f$ which minimizes the following LOO error,
\bal\label{eq:LOO-error}
\sum_{i=1}^n \ell\bigl(\hat f(x_i;S^{-i}),y_i\bigr),
\eal
where $\ell(y,\hat y) \defeq \indict{y \neq \hat y}$, and $\hat f(\cdot ; S^{-i})$ is the predictor using the class labels of $S^{-i}$ and the unlabeled sample $\set{x_i}_{i=1}^n$.
For binary classification where $y_i \in \set{0,1}$, \citep{Haussler1994-vc-dim-binary-class}
proves that a standard one-inclusion hypergraph strategy achieves LOO error at most $\dVC$ under the BVC setup. Then an exchangeability argument yields a predictor
$\hat f_{\textup{OIG}}$ with the following  \emph{inductive generalization bound}:
\bals
\Expect{(X,Y)\sim P \atop S\sim P^n}{ \ell\bigl(\hat f_{\mathrm{OIG}}(X;S),Y\bigr) } \le \frac{\dVC}{n}.
\eals
The transductive bound $\sum_{i=1}^n \ell\bigl(\hat f(x_i;S^{-i}),y_i\bigr) \le \dVC$ directly concerns the case where the test data size is $u=1$ and the training data size is $m=n-1$. In fact, the LOO error (\ref{eq:LOO-error}) is the scaled expectation of the test loss of the predictor $\hat f$, that is,
\bal\label{eq:LOO-error-Lu}
\sum_{i=1}^n \ell\bigl(\hat f(x_i;S^{-i}),y_i\bigr) = n \cdot \Expect{\set{\bZ}}{\cL_u \bigl(\ell_{\hat f} \bigr)},
\eal
where $\cL_u = \cL_{\ell_{\hat f}}(\set{\bZ})$ is the test loss associated with the random set $\bZ$ of size $u = 1$ sampled uniformly without replacement from $[n]$, defined in Section~\ref{sec:transductive-basic-setup} of this paper. This identity concerns the expected test loss for $u=1$. The LOO bound also supplies the reverse martingale step in the prefix aggregation argument above, which gives a probability bound for general splits with $u\ge m$. 
In comparison, this paper studies the general transductive setting allowing for all $u \in [2\relcolon n-2]$ and establishes a high-probability bound for the supremum of the test-train process $g(\cH) = \sup_{h \in \cH} \pth{\cL_u(h)-\cL_m(h)}$ defined in~(\ref{eq:def-g}) over a bounded class of loss functions. Such uniform control supports the TLC analysis of a static empirical minimizer, but is not needed for the prefix aggregation guarantee.


\subsection{One Inclusion Hypergraph and LOO Control}
\label{sec:one-include-graph-LOO}
The one inclusion hypergraph provides the LOO control used in the prefix aggregation argument. Indeed, fix a sample
$U=\set{x_i}_{i=1}^n$ and a hypothesis class $\cH$, and recall that the vertices of the one-inclusion hypergraph are the projected hypotheses
$ \cH|_U = \set{(f(x_1),\dots,f(x_n)) : f\in\cH}$. For every coordinate $i\in[n]$ and every restriction $f$ of a function in $\cH$ to $U \setminus \set{x_i}$, the one-inclusion hypergraph creates the hyperedge
\bals
e_{i,f}
=
\set{
f'\in\cH|_U :
f'(j)=f(j)\ \forall j\neq i
}.
\eals
As a result, all hypotheses in the same hyperedge agree on all coordinates except possibly the coordinate $i$.

This construction is precisely designed for the leave-one-out prediction problem. In LOO prediction, when predicting the label of the point $x_i$, the learner observes the labels of all other points
$\set{x_j : j\neq i}$, while the label of $x_i$ is hidden. Consequently, two hypotheses are indistinguishable from the observed data if and only if they agree on all coordinates other than $i$. The only remaining ambiguity is their prediction on the held-out point $x_i$.
For each coordinate $i$, the hyperedge $e_{i,f}$ represents exactly the collection of hypotheses that are all consistent with the observed \(n-1\) labels but may disagree on the hidden coordinate $i$. In this sense, the one-inclusion hypergraph encodes all possible ambiguities arising in leave-one-out prediction.

The orientation rule is then designed to resolve these ambiguities. An orientation assigns to every hyperedge $e$ a  vertex
$\sigma_G(e)\in e$. Suppose $f^*\in\cH|_U$ is the true hypothesis generating the labels of $U$, and a one-inclusion hypergraph $G$ is constructed over $U$ and $\cH$. When predicting the label of $x_i$ for every $i \in [n]$, the algorithm uses the orientation of the hyperedge $e_{i,f^*}$, that is, $\sigma_G(e_{i,f^*})$, to choose one of the hypotheses belonging to $e_{i,f^*}$ and use the chosen hypothesis to predict the label of $x_i$. Such a predictor is denoted as $\hat f_{\textup{OIG}}$.
A mistake occurs on coordinate $i$  when the orientation does not select $f^*$ inside the corresponding hyperedge, that is, $\sigma_G(e_{i,f^*}) \neq f^*$. Consequently, the total number of LOO mistakes equals the out-degree of $f^*$ under the chosen orientation:
\bal\label{eq:LOO-one-inclusion-transductive}
\sum_{i=1}^n
\indict{
\hat f_{\textup{OIG}}(x_i;S^{-i})\neq y_i
} \defeq
\textup{out}(f^*;\sigma_G).
\eal
Therefore, the one-inclusion hypergraph is specifically designed so that leave-one-out prediction errors correspond exactly to out-degree events in an appropriate orientation of the hypergraph. The special structure of the one-inclusion hypergraph then allows one to construct orientations with controlled maximum out-degree, which directly yields upper bounds on the  LOO error. In particular, there exists an orientation with maximum outdegree at most $\lceil\mu_{\cH}(n)\rceil$ \citep[Lemma A.2]{Aden-AliCSZ23-optimal-PAC-without-uniform-convergence}, which bounds the LOO error of the corresponding one inclusion predictor. Under the BVC setup, \citep{Haussler1994-vc-dim-binary-class} shows that $\mu_{\cH}(n)\le\dVC$. Since $\dVC$ is an integer, $\lceil\mu_{\cH}(n)\rceil\le\dVC$, so the transductive LOO error (\ref{eq:LOO-one-inclusion-transductive}) is bounded by $\dVC$.

\section{Optimal Realizable Learning}
\label{sec:optimal-transductive-learning}
\begingroup
\setlength{\parskip}{3pt}
This section separates the existence of a statistically optimal learner from
a bound that holds for every empirical risk minimizer.  The distinction is
relevant to the residual factor $\log(me/\dVC)$ in
Theorem~\ref{theorem:TLC-delta-ell-f-excess-risk-upper-bound-VC-dim}.
We first give an exact sampling reduction and a log-free guarantee for a
possibly improper randomized learner.
This result uses the fixed indexed population and the
uniform random split of Section~\ref{sec:transductive-basic-setup}.
A learner is called proper if its output belongs to the given class, and
improper otherwise.  No computational-efficiency assertion is made.

\subsection{An exact sampling transfer and a log-free learner}
\label{sec:optimal-transductive-transfer}

\begin{theorem}[Exact sampling transfer and optimal realizable learning]
\label{theorem:optimal-transductive-transfer}
Let $\cF$ be a binary-valued class with VC dimension $2\le\dVC<\infty$.
Fix a population $((\bx_i,y_i))_{i=1}^n$, where $n=m+u$,
$u\ge m\ge\dVC$, and $y_i=f^*(\bx_i)$ for some $f^*\in\cF$.
The learner observes $\bX_n$ and the labeled training collection $\bS_m$,
whose index set $\bar Z$ is uniform among the size-$m$ subsets of $[n]$.
For every binary predictor $f$, put $\ell_f(i)\defeq
\indict{f(\bx_i)\ne y_i}$.
There is an auxiliary random vector $\bU$, independent of the split, and
a construction of indices $I_1,\ldots,I_m\in\bar Z$ such that
$(I_1,\ldots,I_m)$ has the same unconditional distribution as $m$
independent uniform indices in $[n]$.  The labeled input to a deterministic
realizable PAC learner $A_m$ and its output are
\bal\label{eq:optimal-transductive-labeled-input}
\widetilde{\bS}_m\defeq
\bigl((\bx_{I_t},y_{I_t})\bigr)_{t=1}^m,
\qquad
\widetilde f_m\defeq A_m(\widetilde{\bS}_m).
\eal
The rule $A_m$ may depend on $\cF$, $\bX_n$, and $m$, but not on
unknown labels or the confidence parameter.  The subscript $Z,\bU$
specifies that probability and expectation include both the split and
the auxiliary randomization.  There is an absolute constant
$C_{\mathrm{tr}}>0$ such that, for every $\delta\in(0,1)$,
\bal\label{eq:optimal-transductive-rate}
\PP_{Z,\bU}\set{\cL_u(\ell_{\widetilde f_m})
> C_{\mathrm{tr}}\frac{\dVC+\log(1/\delta)}{m}}\le\delta,
\qquad
\Expect{Z,\bU}{\cL_u(\ell_{\widetilde f_m})}
\le C_{\mathrm{tr}}\frac{\dVC+1}{m}.
\eal
The learner uses no test labels and need not be proper.
\end{theorem}
The construction in (\ref{eq:optimal-transductive-labeled-input}) is not
ordinary bootstrap sampling from $\bS_m$.  It may repeat observed pairs
and leave some training pairs unused.  Its independent-sampling law is
unconditional over the split and the auxiliary randomization, not
conditional on the observed training set.

\begin{proof}
\noindent\textup {Outline.}  Restrict the class to the finite feature
population, construct the labeled sample using a random ordering of the
training indices, prove its independent-sampling law by counting every
possible index sequence, apply an optimal realizable PAC theorem, and
convert the full-population error into test error before integrating the
tail.

\noindent\textit{Step 1: finite restriction and measurable choices.}
Let $\cX_n\defeq\set{\bx_i:i\in[n]}$ and
$\cF_n\defeq\set{f|_{\cX_n}:f\in\cF}$.
The class $\cF_n$ has at most $2^{\abth{\cX_n}}$ elements and VC dimension
at most $\dVC$.  Fix an ordering of $\cX_n$ and a lexicographic ordering
of the corresponding label vectors.  These orderings make all consistent
base-classifier selections deterministic.  An output defined on
$\cX_n$ can be extended by zero outside $\cX_n$ without affecting any risk
considered here.  We therefore apply the PAC result on the finite
measurable space $\cX_n$.  Every event in this proof is measurable.
Realizability also ensures that repeated occurrences of a feature have
the same label.

\noindent\textit{Step 2: the labeled data supplied to $A_m$.}
Write the elements of $\bar Z$ in increasing order as $s_1,\ldots,s_m$.
Let $\bpi=(\pi_1,\ldots,\pi_m)$ be a uniform permutation of $[m]$,
independent of $Z$, and put $R_j\defeq s_{\pi_j}$.
Independently of $Z$ and $\bpi$, let
$\bJ=(J_1,\ldots,J_m)$ have independent coordinates uniform on $[n]$.
We take $\bU\defeq(\bpi,\bJ)$.
For integers $0\le b\le a$, the falling factorial is
$(a)_b\defeq\prod_{j=0}^{b-1}(a-j)$, with $(a)_0=1$.
The ordered vector $\bR=(R_1,\ldots,R_m)$ is uniform among all ordered
size-$m$ samples of distinct indices from $[n]$, because the probability
of each such vector is $1/(\binom nm m!)=1/(n)_m$.

Initialize $k_0=0$.  At step $t\in[m]$, the indices exposed so far are
$R_1,\ldots,R_{k_{t-1}}$.
If $J_t\le k_{t-1}$, put $I_t=R_{J_t}$ and $k_t=k_{t-1}$.
Otherwise, put $I_t=R_{k_{t-1}+1}$ and $k_t=k_{t-1}+1$.
Thus each exposed index is selected with probability $1/n$, and the next
unexposed training index is selected with probability $(n-k_{t-1})/n$.
Before step $t\le m$, we have $k_{t-1}\le t-1\le m-1$, so the required
unexposed index always exists.  Every $I_t$ belongs to $\bar Z$.
Consequently, all $m$ labeled entries in
(\ref{eq:optimal-transductive-labeled-input}) are available from
$\bS_m$, even when entries are repeated.  No knowledge of $f^*$ or of
a label indexed by $Z$ is required to run the construction.

Fix $\bi=(i_1,\ldots,i_m)\in[n]^m$, and list its distinct indices in
their order of first appearance as $a_1,\ldots,a_k$.
The probability that $R_j=a_j$ for every $j\in[k]$ is $1/(n)_k$.
Conditional on this prefix, producing $\bi$ requires probability $1/n$
at each repeated occurrence.  The $k$ successive first appearances
require probabilities $n/n,(n-1)/n,\ldots,(n-k+1)/n$.
The independent coordinates of $\bJ$ therefore give
\bal\label{eq:optimal-transductive-exact-sequence-law}
\PP_{Z,\bU}\set{(I_1,\ldots,I_m)=\bi}
=\frac{1}{(n)_k}\frac{(n)_k}{n^m}
=\frac{1}{n^m}.
\eal
Since (\ref{eq:optimal-transductive-exact-sequence-law}) holds for every
$\bi\in[n]^m$, the indices are independent and uniform on $[n]$.
The labeled sequence in (\ref{eq:optimal-transductive-labeled-input})
is accordingly an independent sample from the uniform indexed labeled
population.  Equivalently, its features have distribution
$P_n\defeq n^{-1}\sum_{i=1}^n\delta_{\bx_i}$ and its labels are given by
$f^*$.  Here $\delta_{\bx}$ is the unit point mass at $\bx$, and
multiplicities in $\bX_n$ contribute their corresponding masses.

\noindent\textit{Step 3: the PAC guarantee.}
For a finite concept class with at least three distinct labelings, use
the deterministic majority-vote learner in
\citep[Theorem 2]{Hanneke16-optimal-sample-PAC}.
Equation (5) in that theorem's proof gives the fixed-sample-size bound
below with $C_{\mathrm{PAC}}=1800$.
The harmless change from labels $\set{-1,1}$ to $\set{0,1}$ preserves
VC dimension and classification error.
If $\abth{\cF_n}\le2$, choose a deterministic consistent member of
$\cF_n$ instead.  For an independent realizable sample and $\eps>0$,
each member with error greater than $\eps$ is consistent with
probability at most $\exp(-m\eps)$.
A union bound over at most two members yields error at most
$\log(2/\delta)/m$ with probability at least $1-\delta$.
Because $(m+1)/m\le2$, this bound is also at most the threshold below.
Thus, in all cases, the same deterministic rule $A_m$, fixed
independently of $\delta$, satisfies
\bal\label{eq:optimal-transductive-PAC-bound}
\PP_{Z,\bU}\set{\cL_n(\ell_{\widetilde f_m})
> \frac{C_{\mathrm{PAC}}}{m+1}
\pth{\dVC+\log(18/\delta)}}\le\delta.
\eal
The risk in (\ref{eq:optimal-transductive-PAC-bound}) is exactly the
$P_n$-classification error, and
(\ref{eq:optimal-transductive-exact-sequence-law}) supplies the
independence required by the PAC theorem.  That theorem is not being
applied conditionally on $\bar Z$.

\noindent\textit{Step 4: the test risk and its expectation.}
For every split and every binary predictor $f$, even if $f$ depends on
that split, nonnegativity gives the pathwise comparison
\bal\label{eq:optimal-transductive-test-risk-comparison}
\cL_n(\ell_f)=\frac mn\cL_m(\ell_f)+\frac un\cL_u(\ell_f),
\qquad
\cL_u(\ell_f)\le\frac nu\cL_n(\ell_f)\le2\cL_n(\ell_f).
\eal
No independence between $f$ and $Z$ is needed in
(\ref{eq:optimal-transductive-test-risk-comparison}).
Since $\dVC\ge2$ and $m+1\ge m$,
(\ref{eq:optimal-transductive-PAC-bound}) implies the probability bound
in (\ref{eq:optimal-transductive-rate}) with
$C_{\mathrm{tr}}\defeq
2C_{\mathrm{PAC}}\pth{1+\frac{\log18}{2}}$.
For every $x>0$, taking $\delta=\exp(-x)$ bounds the probability that
$\cL_u(\ell_{\widetilde f_m})>C_{\mathrm{tr}}(\dVC+x)/m$
by $\exp(-x)$.  Integrating this tail gives
\bals
\Expect{Z,\bU}{\cL_u(\ell_{\widetilde f_m})}
\le\frac{C_{\mathrm{tr}}\dVC}{m}
+\int_0^\infty\exp(-mt/C_{\mathrm{tr}})\diff t
=\frac{C_{\mathrm{tr}}(\dVC+1)}{m}.
\eals
This proves the expectation bound in
(\ref{eq:optimal-transductive-rate}) and completes the proof.
\end{proof}

\begin{remark}[Scope of the sampling transfer]
\label{remark:optimal-transductive-transfer-scope}
Theorem~\ref{theorem:optimal-transductive-transfer} is an existence result
for a randomized transductive learner, not a result about every ERM
selection.  Its probability guarantee includes $\bU$.
It neither asserts that the output belongs to $\cF$ nor supplies a
log-free TLC bound for the ERM in
Theorem~\ref{theorem:TLC-delta-ell-f-excess-risk-upper-bound-VC-dim}.
Its expected error has order at most $\dVC/m$.
Together with the lower bound in
\citep[Theorem 3]{tolstikhin2016minimaxlowerboundsrealizable}, this gives
the same order for the expected minimax risk over unrestricted randomized
learners when $u\ge m\ge\max\set{9,\dVC}$.
\end{remark}

\endgroup

\end{appendix}

\bibliographystyle{imsart-number} 
\bibliography{../../../revision2/ref}       

\end{document}